\documentclass[11pt,a4paper]{article}
\usepackage[T1]{fontenc}
\usepackage[utf8]{inputenc}
\usepackage{lmodern}
\usepackage[a4paper,left=23mm,right=23mm,top=24mm,bottom=25mm,headheight=14pt,headsep=6mm,footskip=11mm]{geometry}
\usepackage{amsmath,amssymb,amsthm}
\usepackage{IEEEtrantools,bbm}
\usepackage{graphicx,booktabs,multirow,longtable,tabularx,array}
\usepackage{algorithm,algpseudocode}
\usepackage{enumitem,placeins}
\usepackage[backend=biber,style=apa,natbib=true,sorting=nyt,maxbibnames=20]{biblatex}
\usepackage{xurl,microtype,xcolor}
\usepackage{caption,fancyhdr}
\usepackage[hyperfootnotes=false,bookmarksnumbered=true]{hyperref}
\hypersetup{hidelinks,pdftitle={Prediction Limits and Koopman Closure of Geometry-Induced Soft State Abstractions},pdfauthor={Mohit Kumar and Somayeh Kargaran},pdfsubject={Research manuscript}}
\fancypagestyle{plain}{\fancyhf{}\fancyfoot[C]{\thepage}}
\allowdisplaybreaks[1]
\providecommand{\Description}[1]{}
\newcolumntype{Y}{>{\raggedright\arraybackslash}X}

\DeclareMathOperator{\argmin}{argmin}

\DeclareMathOperator{\spanop}{span}
\DeclareMathOperator{\Proj}{Proj}

\newtheorem{theorem}{Theorem}[section]
\newtheorem{proposition}[theorem]{Proposition}

\newtheorem{corollary}[theorem]{Corollary}

\newtheorem{result}[theorem]{Result}
\newtheorem{problem}[theorem]{Problem}
\theoremstyle{definition}
\newtheorem{definition}[theorem]{Definition}

\theoremstyle{remark}
\newtheorem{remark}[theorem]{Remark}

\title{\bfseries\fontsize{17}{21}\selectfont Prediction Limits and Koopman Closure of Geometry-Induced Soft State Abstractions}
\author{\normalsize Mohit Kumar$^{1,2,*}$\qquad Somayeh Kargaran$^{2}$\\[7pt]
\small $^{1}$University of Rostock, Rostock, Germany\\
\small $^{2}$Software Competence Center Hagenberg, Hagenberg, Austria\\[6pt]
\small $^{*}$Corresponding author: \texttt{mohit.kumar@uni-rostock.de}\\
\small Somayeh Kargaran: \texttt{somayeh.kargaran@scch.at}\\[4pt]
\small ORCID: M. Kumar, 0000-0002-7368-5157; S. Kargaran, 0009-0005-8421-0447}
\date{}
\begin{document}
\raggedbottom
\maketitle
\thispagestyle{plain}
\vspace{-1.5em}

% ===== BEGIN sections/00_abstract.tex =====
\begin{abstract}
Soft state representations assign states nonnegative class weights that sum to one, but agreement with reference labels does not guarantee accurate prediction of the next weight vector. We derive lower confidence bounds on the minimum population root-mean-square prediction error over linear predictors with a prescribed spectral-norm limit. The bounds compare within-class successor variation with retained within-class input variation, with an alternative based on reference-label deviation. They are estimated from independent transition pairs without fitting a prediction matrix to the evaluation data. A bound above a chosen error tolerance excludes that accuracy for every admissible predictor at the stated confidence level; a bound of zero does not establish whether an accurate predictor exists. For coordinates constructed from Kernel Affine Hull Machine reconstruction scores, we relate the reference class's affinity advantage over competitors to label agreement. The class-centered population bound equals the optimum in a four-state example. Using 117,000 previously generated evaluation datasets from fixed four-state designs, we compare the finite-sample certificates with an empirical-risk-minimization baseline. For one Van der Pol representation selected before evaluation, the label-based bound is positive in all 64 independent evaluation datasets at the fitted matrix's spectral norm. Separately, exact matrix propagation under deterministic dynamics yields a Koopman interpretation of the coordinate-function span. The certificates help identify forecasting targets that cannot be reached by refitting a fixed representation within a specified predictor class and evaluation distribution.
\end{abstract}

% ===== END sections/00_abstract.tex =====

\par
\medskip
{\small
\noindent\textbf{Acknowledgments.} The research reported in this paper has been supported by the Austrian Ministry for Transport, Innovation and Technology, the Federal Ministry for Digital and Economic Affairs, and the State of Upper Austria through the SCCH competence center INTEGRATE (FFG grant no. 892418), part of the FFG COMET Competence Centers for Excellent Technologies Programme.\par
}
\clearpage

% ===== BEGIN sections/01_introduction.tex =====
\section{Introduction}
\label{sec:introduction}
A central goal of dynamical representation learning is to replace a complex state with fewer coordinates that still support forecasting. Hard abstractions assign states to reference classes, whereas soft abstractions describe each state by nonnegative weights that sum to one \citep{singh1994soft,abel2016,gelada2019}. A representation may closely reproduce a chosen partition yet fail to retain distinctions needed to predict how those weights evolve. Agreement with reference labels and predictive sufficiency are therefore different properties.

Consider two states in the same reference class. Suppose that their current soft-coordinate vectors are identical but that their successors have different coordinate vectors. Any predictor that receives only the current coordinates must make the same prediction for both states, so it cannot predict both successors exactly. If the inputs are merely close rather than identical, a linear predictor can amplify their difference only to the extent permitted by its spectral norm. The limitation therefore concerns information available to an entire \emph{class} of predictors, not necessarily the quality of any one fitted model.

Our central question is whether independent data can establish that \emph{every} norm-bounded linear predictor of a fixed soft representation must incur a specified amount of population prediction error. The answer matters when choosing representations and interpreting unsuccessful training runs. A poor fitted forecast may suggest that optimization, data, or the predictive model should be improved; a valid lower confidence bound above the required error tolerance instead shows that changing the fitting algorithm alone cannot reach that tolerance while the representation, evaluation law, and norm limit remain unchanged. We call these lower confidence bounds \emph{prediction-error certificates}: they are computed from variation within reference classes, without optimizing a prediction matrix on the evaluation data. Their conclusion applies only to the specified representation, evaluation distribution, and predictor class.

The motivating representations use Kernel Affine Hull Machines (KAHMs), published kernel-based models that reconstruct states from reference samples \citep{kumar2023federated,kumar2024kahm}. For each reference class, the reconstruction score measures how well that class's samples reconstruct a query state. Normalizing affinities derived from the scores produces a soft-coordinate vector. This concrete construction lets us ask how reconstruction quality, agreement with reference labels, and successor predictability relate. The statistical certificates apply more generally to any fixed measurable simplex-valued representation under the stated evaluation assumptions.

We also study an idealized counterpart to approximate forecasting: \emph{exact deterministic linear closure}, in which one matrix propagates the soft coordinates at every state. In that special case the coordinate functions span a Koopman-invariant space, and the associated finite-dimensional operator must be interpreted after removing redundant coordinates. We characterize this conditional structure and its reproducing-kernel Hilbert-space (RKHS) adjoint. Neither small prediction error nor the forecasting experiments proves exact closure.

\subsection{From the obstruction to a quantitative bound}
\label{sec:intro_example}
A simple example illustrates the issue. Partition four states into two reference classes, and assign the same current soft-coordinate vector to both states in each class. If two states in the first class have different successor coordinates, no predictor using the current coordinates can distinguish their next representations. Section~\ref{ex:kahm_conflict} provides an explicit reconstruction-based realization and computes its best achievable prediction error exactly.

For a general representation, averaging successor coordinates within each reference class gives the optimal predictor that knows only the hard label. The variation around these class means measures the information that the label alone cannot capture. Soft coordinates may retain additional distinctions \emph{within} a class, but a bounded-norm linear matrix can exploit them only to a limited degree. Our main population inequality subtracts this possible improvement from the class-conditional successor variation. We call this the \emph{class-centered bound}: the current coordinate vector is centered by subtracting its class-conditional mean, and the remaining within-class dispersion enters the bound. The term names our inequality, not an additional data assumption.

We also derive a more readily evaluated bound based on deviation from one-hot reference labels. Fresh, independent transition pairs provide confidence corrections for both approaches. A positive bound is not automatically evidence against a useful nonzero accuracy target: to exclude a target root-mean-square error $\varepsilon$, the certificate must exceed $\varepsilon$. Conversely, a zero certificate does not establish that an accurate predictor exists.

Predictability alone is not a complete measure of representation quality. An almost constant coordinate map can be easy to forecast while failing to preserve distinctions of interest. Our forecasting studies therefore report variation in the coordinates and distinguish comparisons on a common target from comparisons in which each representation defines its own target.

\subsection{Representation and prediction task}
\label{sec:intro_problem}\label{sec:intro_abstraction_scope}
For $C$ reference classes, the coordinate map is $\Phi(x)=(\Phi_1(x),\ldots,\Phi_C(x))^\top\in\Delta_C$, where $\Delta_C$ denotes the probability simplex. A reference partition supplies a class label $c(x)$ and its one-hot vector $y(x)=\mathrm e^{c(x)}$. The classes may be fitted to training data (for example, by clustering); they need not be ground-truth categories or sufficient for prediction. Moreover, the label need not agree with the largest component of $\Phi(x)$. The discrepancy between these two descriptions enters the simpler label-based bound.

Our prediction target is the \emph{next coordinate vector} $\Phi(x^+)$, not the raw successor state or the next hard label. We estimate it from the current vector $\Phi(x)$ under a fixed joint law for current and successor states, using squared Euclidean coordinate loss. Transitions may be deterministic or stochastic, and current and successor marginal laws need not coincide. Reconstruction models, reference labels, and the coordinate map are constructed or selected before the evaluation sample is drawn; the conditional confidence guarantees treat them as fixed.

Although the coordinate vectors lie in a probability simplex, their entries quantify \emph{relative reconstruction affinity}; they are not necessarily calibrated class probabilities or features sufficient for a downstream task. Accordingly, the numerical control experiments concern prediction under fixed behavior policies, not learned policies or planning performance. For a deterministic map $F$, exact closure means $\Phi(F(x))=B^\top\Phi(x)$ throughout the specified domain. With $K_Fg=g\circ F$, this identity induces invariance of the span of the coordinate functions, as in operator-based observable modeling \citep{williams2015}. Redundant coefficient directions must be removed to obtain a faithful matrix representation of the induced operator.

\subsection{KAHM-based coordinates in brief}
\label{sec:intro_kahm_overview}
The construction has three steps: reconstruct a state using a KAHM fitted to each reference class; convert each class-specific reconstruction score into a positive affinity; then normalize those affinities across classes. Specifically, let $\mathcal T_c(x)\in[0,1]$ be the class-$c$ reconstruction score, with smaller values indicating better reconstruction. Adapting the normalized-affinity idea of \citet{kumar2026semantic}, we use
\[
\Phi_c(x)=\frac{(1-\mathcal T_c(x)+\tau)^\omega}{\sum_{j=1}^C(1-\mathcal T_j(x)+\tau)^\omega},
\qquad \tau>0,\ \omega>0.
\]
Here $\tau$ keeps all affinities positive, while $\omega$ controls how strongly better-reconstructing classes are favored. The detailed reconstruction and normalization conventions appear in Sections~\ref{sec:mathematical_background} and~\ref{sec:method}. The numerical implementation also assigns equal class weights when the powered-affinity sum falls below a fixed threshold. This fallback is part of the evaluated map (Section~\ref{sec:exp_protocol}). For later use, a \emph{reconstruction-score margin} compares the reference class's log-affinity with the largest competing log-affinity (Section~\ref{sec:kahm_geometric_fidelity}); it is not a geometric distance. No general risk certificate requires this particular coordinate formula.

\subsection{Contributions and scope}
\label{sec:intro_contributions}
\textbf{Certifying an unavoidable forecasting error.}
We bound the smallest population root-mean-square error that \emph{any} linear predictor can attain at a given spectral-norm limit. Proposition~\ref{prop:population_obstruction} quantifies the competition between successor variation and distinguishable within-class input variation. Its class-centered inequality leads to a finite-sample certificate (Proposition~\ref{prop:centered_certificate}); a simpler reference-label-based certificate (Proposition~\ref{prop:exclusion_certificate}) and a refinement using representative coordinate vectors fixed before evaluation (Corollary~\ref{cor:prototype_certificate}) provide alternatives. These confidence statements use evaluation observations rather than a fitted predictor and hold, individually, for every norm limit once the representation is fixed. Proposition~\ref{prop:class_occupancy_bias} quantifies the effect of estimating class means with a finite sample.

\textbf{Reconstruction-score bounds and exact evolution.}
Theorem~\ref{thm:kahm_margin_fidelity} shows how a reconstruction-score advantage for the reference class controls the deviation of the resulting soft coordinates from that label. The four-state construction (Section~\ref{ex:kahm_conflict}) then shows that close label agreement need not imply accurate forecasting: its class-centered population bound equals the exact optimum. We separately analyze error introduced by finite matrix-fitting updates. Under the \emph{additional assumption} of exact deterministic closure, Proposition~\ref{prop_koopman_eigenfunction} describes Koopman spectral correspondence on the nonredundant coordinate span. We do not claim the basic implication from exact closure to Koopman invariance as a new result.

\textbf{Testing what certificates do and do not establish.}
A controlled study with a known optimum uses 117,000 archived evaluation datasets from fixed four-state designs to compare the bounds based on class-wise statistics with an empirical risk minimization (ERM) baseline. A second study uses 64 datasets of independently sampled transition pairs under two transition laws to test the label-based certificate for \emph{one} representation chosen using exploratory data before the independent evaluation. Forecasting comparisons examine coordinate variation, prediction horizons, and prediction against both common and representation-specific targets. The repeated four-state datasets are not independently learned abstractions, and positive realized-successor prediction risk under stochastic dynamics does not by itself contradict exact linear evolution of the conditional mean.

\subsection{Related work and positioning}
\label{sec:related_work}
\paragraph{State aggregation and learned representations.}
Soft aggregation was originally developed for value approximation \citep{singh1994soft}. Other abstractions preserve rewards, transitions, behavior, Markov structure, or causal properties \citep{abel2016,gelada2019,allen2021markov,wang2024causal}. Markov-chain lumpability, for example, equates class-transition probabilities among states in a class \citep{zhang2020compression}. Our target is instead the population error of predicting the next \emph{soft-coordinate vector} within a prescribed linear model class. With stochastic transitions, that error includes variation around the successor's conditional mean.

\paragraph{Planning and interpretation.}
Symbols intended for evaluating high-level plans \citep{konidaris2018symbols} and supervised concept bottlenecks \citep{koh2020} have task-dependent semantics. KAHM reference samples instead provide a reconstruction-based interpretation of coordinates, without establishing planning or decision sufficiency. Transfer-operator methods identify coherent sets \citep{froyland2013}, while VAMPnets learn memberships jointly with kinetic models \citep{mardt2018}. Predictive sufficiency or dynamical coherence requires evidence beyond reconstruction affinity.

\paragraph{KAHM reconstruction and normalized affinities.}
Existing work applies KAHMs in federated learning \citep{kumar2023federated}, develops the KAHM reconstruction model \citep{kumar2024kahm}, and studies induced kernels \citep{kumar2025collaborative}, bounded reconstruction scores \citep{kumar2025operator}, and related normalized encodings \citep{kumar2026semantic}. We construct the representation from these ingredients, retaining all classes and using a positive offset and a general positive exponent. Our reconstruction-score margin result bounds deviation from reference labels using differences among class-wise reconstruction scores, whereas the general risk certificates require only a fixed measurable simplex-valued representation.

\paragraph{Lower confidence bounds on minimum risk.}
\citet{muller2025falsification} examine model-class falsification and its distribution-free limitations, and \citet{wu2026minimalrisk} derive nonasymptotic lower confidence bounds on minimal risk. Thus the general idea of certifying minimum risk is not new. Our contribution is to exploit \emph{class-conditioned successor and input variation} in a dynamical representation to obtain bounds without minimizing an empirical loss over prediction matrices. Section~\ref{sec:certificate_comparison} gives a matched comparator based on empirical risk minimization using independent transition pairs and identifies the statistical and computational distinctions.

\paragraph{Linear prediction and Koopman approximation.}
Dynamic mode decomposition (DMD) and extended dynamic mode decomposition (EDMD) estimate linear evolution in prescribed spaces of observable functions \citep{schmid2010,williams2015,williams2015kernel}; learned dictionaries adapt these spaces \citep{li2017dictionary,takeishi2017,lusch2018,otto2019,mondal2024interactive}. RKHS Koopman regression studies estimation and prediction guarantees \citep{kostic2022regression}. Our minimum-risk bounds concern a fixed representation and norm-bounded matrix class, rather than estimating an operator or establishing a small residual for one fitted matrix.

\paragraph{Reliability of spectral claims.}
Residual DMD \citep{colbrook2023}, stochastic residual analysis \citep{colbrook2024stochastic}, RKHS spectral approximation \citep{boulle2025}, and principal-angle invariance diagnostics \citep{conradie2026trustworthy} address reliability of operator approximations. These methods address questions other than whether an error tolerance can be excluded for a norm-bounded predictor class. The empirical spectral measures used here are descriptive unless the assumptions supporting a population operator interpretation hold (Section~\ref{sec:finite_koopman_closure}).

\subsection{Notation and basic conventions}\label{sec:intro_notation}
We use $\mathbb Z_+$, $\mathbb R$, $\mathbb C$, and $\mathbb R_+=[0,\infty)$ for positive integers, real numbers, complex numbers, and nonnegative reals; logarithms are natural. The norms $\|\cdot\|$ and $\|\cdot\|_1$ are Euclidean and $1$-norms, while $\|\cdot\|_2$ and $\|\cdot\|_F$ denote spectral and Frobenius matrix norms. Superscript $\top$ denotes transpose; columns are concatenated as $[v^1\ \cdots\ v^N]$. We write $I_N$, $\mathbf1_C$, $\mathbf0_C$, $\mathrm e^c$, and $\mathbbm1_A$ for the identity, all-ones vector, zero vector, canonical basis vector, and indicator. The state region $\mathcal X\subset\mathbb R^n$ carries its trace Borel $\sigma$-algebra \citep{folland1999}.

Current and successor states are $x$ and $x^+$, with joint law $\mathbb P_{x,x^+}$ and marginals $\mathbb P_x$ and $\mathbb P_{x^+}$. We abbreviate the joint law to $\mathbb P$ in risk expressions. Superscripts enumerate observations, primes denote separate evaluation samples, subscripts denote time, and boldface $\mathbf X$ denotes a reference-data matrix. An expectation subscript identifies the integrated variables and their law, for example $\mathbb E_{(x,x^+)\sim\mathbb P}$; conditioning is explicit and all finite-sample confidence statements specify their sampling law. For an integrable $g$,
\begin{equation}
\mathbb E_{x\sim\mathbb P_x}[g(x)]
:=\int_{\mathcal X}g(u)\,\mathbb P_x(\mathrm du).
\label{eq:expectation_convention}
\end{equation}
For $y(x)=\mathrm e^{c(x)}$, $y_c=1$ means $c(x)=c$; whenever that class has positive probability,
\begin{equation}
\mathbb E_{x\sim\mathbb P_x}[g(x)\mid y_c=1]
=\int_{\mathcal X}g(u)\,\mathbb P_x(\mathrm du\mid c(u)=c).
\label{eq:conditional_expectation_convention}
\end{equation}
We write $\mathbb P_{x\mid y}$ for the corresponding conditional laws. The probability simplex is
\begin{equation}
\Delta_C=\{v\in\mathbb R^C:\ v_c\geq0\text{ for all }c,\ \mathbf1_C^\top v=1\},
\label{eq:simplex_definition}
\end{equation}
and $\Phi_c(x)=(\Phi(x))_c$.

\subsection{Organization}
Section~\ref{sec:prediction_limits} derives the general population bound and its finite-sample certificates. Sections~\ref{sec:mathematical_background} and~\ref{sec:method} introduce the reconstruction-based coordinates and state the conditional Koopman interpretation. Section~\ref{sec:adaptive_closure_estimation} relates agreement with reference labels, successor variation, and matrix-fitting error through KAHM reconstruction-score bounds. Section~\ref{sec:exp_certificate} evaluates certificate validity and informativeness, while Section~\ref{sec:experiments} examines forecasting performance and coordinate variation. Section~\ref{sec:discussion} discusses practical use and limitations. The appendices provide derivations, data protocols, and verification records.
% ===== END sections/01_introduction.tex =====

% ===== BEGIN sections/02_prediction_limits.tex =====
\section{Certifying Prediction Limits for Fixed Soft Coordinates}
\label{sec:prediction_limits}
We begin with a question independent of how the soft coordinates were constructed: what error must remain even after choosing the best admissible linear prediction matrix? The key quantities are (i) how much successor coordinates vary among states bearing the same reference label and (ii) how much those states can still be distinguished by their current soft coordinates. We call the resulting inequality \emph{class-centered}, because current coordinates are centered around their class-conditional mean. A simpler, potentially looser alternative replaces that input variation with disagreement from one-hot reference labels. Sections~\ref{sec:method} and~\ref{sec:kahm_geometric_fidelity} later bound this label-deviation quantity using differences in KAHM reconstruction scores.

\subsection{Evaluation law and a population lower bound}
\label{sec:population_obstruction}
Fix a finite class count $C$, a measurable label map $c:\mathcal X\to\{1,\ldots,C\}$, and a measurable coordinate map $\Phi:\mathcal X\to\Delta_C$. Let $x$ and $x^+$ be the current-state and successor random variables. Their joint evaluation distribution is $\mathbb P_{x,x^+}$, with current-state marginal $\mathbb P_x$. As in Section~\ref{sec:intro_notation}, write $\mathbb P:=\mathbb P_{x,x^+}$. Write $\phi=\Phi(x)$, $z=\Phi(x^+)$, and $y=y(x)=\mathrm e^{c(x)}$, the one-hot reference label. Every reference class considered in the analysis has positive probability $\pi_c=\mathbb P_y(y_c=1)>0$. The coordinate map, class map, and evaluation law are fixed for each certificate. In particular, selecting a representation using these same evaluation outcomes would invalidate the stated conditional confidence guarantee.

The joint law allows deterministic or stochastic transitions. The current-state and successor marginals may differ. For example, the law can describe one specified time or a fixed mixture of times along a nonstationary process. For the linear predictor $B^\top\phi$, let $R_{\mathbb P}$ denote mean squared prediction error. Let $\mathcal D_\kappa$ be the smallest root-mean-square error over prediction matrices $B\in\mathbb R^{C\times C}$ with spectral norm at most $\kappa$:
\begin{equation}
\begin{aligned}
R_{\mathbb P}(\Phi,B)
&:=\mathbb{E}_{(x,x^+)\sim\mathbb P}\!\left[\|z-B^\top\phi\|^2\right],\\
\mathcal D_\kappa(\Phi)
&:=\inf_{\|B\|_2\leq\kappa}R_{\mathbb P}(\Phi,B)^{1/2},
\qquad \kappa\geq0.
\end{aligned}
\label{eq:joint_risk_defect}
\end{equation}
If $x^+=F(x)$, we also write $R_{\mathbb P_x}$ to emphasize that the initial-state law determines this risk. The spectral-norm limit constrains sensitivity: if two current coordinate vectors are close, their predictions cannot differ by more than $\kappa$ times that distance. As a reference point, average the successor coordinates separately within each current reference class:
\begin{equation}
\begin{aligned}
z_*^c&:=\mathbb{E}_{(x,x^+)\sim\mathbb P}\!\left[z\mid y_c=1\right],\\
B_*^\top&:=[z_*^1\ \cdots\ z_*^C],\\
\sigma_{\mathrm{res}}^2
&:=\mathbb{E}_{(x,x^+)\sim\mathbb P}\!\left[\|z-z_*^{c(x)}\|^2\right].
\end{aligned}
\label{eq:joint_class_means}
\end{equation}
The prediction $B_*^\top y$ assigns the same mean successor vector to every state in a class. It minimizes mean squared error among all predictors that use only the hard label. Its root-mean-square error is $\sigma_{\mathrm{res}}$. The error is defined relative to the information in the hard label. It may reflect differences between the successors of states in the same class, randomness in the successor at a fixed state, or both. To compare label-based prediction with soft-coordinate prediction, define
\begin{equation}
\begin{aligned}
a_\Phi^2&:=\mathbb{E}_{x\sim\mathbb P_x}\!\left[\|\phi-y\|^2\right],\\
A_{\mathrm{fid}}&:=\mathbb{E}_{x\sim\mathbb P_x}\!\left[(1-\phi_{c(x)})^2\right].
\end{aligned}
\label{eq:joint_assignment_losses}
\end{equation}
The quantity $a_\Phi$ measures Euclidean distance from the one-hot reference label, while $A_{\mathrm{fid}}$ is the mean squared shortfall of the coordinate assigned to that label. Neither statistic measures predictive accuracy or misclassification frequency. Since the other coordinates sum to $1-\phi_{c(x)}$, $a_\Phi^2\leq2A_{\mathrm{fid}}$. To avoid losing information about within-class input differences, center the \emph{current} coordinates around their own class mean:
\begin{equation}
\mu_\Phi^c:=\mathbb E_{x\sim\mathbb P_x}[\phi\mid c(x)=c],\qquad
\tau_\Phi^2:=\mathbb E_{x\sim\mathbb P_x}\!\left[\|\phi-\mu_\Phi^{c(x)}\|^2\right].
\label{eq:within_class_input_dispersion}
\end{equation}
Conditional expectation minimizes squared loss among predictors using only the class label, so
$0\leq\tau_\Phi^2\leq a_\Phi^2\leq2A_{\mathrm{fid}}$. Importantly, $\tau_\Phi=0$ whenever the input coordinates are constant within each reference class, even if those constants are far from one-hot labels. This is why centering can give a substantially stronger bound.

\begin{proposition}[Class-centered population prediction-risk bound]
\label{prop:population_obstruction}\leavevmode\par
For every $\kappa\geq0$, writing $(t)_+:=\max(t,0)$,
\begin{equation}
\mathcal D_\kappa(\Phi)
\geq(\sigma_{\mathrm{res}}-\kappa\tau_\Phi)_+
\geq(\sigma_{\mathrm{res}}-\kappa a_\Phi)_+
\geq(\sigma_{\mathrm{res}}-\kappa\sqrt{2A_{\mathrm{fid}}})_+.
\label{eq:population_obstruction}
\end{equation}
If $\kappa\geq\|B_*\|_2$, then
\begin{equation}
\mathcal D_\kappa(\Phi)\leq\sigma_{\mathrm{res}}+\|B_*\|_2a_\Phi.
\label{eq:population_attainability}
\end{equation}
In particular, $\kappa=\sqrt C$ admits $B_*$.
\end{proposition}
\begin{proof}
For any matrix $B$, the class-measurable prediction $B^\top\mu_\Phi^{c(x)}$ obeys the conditional-mean orthogonality identity
\begin{equation}
\mathbb E_{(x,x^+)\sim\mathbb P}\!\left[\|z-B^\top\mu_\Phi^{c(x)}\|^2\right]
=\sigma_{\mathrm{res}}^2+
\mathbb E_{x\sim\mathbb P_x}\!\left[\|z_*^{c(x)}-B^\top\mu_\Phi^{c(x)}\|^2\right]
\geq\sigma_{\mathrm{res}}^2.
\label{eq:hard_label_orthogonality}
\end{equation}
The reverse triangle inequality in $L^2$ gives
$R_{\mathbb P}(\Phi,B)^{1/2}\geq\sigma_{\mathrm{res}}-\|B\|_2\tau_\Phi$.
Taking the infimum over the norm ball proves the first inequality. The remaining inequalities follow from $\tau_\Phi\leq a_\Phi\leq\sqrt{2A_{\mathrm{fid}}}$.
For the upper bound, use the feasible $B_*$ and the triangle inequality on $z-B_*^\top\phi=(z-B_*^\top y)+B_*^\top(y-\phi)$. The rows of $B_*$ lie in $\Delta_C$, so $\|B_*\|_2\leq\|B_*\|_F\leq\sqrt C$.
\end{proof}

The first term, $\sigma_{\mathrm{res}}$, quantifies how differently states sharing a reference label evolve. The deduction $\kappa\tau_\Phi$ is the largest improvement that this argument allows from the distinctions retained in their current coordinates. Thus successor variation alone is not enough to force a prediction error: if the input coordinates separate states sufficiently, a constrained matrix may exploit that information. The class-centered inequality is at least as strong as the relaxation based on deviation from the one-hot reference label; it is strictly stronger for $\kappa>0$ when $\tau_\Phi<a_\Phi$ and both positive parts are active. Neither lower bound is generally exact, since equality need not hold in the reverse triangle inequality. The upper bound is available only when $B_*$ satisfies the specified spectral-norm limit. These population claims concern a fixed representation and evaluation law; selection on evaluation outcomes is not covered.

\subsection{A computable lower confidence bound}
\label{sec:main_certificate}
A certificate should be computable without knowing the transition distribution or solving a constrained regression problem. Our procedure groups the independent evaluation pairs by their \emph{current} reference class, measures successor dispersion, measures input-to-label disagreement, and adjusts both statistics for sampling uncertainty. The result is a lower confidence bound on the best possible population RMSE, not a training loss or the error of a particular fitted matrix. We condition on the independent representation-construction and selection data so that the maps are fixed. Appendix~\ref{app:fidelity_concentration} treats independent trajectories whose transitions may be dependent within a trajectory.
\begin{proposition}[Computable prediction-error bound under a spectral-norm limit]
\label{prop:exclusion_certificate}\leavevmode\par
Fix the reference-class map $c$ and measurable $\Phi:\mathcal X\to\Delta_C$ using construction and selection data. Draw $M$ independent and identically distributed (i.i.d.) realizations $(x^i,x^{+,i})$ from $\mathbb P_{x,x^+}$, independently of construction and selection, and set $\phi^i=\Phi(x^i)$, $z^i=\Phi(x^{+,i})$, and $c_i=c(x^i)$. Denote the random evaluation sample by $\mathcal V_M$; its observed entries are these pairs, and its distribution, conditional on the fixed construction and selection, is $\mathbb P_{\mathcal V_M}:=(\mathbb P_{x,x^+})^{\otimes M}$. Let $n_c:=|\{i:c_i=c\}|$. For $n_c>0$, let $\bar z_c:=n_c^{-1}\sum_{i:c_i=c}z^i$ be the mean successor coordinates in class $c$; for an empty class, choose any $\bar z_c\in\Delta_C$, since it contributes no summands. Define the empirical assignment loss $\widehat f$, within-class successor variance $\widehat s$, and confidence corrections by
\begin{align*}
\widehat f&:=\frac1M\sum_i(1-(\phi^i)_{c_i})^2,
&\widehat s&:=\frac1M\sum_i\|z^i-\bar z_{c_i}\|^2,\\
r_\delta&:=\frac{\log(2/\delta)}{M},
&V_\delta&:=(\widehat s-\sqrt{2r_\delta})_+,\\
F_\delta&:=\min\{1,\widehat f+r_\delta+
\sqrt{r_\delta^2+2r_\delta\widehat f}\},
&&0<\delta<1.
\end{align*}
Conditional on construction and selection, with probability at least $1-\delta$, simultaneously for every $\kappa\geq0$,
\begin{equation}
\begin{aligned}
\inf_{\|B\|_2\leq\kappa}
\left(\mathbb{E}_{(x,x^+)\sim\mathbb P}\!\left[\|\Phi(x^+)-B^\top\Phi(x)\|^2\right]\right)^{1/2}
&\geq L_{\kappa,\delta},\\
L_{\kappa,\delta}&:=\bigl(\sqrt{V_\delta}-\kappa\sqrt{2F_\delta}\bigr)_+.
\end{aligned}
\label{eq:exclusion_certificate}
\end{equation}
\end{proposition}
In practical terms, $\sqrt{V_\delta}$ is a conservative estimate of the error left when predicting from the hard label alone. The term $\kappa\sqrt{2F_\delta}$ allows for the improvement a norm-bounded matrix could gain from soft inputs. Subtracting the latter from the former and truncating below at zero gives the certificate. A large observed successor spread can still lead to an inconclusive certificate if the input-to-label discrepancy is large.
\begin{proof}
Throughout this proof, expectations and probability statements are conditional on the data used for representation construction and model selection.

\emph{Step 1: lower-bound successor variation.} For each observed class, the empirical mean minimizes the sum of squared distances to its successor vectors. Consequently,
\[
\widehat s=\min_{a^1,\ldots,a^C\in\Delta_C}
\frac1M\sum_i\|z^i-a^{c_i}\|^2,
\qquad \mathbb{E}_{\mathcal V_M\sim\mathbb P_{\mathcal V_M}}\!\left[\widehat s\right]\leq\sigma_{\mathrm{res}}^2.
\]
The empirical loss at the sample class means is no larger than the empirical loss at the population class means. Taking expectations proves the second inequality.

For any fixed set of class means, each loss lies in $[0,2]$. Replacing one evaluation pair changes the average by at most $2/M$; taking the minimum preserves this bound. McDiarmid's inequality~\citep{mcdiarmid1989} therefore gives
\[
\Pr_{\mathcal V_M\sim\mathbb P_{\mathcal V_M}}\!\left(
\widehat s>\mathbb{E}_{\mathcal V_M\sim\mathbb P_{\mathcal V_M}}\!\left[\widehat s\right]
+\sqrt{2r_\delta}\right)\leq\delta/2.
\]
Hence $\sigma_{\mathrm{res}}^2\geq V_\delta$ on the complementary event. The bounded-differences argument remains valid when class means are estimated from these same evaluation data and when some classes are empty.

\emph{Step 2: upper-bound assignment loss.} For $\ell(x):=(1-\phi_{c(x)})^2\in[0,1]$, let
$f:=\mathbb{E}_{x\sim\mathbb P_x}\!\left[\ell(x)\right]=A_{\mathrm{fid}}$.
For $t\geq0$, the exponential-moment bound
\[
\mathbb{E}_{x\sim\mathbb P_x}\!\left[e^{-t\ell(x)}\right]
\leq1+f(e^{-t}-1)
\]
gives the multiplicative lower-tail inequality~\citep{hoeffding1963}
$\Pr_{\mathcal V_M\sim\mathbb P_{\mathcal V_M}}(\widehat f<f-\sqrt{2fr_\delta})\leq\delta/2$.
For $f=A_{\mathrm{fid}}\geq\widehat f$, squaring $f-\widehat f\leq\sqrt{2fr_\delta}$ and solving the resulting quadratic inequality gives
\[
 f\leq\widehat f+r_\delta+\sqrt{r_\delta^2+2r_\delta\widehat f}.
\]
The case $f<\widehat f$ is immediate. Since $f\leq1$, both cases yield $f\leq F_\delta$.
\emph{Step 3: combine the two bounds.} A union bound gives $\sigma_{\mathrm{res}}^2\geq V_\delta$ and $A_{\mathrm{fid}}\leq F_\delta$ on one event of probability at least $1-\delta$. Substitution into Proposition~\ref{prop:population_obstruction} proves~\eqref{eq:exclusion_certificate}. This confidence event involves only the two scalar population quantities. It does not depend on a particular prediction matrix or norm limit, which proves the simultaneous guarantee over $B$ and $\kappa$.
\end{proof}
If $L_{\kappa,\delta}>\varepsilon$, the certificate establishes, at confidence level $1-\delta$, that every prediction matrix with spectral norm at most $\kappa$ has root-mean-square error above $\varepsilon$. A bound of zero does not establish that any accurate predictor exists. The same confidence event provides the guarantee for every $\kappa\geq0$. For finitely many representations fixed before evaluation, a union bound gives a joint guarantee by assigning a failure probability to each representation. Adaptive selection of a new coordinate map or reference-label map requires fresh evaluation data or a guarantee covering that selection.

Unlike the evaluation RMSE of a fitted matrix, $L_{\kappa,\delta}$ concerns the population optimum over every admissible matrix. Section~\ref{sec:exp_certificate} assesses its empirical coverage only where that optimum is analytically known; comparisons with fitted-predictor errors elsewhere are descriptive.

\subsection{A class-centered certificate on the same evaluation pairs}
\label{sec:centered_certificate}
The previous certificate conservatively replaces actual within-class input variation by distance from one-hot labels. A closer measure is obtained by subtracting the mean input-coordinate vector within each reference class and measuring only the remaining dispersion. This \emph{class-centered} certificate estimates $\tau_\Phi^2$ directly from the same independent pairs. It does not fit a prediction matrix and does not need a separate evaluation sample merely to estimate the class means.
\begin{proposition}[Finite-sample class-centered prediction-error certificate]
\label{prop:centered_certificate}\leavevmode\par
Under the assumptions of Proposition~\ref{prop:exclusion_certificate}, let
$\bar\phi_c:=n_c^{-1}\sum_{i:c_i=c}\phi^i$ for $n_c>0$, and let empty classes contribute no summands. Set
\begin{equation}
\widehat t:=\frac1M\sum_{i=1}^M\|\phi^i-\bar\phi_{c_i}\|^2,
\qquad
T_\delta:=\min\!\left\{1,\widehat t+\sqrt{2r_\delta}+\frac{C}{M}\right\},
\label{eq:input_variance_confidence}
\end{equation}
where $r_\delta=\log(2/\delta)/M$, and use $V_\delta$ as in Proposition~\ref{prop:exclusion_certificate}. Conditional on the construction and selection data, with probability at least $1-\delta$, simultaneously for all $\kappa\geq0$,
\begin{equation}
\boxed{\quad\mathcal D_\kappa(\Phi)\geq
\widetilde L_{\kappa,\delta}
:=\left(\sqrt{V_\delta}-\kappa\sqrt{T_\delta}\right)_+.\quad}
\label{eq:centered_certificate}
\end{equation}
The cost of computing $\widehat t$ and $\widehat s$ is $O(MC)$ after the coordinates have been evaluated.
\end{proposition}
\begin{proof}
The empirical input dispersion is the minimum of $M^{-1}\sum_i\|\phi^i-h^{c_i}\|^2$ over representative coordinate vectors $h^c\in\Delta_C$, one for each reference class. Each loss lies in $[0,2]$, so replacing one pair changes this minimum by at most $2/M$. McDiarmid's inequality therefore gives
$\mathbb E_{\mathcal V_M\sim\mathbb P_{\mathcal V_M}}[\widehat t]\leq\widehat t+\sqrt{2r_\delta}$ except on an event of probability $\delta/2$.
Writing $u_c:=\mathbb E_{x\sim\mathbb P_x}[\|\phi-\mu_\Phi^c\|^2\mid c(x)=c]\leq1$, the class-occupancy calculation used in Proposition~\ref{prop:class_occupancy_bias} gives
\[
\tau_\Phi^2-\mathbb E_{\mathcal V_M\sim\mathbb P_{\mathcal V_M}}[\widehat t]
=\frac1M\sum_{c=1}^C u_c[1-(1-\pi_c)^M]\leq\frac{C}{M}.
\]
Also $\tau_\Phi^2\leq1$. Thus $\tau_\Phi^2\leq T_\delta$ on the stated event. The bounded-differences argument for successor variation yields $\sigma_{\mathrm{res}}^2\geq V_\delta$ with failure probability at most $\delta/2$. Apply a union bound and Proposition~\ref{prop:population_obstruction}. Neither event involves $\kappa$.
\end{proof}

Centering gives the stronger population inequality because $\tau_\Phi^2\leq 2A_{\mathrm{fid}}$. For small evaluation samples, however, estimating the class means incurs a $C/M$ occupancy correction as well as a concentration penalty. The class-centered \emph{sample} certificate can therefore be weaker than the label-based one. Selecting the larger certificate after observing the data requires joint coverage: for example, allocate failure probability $\delta/2$ to each method to obtain overall confidence $1-\delta$.

\begin{corollary}[Certificate using prespecified class representatives]
\label{cor:prototype_certificate}
A \emph{prototype} here means a representative coordinate vector $h^c\in\Delta_C$ selected for reference class $c$ before evaluation, not a KAHM reference state or the unknown class mean successor. Choose these prototypes independently of the evaluation sample (the one-hot choice $h^c=\mathrm e^c$ requires no estimation). Define
\[
\widehat q_h:=\frac{1}{2M}\sum_{i=1}^M\|\phi^i-h^{c_i}\|^2,
\qquad Q_{h,\delta}:=\min\{1,\widehat q_h+r_\delta+
\sqrt{r_\delta^2+2r_\delta\widehat q_h}\}.
\]
Under the assumptions of Proposition~\ref{prop:exclusion_certificate}, with probability at least $1-\delta$, simultaneously for all $\kappa\geq0$,
\begin{equation}
\mathcal D_\kappa(\Phi)\geq
L^{\mathrm{proto}}_{\kappa,\delta}(h)
:=\bigl(\sqrt{V_\delta}-\kappa\sqrt{2Q_{h,\delta}}\bigr)_+.
\label{eq:fixed_prototype_certificate}
\end{equation}
For $h^c=\mathrm e^c$, the bound is at least $L_{\kappa,\delta}$ on \emph{every} evaluation sample.
\end{corollary}
\begin{proof}
Projection onto label-measurable functions gives $\tau_\Phi^2\leq \mathbb E_{x\sim\mathbb P_x}\!\left[\|\phi-h^{c(x)}\|^2\right]$. The variable $\|\phi-h^{c(x)}\|^2/2$ lies in $[0,1]$, so Step~2 of Proposition~\ref{prop:exclusion_certificate} bounds its mean by $Q_{h,\delta}$ with failure probability $\delta/2$. Combine with the successor-variation confidence event by a union bound. For one-hot prototypes,
$\|\phi-\mathrm e^{c(x)}\|^2/2\leq(1-\phi_{c(x)})^2$ pointwise; monotonicity of the correction implies $Q_{h,\delta}\leq F_\delta$.
\end{proof}

\subsection{When additional evaluation data make the certificate informative}
\label{sec:certificate_informativeness}
The fixed-sample guarantee controls erroneous exclusion of an error tolerance. We next establish when the certificate eventually becomes informative as sample size grows. Throughout this analysis, the first $M$ observations come from one infinite i.i.d. sequence, and the representation, class count, and evaluation law are fixed. Subscript $M$ denotes sample-dependent quantities.

\begin{proposition}[Class occupancy and downward bias of within-class sample variances]
\label{prop:class_occupancy_bias}\leavevmode\par
Under the assumptions of Proposition~\ref{prop:exclusion_certificate}, let
\[
v_c:=\mathbb{E}_{(x,x^+)\sim\mathbb P}\!\left[\|z-z_*^c\|^2\mid y_c=1\right],\qquad
u_c:=\mathbb{E}_{x\sim\mathbb P_x}\!\left[\|\phi-\mu_\Phi^c\|^2\mid y_c=1\right].
\]
Conditional on the fixed construction,
\begin{equation}
\begin{aligned}
\mathbb{E}_{\mathcal V_M\sim\mathbb P_{\mathcal V_M}}\!\left[\widehat s_M\right]
&=\sigma_{\mathrm{res}}^2-
\frac1M\sum_{c=1}^C v_c\bigl[1-(1-\pi_c)^M\bigr],\\
0\leq\sigma_{\mathrm{res}}^2-
\mathbb{E}_{\mathcal V_M\sim\mathbb P_{\mathcal V_M}}\!\left[\widehat s_M\right]
&\leq\frac{C}{M},\\
\tau_\Phi^2-\mathbb E_{\mathcal V_M\sim\mathbb P_{\mathcal V_M}}[\widehat t_M]
&=\frac1M\sum_{c=1}^Cu_c[1-(1-\pi_c)^M]\in[0,C/M].
\end{aligned}
\label{eq:class_occupancy_bias}
\end{equation}
\end{proposition}
\begin{proof}
Given $n_c=n>0$, the expected within-class sum of squared deviations from the sample mean is $(n-1)v_c$; for $n=0$ it is zero. The count $n_c$ has distribution $\mathbb P_{n_c}=\operatorname{Bin}(M,\pi_c)$, so
\[
\mathbb{E}_{n_c\sim\mathbb P_{n_c}}\!\left[n_c-\mathbbm1_{\{n_c>0\}}\right]
=M\pi_c-[1-(1-\pi_c)^M].
\]
Sum over classes and divide by $M$, using $\sigma_{\mathrm{res}}^2=\sum_c\pi_cv_c$.
The same calculation with $\phi$ in place of $z$ proves the input identity. Both conditional variances are at most $1$ because the vectors lie in $\Delta_C$.
\end{proof}

The number of evaluation observations available within each class matters because the certificate estimates class means. In particular, finite class occupancy affects the empirical successor variance. A singleton class contributes zero variance because its fitted class mean equals its only successor vector; if every observed class is a singleton, then $\widehat s_M=0$ and the certificate is inconclusive. Equation~\eqref{eq:class_occupancy_bias} quantifies the resulting downward bias when class means are fitted on the evaluation sample.

\begin{corollary}[Consistency and detection above a fixed error tolerance]
\label{cor:certificate_consistency}
Under the fixed-construction and i.i.d.-pair assumptions of Proposition~\ref{prop:exclusion_certificate}, for fixed finite $C$, fixed $\kappa\geq0$, and fixed $0<\delta<1$, the certificate satisfies
\begin{equation}
\begin{aligned}
L_{\kappa,\delta,M}&\longrightarrow
\ell_\kappa(\Phi):=
\bigl(\sigma_{\mathrm{res}}-\kappa\sqrt{2A_{\mathrm{fid}}}\bigr)_+,\\
\widetilde L_{\kappa,\delta,M}&\longrightarrow
\widetilde\ell_\kappa(\Phi):=
\bigl(\sigma_{\mathrm{res}}-\kappa\tau_\Phi\bigr)_+
\geq\ell_\kappa(\Phi)
\end{aligned}
\quad\text{a.s. as }M\to\infty.
\label{eq:certificate_consistency}
\end{equation}
For either sequence, if its corresponding population limit exceeds a tolerance $\varepsilon\geq0$, it eventually exceeds that tolerance almost surely; if the limit is smaller, it eventually lies below the tolerance. Equality at the limit does not decide strict finite-sample exclusion.
\end{corollary}
\begin{proof}
The strong law gives $\widehat f_M\to A_{\mathrm{fid}}$, $n_c/M\to\pi_c$, and $\bar z_c\to z_*^c$ for each class of positive probability. Also,
\begin{equation}
\begin{aligned}
\widehat s_M
&=\frac1M\sum_{i=1}^M\|z^i\|^2
-\sum_{c:n_c>0}\frac{n_c}{M}\|\bar z_c\|^2\\
&\longrightarrow
\mathbb{E}_{x^+\sim\mathbb{P}_{x^+}}\!\left[\|\Phi(x^+)\|^2\right]
-\sum_c\pi_c\|z_*^c\|^2
=\sigma_{\mathrm{res}}^2.
\end{aligned}
\label{eq:sample_variance_consistency}
\end{equation}
The same strong-law argument with $\phi$ instead of $z$ yields $\widehat t_M\to\tau_\Phi^2$. Because $C$ is fixed and $r_{\delta,M}\to0$, we have $F_{\delta,M}\to A_{\mathrm{fid}}$, $V_{\delta,M}\to\sigma_{\mathrm{res}}^2$, and $T_{\delta,M}\to\tau_\Phi^2$. Continuity of the square root and positive-part functions proves both limits. A strict gap from $\varepsilon$ then proves the eventual inequalities.
\end{proof}
More independent data shrink the confidence corrections, but they do not remove a gap between a structural population lower bound and the actual optimum $\mathcal D_\kappa(\Phi)$. This is a limitation of the inequality, not of sample size. The centered population limit is at least as strong as the label-based limit. The consistency statements assume fixed $C$, coordinates, and evaluation law. 

Both prediction-risk certificates apply at a prespecified sample size. Stopping as soon as the certificate exceeds the tolerance requires a confidence guarantee that covers all inspected sample sizes. For example, allocating failure probabilities $\delta_M=\delta_0/[M(M+1)]$, where $\delta_0\in(0,1)$, gives total failure probability at most $\delta_0$ by a union bound. Since $\log(2/\delta_M)/M\to0$, the same consistency proof applies. 

Appendix~\ref{app:fidelity_concentration} gives the label-based certificate for i.i.d. trajectories, allowing dependent transitions within each trajectory. It gives every trajectory equal weight and uses the number of trajectories in the confidence correction. The class-occupancy identity in~\eqref{eq:class_occupancy_bias} specifically concerns i.i.d. pairs.
\subsection{Comparison with certificates based on empirical risk minimization}
\label{sec:certificate_comparison}
The bounds based on class-wise variation avoid fitting a prediction matrix to the evaluation sample. We call these \emph{structural certificates} because they use class-wise statistics rather than minimizing prediction error over matrices on the evaluation sample. Is that computational advantage accompanied by reduced tightness? To place the trade-off in context, we compare with a lower confidence bound obtained by \emph{constrained empirical risk minimization} (ERM). Distribution-free minimum-risk certification and its limitations are studied by \citet{muller2025falsification,wu2026minimalrisk}. Under the same independent-transition-pair assumptions as Proposition~\ref{prop:exclusion_certificate}, the comparator minimizes sample prediction loss before applying a confidence correction. For a \emph{prespecified} spectral-norm limit $\kappa$, let
\begin{equation}
\widehat R^*_{\kappa,M}:=
\min_{\|B\|_2\leq\kappa}\frac1M\sum_{i=1}^M
\|z^i-B^\top\phi^i\|^2.
\label{eq:uniform_regression_comparison}
\end{equation}
Every squared loss lies in $[0,(1+\kappa)^2]$, so replacing one pair changes $\widehat R^*_{\kappa,M}$ by at most $(1+\kappa)^2/M$. Moreover, minimizing before averaging gives $\mathbb E_{\mathcal V_M\sim\mathbb P_{\mathcal V_M}}[\widehat R^*_{\kappa,M}]\leq\mathcal D_\kappa(\Phi)^2$. The upper-tail form of McDiarmid's inequality consequently yields, with probability at least $1-\delta$,
\begin{equation}
\boxed{\quad\mathcal D_\kappa(\Phi)\geq
\left[\widehat R^*_{\kappa,M}
-(1+\kappa)^2\sqrt{\frac{\log(1/\delta)}{2M}}\right]_+^{1/2}.\quad}
\label{eq:optimized_empirical_lower_bound}
\end{equation}
The confidence correction does not depend on the number of matrix parameters. Computing the empirical minimum nevertheless requires solving a convex, spectral-norm-constrained least-squares problem or obtaining a certified lower bound on its optimal value. The empirical risk of any fitted matrix satisfying the spectral-norm constraint gives an \emph{upper}, not a lower, bound on that minimum. Unlike the two structural certificates, the displayed comparator does not by itself hold simultaneously for spectral-norm limits $\kappa$ selected after inspecting the data.

For the four-state construction at $\kappa=\sqrt2$, the empirical optimum in~\eqref{eq:uniform_regression_comparison} has the exact simplification $\widehat R^*_{\sqrt2,M}=\widehat s$ for every state-count sample. For the nonidentity four-state transition map, write $v=(p,1-p)^\top$, $u=(1-p,p)^\top$, and $t=n_a/(n_a+n_b)$ for a nonempty first class (take any $t\in[0,1]$ when it is empty). The feasible predictor $B^\top=tI+(1-t)u\mathbf1^\top$ maps $v$ and $u$ to the empirical class successor means and satisfies $\|B\|_2\leq\sqrt2$. Every input is constant within its class, so no predictor can improve on these means. Under identity dynamics the feasible identity matrix attains zero empirical risk. Thus the archived four-state counts determine the ERM comparator exactly, without solving a numerical optimization problem. This special structural identity is not available for general soft-coordinate evaluations.

Once the coordinates are available, the label-based, prototype-based, and centered certificates require $O(MC)$ arithmetic operations. None requires optimizing a prediction matrix, and each confidence event applies simultaneously to every $\kappa\geq0$. These properties distinguish the class-wise variation bounds from the ERM-based comparator. Nevertheless, the class-wise bounds can be conservative. Indeed, at the sample level the same projection and triangle arguments give $L_{\kappa,\delta}\leq\sqrt{\widehat R^*_{\kappa,M}}$ and $\widetilde L_{\kappa,\delta}\leq\sqrt{\widehat R^*_{\kappa,M}}$: neither can certify a positive lower bound if the evaluated sample is interpolated by a feasible linear matrix. This inability to certify positive risk on interpolated evaluation data is related to the limits of model-class falsification studied by \citet{muller2025falsification}.

RKHS Koopman operator regression gives complementary prediction and estimation guarantees under its stated stationary or mixing assumptions~\citep{kostic2022regression}. Those results primarily control Koopman-operator prediction or estimation errors and depend on their specified operator classes and sampling assumptions. Section~\ref{sec:certificate_four_state_evaluation} compares the structural and ERM-based bounds at matched confidence levels on the four-state family. The relative tightness of the bounds can differ for other representations.

\subsection{Interpreting and using the certificate} \label{sec:certificate_use}
For a preselected representation and evaluation law, the certificate is an exclusion test. Choose the largest predictor norm $\kappa$ of interest and the required RMSE tolerance $\varepsilon\geq0$ before interpreting evaluation outcomes. Group independent transition pairs by their current reference class, compute $\widehat s$ and either $\widehat f$ or the centered statistic $\widehat t$, and apply the stated confidence correction. Compare the resulting lower confidence bound with $\varepsilon$, using a joint confidence allocation if several certificate types are selected after inspection. The label-based certificate exceeds $\varepsilon$ precisely when
\begin{equation}
\widehat s>\sqrt{2r_\delta}
+\bigl(\varepsilon+\kappa\sqrt{2F_\delta}\bigr)^2.
\label{eq:observable_exclusion_condition}
\end{equation}
Strict satisfaction excludes that tolerance for all permitted matrices; a zero certificate is inconclusive. Under stochastic dynamics, variation of the realized successor around its conditional mean contributes to prediction error even if the conditional mean evolves linearly. A separately certified upper bound for a feasible fitted matrix can establish attainable accuracy; combining upper and lower confidence statements or selecting a new representation requires appropriate joint coverage or fresh evaluation data.

% ===== END sections/02_prediction_limits.tex =====

% ===== BEGIN sections/04_reconstruction_background.tex =====
\section{KAHM Reconstruction Scores and Koopman Operator Preliminaries}\label{sec:mathematical_background}
The general certificate in Section~\ref{sec:prediction_limits} applies to arbitrary fixed soft coordinates. We describe the \emph{reconstruction-based} coordinates used in the examples and experiments. A KAHM reconstructs a query from reference states; its reconstruction discrepancy becomes an affinity for a reference class. We then introduce the operator-theoretic conventions needed to interpret exact, rather than approximate, coordinate evolution. The KAHM reconstruction comes from prior work, and neither it nor the Koopman preliminaries is required to prove the general statistical certificate.
\subsection{Projection and function-space conventions}
An overline denotes complex conjugation. For complex vectors, we use the inner product $\langle u,v\rangle=u^\top\overline v$, which is linear in its first argument, and the associated norm $\|u\|^2=u^\top\overline u$. For real matrices, the Frobenius inner product is $\langle A,B\rangle_F=\operatorname{tr}(A^\top B)$. Subscripts $i,:$, $:,j$, and $i,j$ select a row, column, and entry. Inequalities between vectors or matrices are component-wise; $W\succ0$ means that $W$ is symmetric positive definite. The notation $F^m$ denotes $m$ successive applications of $F$, $\otimes$ denotes a product measure or product $\sigma$-algebra as appropriate, and $\oplus$ denotes a direct sum.

The Euclidean projection of $u\in\mathbb R^C$ onto $\Delta_C$ is
\begin{equation}
\label{eq:simplex_projection_operator}
\Proj_{\Delta_C}(u):=\mathop{\argmin}_{p\in\Delta_C}\|p-u\|^{2}.
\end{equation}
The projection is unique because $\Delta_C$ is nonempty, closed, and convex~\citep{bauschke2017}. We write $\operatorname{aff}(S)$ and $\spanop S$ for an affine hull and a linear span. For a finite-dimensional subspace $V$, $V^\perp$ is its orthogonal complement. Its orthogonal projection $P_V$ is characterized by $P_Vw\in V$ and $w-P_Vw\in V^\perp$~\citep{conway1990,bauschke2017}.

A reproducing kernel Hilbert space (RKHS) $\mathcal H_k(\mathcal X)$ is a Hilbert space of functions in which the kernel section $k(\cdot,x)$ belongs to the space and reproduces evaluation at $x$:
\[
\langle f,k(\cdot,x)\rangle_{\mathcal H_k}=f(x),
\qquad f\in\mathcal H_k,\quad x\in\mathcal X.
\]
Here a \emph{kernel section} is the function obtained by fixing the kernel's second argument~\citep{aronszajn1950,paulsen2016}. We use the same first-argument-linear convention for complex-valued RKHSs.
\subsection{Kernel Affine Hull Machine (KAHM)}
Let $\mathbf{X}=\left[\begin{IEEEeqnarraybox*}[][c]{,c/c/c,}x^1&\cdots&x^N\end{IEEEeqnarraybox*}\right]^{\top}\in\mathbb{R}^{N\times n}$ be a reference dataset. Appendix~\ref{app:kahm_construction} constructs a linear encoding $\mathbf P_{\mathbf X}$ and kernel-regression functions $h_{\mathbf X}^i$, one for each reference sample. Their outputs determine reconstruction weights. The reconstruction is defined on
\begin{equation}
\Omega_{\mathbf X}:=\left\{x\in\mathbb R^n:
\sum_{i=1}^{N}h_{\mathbf X}^{i}(\mathbf P_{\mathbf X}x)\neq0\right\}.
\end{equation}
A KAHM~\citep{kumar2024kahm,kumar2025collaborative,kumar2025operator} represents a query $x\in\Omega_{\mathbf X}$ as an affine combination of reference samples:
\begin{IEEEeqnarray}{rCl}
\label{eq_220420251843}
\mathcal{A}_{\mathbf{X}}(x)
&:=&
\sum_{i=1}^N w_i(x)x^i,
\qquad
w_i(x):=\frac{h_{\mathbf X}^i(\mathbf P_{\mathbf X}x)}
{\sum_{j=1}^N h_{\mathbf X}^j(\mathbf P_{\mathbf X}x)}.
\end{IEEEeqnarray}
Regularized kernel regression estimates each $h_{\mathbf X}^i$ using target $1$ at encoded sample $i$ and target $0$ at the other samples. Regularization allows deviations from the zero-one regression targets. After normalization, the reconstruction weights sum to one, placing $\mathcal A_{\mathbf X}(x)$ in the \emph{affine hull} of the samples. Unlike the final soft class weights, these reconstruction coefficients may be negative. The distinction is important: the KAHM reconstructs a state using possibly signed sample weights, while the abstraction assigns nonnegative weights to reference \emph{classes}.
\subsection{KAHM reconstruction score}
To compare reconstruction quality across classes, we use a score in $[0,1]$ that combines how far a state lies from its reconstruction with how much their directions differ. A low score means a better match under this specified combination of Euclidean and angular discrepancy.
\begin{definition}[KAHM reconstruction score~\citep{kumar2025operator}]\label{def_space_folding_measure}\leavevmode\par
For a reference dataset $\mathbf{X}$, define the admissible domain
\begin{equation}
D_{\mathbf{X}}
:=\{x\in\Omega_{\mathbf X}:\|x\|>0\ \text{and}\ \|\mathcal{A}_{\mathbf{X}}(x)\|>0\}.
\end{equation}
For $x\in D_{\mathbf{X}}$, define the reconstruction score $\mathcal{T}_{\mathbf{X}}:D_{\mathbf{X}}\to[0,1]$ by
\begin{IEEEeqnarray}{rCl}
\label{eq_170220251438}
\mathcal{T}_{\mathbf{X}}(x)
&:=&
\sqrt{\frac{1}{2}\left(\left|\mathcal{T}_{\mathbf{X}}^{\mathrm{Euc}}(x)\right|^{2}+\left|\mathcal{T}_{\mathbf{X}}^{\mathrm{Cos}}(x)\right|^{2}\right)},
\end{IEEEeqnarray}
where
\begin{IEEEeqnarray}{rCl}
\mathcal{T}_{\mathbf{X}}^{\mathrm{Euc}}(x)
&:=&1-\exp\!\left(-\|x-\mathcal{A}_{\mathbf{X}}(x)\|\right),\\
\mathcal{T}_{\mathbf{X}}^{\mathrm{Cos}}(x)
&:=&\frac{1}{\pi}\arccos\!\left(
\frac{\mathcal{A}_{\mathbf{X}}(x)^{\top}x}
{\|\mathcal{A}_{\mathbf{X}}(x)\|\,\|x\|}
\right).
\end{IEEEeqnarray}
\end{definition}
The Euclidean component lies in $[0,1)$, and the angular component lies in $[0,1]$, so the combined score lies in $[0,1]$. The nonzero-norm conditions in $D_{\mathbf X}$ ensure that the angle is defined. Prior KAHM work calls this quantity the ``space-folding measure'' \citep{kumar2025operator}; we use \emph{reconstruction score} to describe its role here. A smaller score means a smaller combined discrepancy. The two components can trade off: a decrease in one may offset an increase in the other.

Because the Euclidean component depends on coordinate scale and the angular component on the origin and relative scaling, the reconstruction score is not generally invariant under state-coordinate transformations. Rescaling the reference data can also change the fitted KAHM. The induced soft coordinates and prediction results must therefore be interpreted relative to the fixed state-coordinate convention of each benchmark.

The reconstruction-specific analysis requires $x\in D_{\mathbf X}$. Section~\ref{sec:exp_protocol} and Appendix~\ref{app:protocol} specify the numerical extensions outside this domain and the threshold used when normalizing very small affinities. These conventions are part of the implemented representation.
\subsection{Koopman Operators and Reproducing-Kernel Hilbert Spaces}
Let \(F:\mathcal{X}\rightarrow\mathcal{X}\) be a measurable discrete-time state map, and let
\(\mathcal{H}_{\mathcal{K}}(\mathcal{X})\) denote the RKHS associated with the
positive semidefinite kernel
\(\mathcal{K}:\mathcal{X}\times\mathcal{X}\rightarrow\mathbb{R}\).
An observable is a scalar-valued function of the state. The Koopman operator
maps an observable to its value one step into the future, expressed as a function
of the current state~\citep{koopman1931,budisic2012,das2020}. For an observable \(g\in\mathcal{H}_{\mathcal{K}}(\mathcal{X})\)
such that \(g\circ F\in\mathcal{H}_{\mathcal{K}}(\mathcal{X})\), its Koopman
image is defined by
\begin{equation}
    (K_F g)(x) := g(F(x)),
    \qquad x\in\mathcal{X}. \label{eq:koopman_operator_background}
\end{equation}
The domain therefore contains exactly the RKHS functions whose composition with $F$ remains in the RKHS:
\begin{equation}
    \mathcal{D}(K_F)
    :=
    \left\{
        g\in\mathcal{H}_{\mathcal{K}}(\mathcal{X})
        \;:\;
        g\circ F\in\mathcal{H}_{\mathcal{K}}(\mathcal{X})
    \right\}.
\end{equation}
If composition with \(F\) preserves the RKHS, namely,
\begin{equation}
    g\in\mathcal{H}_{\mathcal{K}}(\mathcal{X})
    \quad\Longrightarrow\quad
    g\circ F\in\mathcal{H}_{\mathcal{K}}(\mathcal{X}),
\end{equation}
then composition maps every function in the RKHS back into the same space. This property is called Koopman invariance, and it gives
\(\mathcal{D}(K_F)=\mathcal{H}_{\mathcal{K}}(\mathcal{X})\). Our adjoint statements concern a Koopman operator that is defined and bounded on the whole RKHS. In the finite-dimensional space of Section~\ref{sec:method}, exact closure ensures both properties. Under these conditions,
\(K_F\) admits a unique bounded adjoint
\(K_F^{*}:\mathcal{H}_{\mathcal{K}}(\mathcal{X})
\rightarrow\mathcal{H}_{\mathcal{K}}(\mathcal{X})\), characterized by
the Hilbert-space adjoint relation~\citep{conway1990}
\begin{equation}
    \left\langle K_F g,h\right\rangle_{\mathcal{H}_{\mathcal{K}}(\mathcal{X})}
    =
    \left\langle g,K_F^{*}h\right\rangle_{\mathcal{H}_{\mathcal{K}}(\mathcal{X})},
    \qquad
    g,h\in\mathcal{H}_{\mathcal{K}}(\mathcal{X}). \label{eq:koopman_adjoint_definition_background}
\end{equation}
\begin{proposition}\label{proposition_300620260944}
Assume $K_F$ is bounded on $\mathcal{H}_{\mathcal{K}}(\mathcal{X})$. Then, for every $x\in\mathcal{X}$,
\begin{equation}
K_F^*\mathcal{K}(\cdot,x)=\mathcal{K}(\cdot,F(x)).
\label{eq:adjoint_kernel_section_identity_background}
\end{equation}
\end{proposition}
\begin{proof}
The proof is given in Appendix~\ref{app:adjoint_kernel_proof}.
\end{proof}
Real kernels and coordinates can still produce complex eigenvalues and eigenfunctions. For the spectral analysis in Section~\ref{sec:method}, we therefore extend the real RKHS to complex-valued functions by standard complexification.
\begin{definition}[Complexified RKHS]\label{def_050720262001}
The standard complexification of $\mathcal{H}_{\mathcal{K}}(\mathcal{X})$~\citep{conway1990} is
\begin{IEEEeqnarray}{rCl}
\mathcal{H}_{\mathcal{K}}^{\mathbb{C}}(\mathcal{X})
&:=&\{f_1+\mathrm{i}f_2\mid f_1,f_2\in\mathcal{H}_{\mathcal{K}}(\mathcal{X})\},
\end{IEEEeqnarray}
with inner product, for $f=f_1+\mathrm{i}f_2$ and $g=g_1+\mathrm{i}g_2$,
\begin{IEEEeqnarray}{rCl}
\label{eq_050720262018}
\left\langle f,g\right\rangle_{\mathcal{H}_{\mathcal{K}}^{\mathbb{C}}(\mathcal{X})}
&:=&
\left\langle f_1,g_1\right\rangle_{\mathcal{H}_{\mathcal{K}}(\mathcal{X})}
+\left\langle f_2,g_2\right\rangle_{\mathcal{H}_{\mathcal{K}}(\mathcal{X})}
+\mathrm{i}\left(
\left\langle f_2,g_1\right\rangle_{\mathcal{H}_{\mathcal{K}}(\mathcal{X})}
-\left\langle f_1,g_2\right\rangle_{\mathcal{H}_{\mathcal{K}}(\mathcal{X})}
\right).
\end{IEEEeqnarray}
The reproducing property extends to the complexification:
\begin{equation}
f(x)=\left\langle f,\mathcal{K}(\cdot,x)\right\rangle_{\mathcal{H}_{\mathcal{K}}^{\mathbb{C}}(\mathcal{X})},
\qquad x\in\mathcal{X}.
\end{equation}
\end{definition}

When $K_F$ is bounded on $\mathcal H_{\mathcal K}(\mathcal X)$, its complexification is
\begin{IEEEeqnarray}{rCl}
K_{F,\mathbb{C}}(f_1+\mathrm{i}f_2)
&:=&K_Ff_1+\mathrm{i}K_Ff_2,
\qquad f_1,f_2\in\mathcal{H}_{\mathcal{K}}(\mathcal{X}),
\end{IEEEeqnarray}
and is bounded on $\mathcal{H}_{\mathcal{K}}^{\mathbb{C}}(\mathcal{X})$. Its adjoint satisfies the same kernel-section identity:
\begin{equation}
K_{F,\mathbb{C}}^*\mathcal{K}(\cdot,x)=\mathcal{K}(\cdot,F(x)).
\label{eq_020720261144}
\end{equation}
The identity follows by the same reproducing-kernel argument as Proposition~\ref{proposition_300620260944}.
\begin{definition}[Adjoint Koopman Eigenpair]\label{def_020720160714}
A pair $(\lambda,f)\in\mathbb{C}\times\mathcal{H}_{\mathcal{K}}^{\mathbb{C}}(\mathcal{X})$ is an adjoint Koopman eigenpair if
\begin{equation}
f\neq0,
\qquad
K_{F,\mathbb{C}}^*f=\lambda f.
\end{equation}
\end{definition}

% ===== END sections/04_reconstruction_background.tex =====

% ===== BEGIN sections/05_abstraction_closure.tex =====
\section{Reconstruction-Based Soft Coordinates and Conditional Koopman Closure}\label{sec:method}
This section formalizes the normalized KAHM affinities introduced in Section~\ref{sec:intro_kahm_overview} and characterizes their exact linear evolution when it exists. Under pointwise deterministic closure, the coordinate functions span a finite-dimensional Koopman-invariant space; the RKHS adjoint and spectral representation must account for redundant coefficient directions. We treat fitted prediction residuals and spectral diagnostics separately from the lower confidence bounds on minimum population risk.
\subsection{Data, reference classes, and transition samples}
Let $\mathcal X\subset\mathbb R^n$ be a measurable state space and let $x\sim\mathbb P_x$. A measurable reference-class map
\[
c:\mathcal X\rightarrow\{1,\ldots,C\}
\]
defines the one-hot label $y(x):=\mathrm e^{c(x)}$, with a $1$ at the assigned class and zeros elsewhere. Together with $\mathbb P_x$, it determines the joint distribution $\mathbb P_{x,y}$. The classes may be predefined or learned from training data. Each analysis fixes the class count $C$, the assignment map $c$, and the reference datasets. Model selection may compare alternatives, but the subsequent analysis concerns the selected representation.

For the operator results, let $F:\mathcal X\rightarrow\mathcal X$ be a measurable deterministic one-step map and write $x^+=F(x)$. The requirement $F(\mathcal X)\subseteq\mathcal X$ ensures that successive coordinates are defined on the same state space. For stochastic transitions, including those generated by a fixed randomized policy, the analogous linear condition predicts the conditional mean:
\[
\mathbb{E}_{x^+\sim\mathbb P(\cdot\mid x)}\!\left[\Phi(x^+)\right]=B^\top\Phi(x).
\]
This conditional-mean equation allows individual successors to differ from the predicted mean. Remark~\ref{rem:stochastic_risk} gives the corresponding decomposition of one-step prediction error. The Koopman operator identities in Section~\ref{sec:finite_koopman_closure} assume deterministic $F$.
\begin{definition}[Training transitions and reference datasets]\label{definition_200720261153}
Let
\begin{IEEEeqnarray}{rCl}
\label{eq_290620261445}
\mathcal S:=\{(x^i,y^i,x^{+,i})\}_{i=1}^N, \qquad x^{+,i}=F(x^i),
\end{IEEEeqnarray}
be a finite collection of observed one-step transitions, where $x^i$, $x^{+,i}$, and $y^i=y(x^i)$ are realizations of $x$, $x^+$, and $y$, respectively. The labels are supplied by the fixed reference-class map. The representation construction and the identities for a fixed number of matrix updates remain valid with dependent training transitions. Each statistical guarantee separately states the sampling assumptions it requires. For each reference class $c\in\{1,\cdots,C\}$, define
\[
\mathcal I^c:=\{i\in\{1,\cdots,N\}:y^i=\mathrm e^c\},\qquad N_c:=|\mathcal I^c|,
\]
and let $\mathrm I^c=(\mathrm I_1^c,\cdots,\mathrm I_{N_c}^c)$ denote the elements of $\mathcal I^c$ in increasing order. The class-$c$ reference matrix is
\begin{IEEEeqnarray}{rCl}
\label{eq_200720261205}
\mathbf X^c:=\left[x^{\mathrm I_1^c}\ \cdots\ x^{\mathrm I_{N_c}^c}\right]^\top\in\mathbb{R}^{N_c\times n}.
\end{IEEEeqnarray}
Each KAHM dataset must contain at least two samples. Their encoded vectors must be pairwise distinct, have at least one coordinate, and have a positive-definite sample covariance (Appendix~\ref{app:kahm_construction}). Reference datasets are fixed before coordinate construction; Section~\ref{sec:exp_protocol} explains how experiments augment classes containing one sample.
\end{definition}
\begin{definition}[$t$-Step Evaluation Dataset]
For a fixed $t\in\mathbb{Z}_+$, let
\begin{IEEEeqnarray}{rCl}
\label{eq_200720262025}
\mathrm E_t:=\{(x'^i,F^t(x'^i))\}_{i=1}^M
\end{IEEEeqnarray}
be an evaluation set independent of the data used to construct and select the abstraction and fit the closure model.
\end{definition}
\begin{definition}[Spectral Evaluation Dataset]
Let
\begin{equation}
\mathrm E_{\mathrm{spec}}:=\{(x'^i,F(x'^i))\}_{i=1}^M
\end{equation}
be a one-step dataset kept independent of representation construction, model selection, and matrix fitting. It supplies the spectral diagnostics and may also serve as a one-step prediction-evaluation set. Independence of this dataset does not, by itself, give a statistical guarantee for a spectral diagnostic.
\end{definition}
\subsection{Construction of the soft abstraction}
The state is reconstructed separately from each reference class. Classes with smaller reconstruction scores receive larger positive affinities; normalizing across classes produces one soft-coordinate vector with nonnegative entries summing to one. The modified affinity rule in~\eqref{eq_220720260633} follows the approach of \citet{kumar2026semantic}, retains all classes, adds a positive offset $\tau$, and allows a tunable exponent $\omega$. These coordinates express relative reconstruction quality rather than posterior class probabilities.
\begin{definition}[Family of KAHM Soft Abstractions]
Given $\mathcal S$ and fixed measurable class scores $\mathcal T_c:\mathcal X\to[0,1]$, define
\begin{IEEEeqnarray}{rCl}
\label{eq_220720260633}
\mathcal G_{\mathcal S}:=\Bigg\{
\Phi:\mathcal X\rightarrow\Delta_C\;\Bigg|\;
(\Phi(x))_c=
\frac{\bigl(1-\mathcal T_c(x)+\tau\bigr)^\omega}
{\sum_{j=1}^C\bigl(1-\mathcal T_j(x)+\tau\bigr)^\omega},
\quad c=1,\cdots,C,\ \omega>0,\ \tau>0
\Bigg\}.
\end{IEEEeqnarray}
The offset $\tau$ keeps every affinity positive. The exponent $\omega$ controls how strongly the coordinates distinguish classes with different reconstruction scores.
\end{definition}
To reduce the computational cost of KAHM fitting on large datasets, we follow prior work~\citep{kumar2025collaborative,kumar2025operator} and partition each $\mathbf X^c$ into fixed subsets
$\mathbf X_s^c$, $s=1,\ldots,S_{\mathrm{sub}}^c$. Every subset must satisfy the KAHM admissibility requirements in Definition~\ref{definition_200720261153}. On the common score domain $\mathcal X\subseteq\bigcap_{c,s}D_{\mathbf X_s^c}$, set
\begin{align*}
\mathcal T_c(x)
&:=\min_{1\leq s\leq S_{\mathrm{sub}}^c}
\mathcal T_{\mathbf X_s^c}(x),\\
s_c(x)
&:=\min\operatorname*{argmin}_{1\leq s\leq S_{\mathrm{sub}}^c}
\mathcal T_{\mathbf X_s^c}(x),
&
\mathcal A_c(x)
&:=\mathcal A_{\mathbf X_{s_c(x)}^c}(x).
\end{align*}
For each query, $s_c(x)$ selects the submodel with the smallest combined score; ties are resolved by the smallest index. The Euclidean and angular components both use that submodel's reconstruction $\mathcal A_c(x)$. When $S_{\mathrm{sub}}^c=1$, the class has a single KAHM. A finite minimum of measurable scores is measurable, so this construction supplies the class scores required in~\eqref{eq_220720260633}.

When states outside the common score domain are used, the score extension must be specified and fixed as part of the representation. The implementation also uses the normalization threshold in~\eqref{eq:implemented_affinity_normalization}, which can change the exact coordinates even inside that domain. The RKHS and linear-prediction results apply to any fixed measurable map $\Phi:\mathcal X\to\Delta_C$, including that implemented map. Results that use the reconstruction formula retain their stated score-domain and normalization assumptions.
\begin{remark}[Reference partitions and candidate families]
\label{rem:conditional_abstraction_class}
States with the same label $c(x)$ form a reference class. Fixing these classes, their reference datasets, and internal subpartitions determines $\mathcal G_{\mathcal S}$; varying $\omega$ or $\tau$ selects a map within this family. Model selection compares maps from this or alternative reference constructions.
\end{remark}
For every finite $\omega>0$, the map in~\eqref{eq_220720260633} is well defined whenever the reconstruction scores are measurable and take values in $[0,1]$. Indeed,
\[
\tau\le 1-\mathcal T_c(x)+\tau\le 1+\tau,
\]
so the denominator is at least $C\tau^\omega>0$. Thus $\Phi$ is measurable, every coordinate is strictly positive, and the coordinates sum to one. Writing $a_c=1-\mathcal T_c(x)+\tau$, the ratio of coordinates $c$ and $j$ is $(a_c/a_j)^\omega$. Increasing $\omega$ therefore increases their contrast when the affinities differ. For fixed $C$ and $\tau>0$, the bounded affinity range also implies $\Phi(x)\to C^{-1}\mathbf 1_C$ uniformly on $\mathcal X$ as $\omega\downarrow0$.
\begin{remark}[Interpretation of a coordinate]
Coordinate $c$ uses the smallest reconstruction score among the submodels in class $c$. Normalization compares this affinity with those of all other classes. For $C>1$, increasing $1-\mathcal T_c(x)$ while fixing competing affinities increases $(\Phi(x))_c$. Thus a larger coordinate means better relative reconstruction by that class's reference data.
\end{remark}
Once $\Phi$ has been fixed, define the successor abstraction vector
\begin{IEEEeqnarray}{rCl}
\label{eq_240720262158}
z &:=& \Phi(F(x))=\Phi(x^+),
\end{IEEEeqnarray}
and, for the raw transitions in $\mathcal S$, let $z^i:=\Phi(x^{+,i})$. The training set for the linear predictor is therefore
\begin{equation}
\mathrm E_{\mathrm{trn}}:=\left \{\left(x^i,z^i = \Phi(F(x^i))\right) \right\}_{i=1}^N.
\end{equation}
We construct the coordinate map from labelled current-state samples, encode the successors with that fixed map, and then fit the linear coordinate predictor.
\subsection{Finite-rank kernel and spectral correspondence}
Koopman theory describes how scalar functions of state evolve under the dynamics. Here those functions are linear combinations of the soft coordinates, but different coefficient vectors can represent exactly the same function. For example, if every coordinate equals $1/C$ and $C\geq2$, then the nonzero vector $w=\mathrm e^1-\mathrm e^2$ satisfies $w^\top\Phi(x)=0$ at every state. Thus a matrix eigenvector can represent the zero function. We remove this ambiguity by restricting coefficients to the span of all feature vectors. This gives one coefficient vector for each function and a reduced matrix for the Koopman operator when exact closure holds.

Fix a measurable $\Phi:\mathcal X\to\Delta_C$, such as a map in $\mathcal G_{\mathcal S}$. It defines the positive-semidefinite finite-rank kernel
\begin{IEEEeqnarray}{rCl}
\mathcal K(x,x')&:=&\Phi(x)^\top\Phi(x'),
\qquad x,x'\in\mathcal X.
\label{eq_040720260823}
\end{IEEEeqnarray}
The redundant directions are precisely the vectors orthogonal to every $\Phi(x)$. To remove them, define the span of all feature vectors and its orthogonal projection by
\begin{IEEEeqnarray}{rCl}
V_\Phi&:=&\spanop\{\Phi(x):x\in\mathcal X\}\subseteq\mathbb{R}^C,\\
P_\Phi&:=&P_{V_\Phi}:\mathbb{R}^C\rightarrow V_\Phi.
\end{IEEEeqnarray}
By the definition of orthogonal projection, $w-P_\Phi(w)\in V_\Phi^\perp$. Since $\Phi(x)\in V_\Phi$ for every $x\in\mathcal X$,
\begin{equation}
w^\top\Phi(x)=P_\Phi(w)^\top\Phi(x),
\qquad x\in\mathcal X.
\end{equation}
Thus $w$ and $P_\Phi(w)$ represent the same function. Equivalently, coefficient vectors are identified whenever their difference belongs to $V_\Phi^\perp$; $V_\Phi$ supplies the unique minimum-norm representative of each such equivalence class. Restricting coefficients to $V_\Phi$ gives the RKHS of~\eqref{eq_040720260823}:
\[
\mathcal H_\Phi
:=\{f_w:\mathcal X\rightarrow\mathbb{R}\mid f_w(x)=w^\top\Phi(x),\ w\in V_\Phi\},
\]
with inner product and norm
\[
\langle f_w,f_{w'}\rangle_{\mathcal H_\Phi}:=w^\top w',
\qquad
\|f_w\|_{\mathcal H_\Phi}:=\|w\|,
\qquad w,w'\in V_\Phi.
\]
To verify uniqueness, suppose $w,w'\in V_\Phi$ represent the same function. Their difference is then orthogonal to every $\Phi(x)$, so it also belongs to $V_\Phi^\perp$. Consequently, $w-w'\in V_\Phi\cap V_\Phi^\perp=\{\mathbf{0}_C\}$. The inner product is therefore well defined. Moreover, $\mathcal K(\cdot,x)=f_{\Phi(x)}\in\mathcal H_\Phi$ and, for every $f_w\in\mathcal H_\Phi$,
\[
\langle f_w,\mathcal K(\cdot,x)\rangle_{\mathcal H_\Phi}
=w^\top\Phi(x)
=f_w(x),
\]
which verifies the reproducing property. The RKHS dimension, and hence the kernel rank, is $r_\Phi:=\dim V_\Phi\le C$. Let $\mathcal H_\Phi^{\mathbb{C}}$ denote the complexification defined in Definition~\ref{def_050720262001}. The corresponding complex coefficient space and projection are
\begin{IEEEeqnarray}{rCl}
V_\Phi^{\mathbb{C}}&:=&V_\Phi\oplus \mathrm i V_\Phi\subseteq\mathbb{C}^C,\\
P_\Phi^{\mathbb{C}}(w^1+\mathrm i w^2)&:=&P_\Phi(w^1)+\mathrm i P_\Phi(w^2).
\end{IEEEeqnarray}
Thus, as in the real case, only the component of a complex coefficient vector lying in $V_\Phi^{\mathbb{C}}$ affects the represented function. In particular, the RKHS norm is the norm of that projected coefficient. Adding a redundant coefficient direction changes neither the function nor its RKHS norm.

\begin{proposition}\label{proposition_070720261027}
For $w\in\mathbb{C}^C$, define
\begin{IEEEeqnarray}{rCl}
\label{eq_220820261131}
f_{w}(x)&:=&w^\top\Phi(x).
\end{IEEEeqnarray}
Then, for all $w,w'\in\mathbb{C}^C$,
\begin{IEEEeqnarray}{rCl}
\label{eq_060720261925}
f_{w}&=&f_{P_\Phi^{\mathbb{C}}(w)},\\
f_{w}&\in&\mathcal H_\Phi^{\mathbb{C}},\\
\|f_{w}\|_{\mathcal H_\Phi^{\mathbb{C}}}^2&=&\|P_\Phi^{\mathbb{C}}(w)\|^2,\\
\langle f_{w},f_{w'}\rangle_{\mathcal H_\Phi^{\mathbb{C}}}
&=&\bigl(P_\Phi^{\mathbb{C}}(w)\bigr)^\top
\overline{P_\Phi^{\mathbb{C}}(w')}.
\end{IEEEeqnarray}
\end{proposition}
\begin{proof}
The proof is provided in Appendix~\ref{app:feature_rkhs_proof}.
\end{proof}

\begin{proposition}[Projection onto the span of feature vectors]\label{proposition_170720262159}
Let $Q_\Phi\in\mathbb{R}^{C\times r_\Phi}$ have orthonormal columns spanning $V_\Phi$, where $r_\Phi=\dim V_\Phi$. Then
\begin{IEEEeqnarray}{rCl}
\label{eq_170720261001}
P_\Phi^{\mathbb{C}}(w)&=&Q_\Phi Q_\Phi^\top w,
\qquad w\in\mathbb{C}^C.
\end{IEEEeqnarray}
\end{proposition}
\begin{proof}
The proof is provided in Appendix~\ref{app:projection_proof}.
\end{proof}
The next proposition assumes that a single matrix predicts $\Phi(F(x))$ exactly for every $x\in\mathcal X$. It follows that $f\circ F\in\mathcal H_\Phi$ for every $f\in\mathcal H_\Phi$. The Koopman operator restricted to this finite-dimensional space is then bounded and has an adjoint with respect to the RKHS inner product. Specifying the inner product matters: an adjoint defined using an $L^2$ inner product can be a different operator.
\begin{proposition}[Spectral Correspondence under Exact Closure]\label{prop_koopman_eigenfunction}\leavevmode\par
Assume that the soft-abstraction feature map is exactly closed under a matrix $B\in\mathbb{R}^{C\times C}$:
\begin{equation}
\Phi(F(x))=B^\top\Phi(x),\qquad x\in\mathcal X.
\label{eq:background_feature_closure}
\end{equation}
Let $Q_\Phi\in\mathbb{R}^{C\times r_\Phi}$ have orthonormal columns spanning $V_\Phi$, and define
\begin{equation}
\label{eq:background_reduced_matrices_QPhi}
\widetilde B_\Phi:=Q_\Phi^\top BQ_\Phi.
\end{equation}
For $\lambda\in\mathbb C$, the following statements hold:
\begin{enumerate}
\item for every $w\in\mathbb{C}^C$,
\begin{equation}
\label{eq_190720260749}
K_{F,\mathbb{C}}f_{w}=f_{Bw};
\end{equation}
in the unique reduced coordinates, this is
\[
K_{F,\mathbb C}f_{Q_\Phi\widetilde w}
=f_{Q_\Phi\widetilde B_\Phi\widetilde w},
\qquad \widetilde w\in\mathbb C^{r_\Phi};
\]

\item \label{prop_koopman_eigenfunction_1} the function $f_w$ is a Koopman eigenfunction with eigenvalue $\lambda$ if and only if
\[
P_\Phi^{\mathbb C}(w)\neq0,
\qquad
P_\Phi^{\mathbb C}(Bw-\lambda w)=0.
\]
In particular, $Bw=\lambda w$ is sufficient when $f_w\neq0$. In reduced coordinates, $f_{Q_\Phi\widetilde w}$ is a Koopman eigenfunction if and only if $\widetilde w\neq0$ and $\widetilde B_\Phi\widetilde w=\lambda\widetilde w$;

\item for every $\widetilde w\in\mathbb{C}^{r_\Phi}$,
\begin{equation}
\label{eq_190720260756}
K_{F,\mathbb{C}}^*f_{Q_\Phi\widetilde w}
=f_{Q_\Phi\widetilde B_\Phi^\top\widetilde w};
\end{equation}

\item \label{prop_koopman_eigenfunction_2} the function $f_{Q_\Phi\widetilde w}$ is an adjoint Koopman eigenfunction with eigenvalue $\lambda$ if and only if $\widetilde w\neq0$ and $\widetilde B_\Phi^\top\widetilde w=\lambda\widetilde w$;

\item \label{prop_koopman_eigenfunction_5} for a held-out set $\mathrm E_{\mathrm{spec}}=\{(x'^i,F(x'^i))\}_{i=1}^M$ and $v=(v_1,\cdots,v_M)^\top\in\mathbb{C}^M$, define
\begin{equation}
g_v:=\sum_{i=1}^M v_i\mathcal K(\cdot,x'^i)
=\Phi(\cdot)^\top\sum_{i=1}^M v_i\Phi(x'^i).
\end{equation}
Then
\begin{IEEEeqnarray}{rCl}
\label{eq_210720260634}
\|K_{F,\mathbb{C}}^*g_v-\lambda g_v\|_{\mathcal H_\Phi^{\mathbb{C}}}^2
&=&\left\|\sum_{i=1}^M v_i\Phi(F(x'^i))-\lambda\sum_{i=1}^M v_i\Phi(x'^i)\right\|^2,\\
\label{eq_210720260635}
\|K_{F,\mathbb{C}}^*g_v\|_{\mathcal H_\Phi^{\mathbb{C}}}^2
&=&\left\|\sum_{i=1}^M v_i\Phi(F(x'^i))\right\|^2.
\end{IEEEeqnarray}
\end{enumerate}
\end{proposition}
\begin{proof}
The proof is provided in Appendix~\ref{app:spectral_correspondence_proof}.
\end{proof}
The reduced matrix describes the action on functions after redundant coefficients have been removed. Its adjoint is represented by $\widetilde B_\Phi^\top$ because $B$ and $Q_\Phi$ are real; this transpose is also the conjugate transpose. Complex eigenvectors do not change that matrix identity.

Under the exact-closure assumption of Proposition~\ref{prop_koopman_eigenfunction}, the next two definitions measure how closely a nonzero RKHS function satisfies an adjoint eigenrelation. When exact closure has not been established, the same finite sums can be calculated as empirical diagnostics. Their interpretation as a residual of the stated RKHS adjoint remains conditional on exact closure.
\begin{definition}[Squared Relative Adjoint Residual]\label{def_210720260931}
Assume the exact closure of Proposition~\ref{prop_koopman_eigenfunction}. For $g_v\neq0$ with $\|K_{F,\mathbb{C}}^*g_v\|_{\mathcal H_\Phi^{\mathbb{C}}}>0$, define
\begin{IEEEeqnarray}{rCl}
\rho(\lambda,v;\Phi,\mathrm E_{\mathrm{spec}})
&:=&
\frac{\|K_{F,\mathbb{C}}^*g_v-\lambda g_v\|_{\mathcal H_\Phi^{\mathbb{C}}}^2}
{\|K_{F,\mathbb{C}}^*g_v\|_{\mathcal H_\Phi^{\mathbb{C}}}^2}\\
&=&
\frac{\left\|\sum_{i=1}^M v_i\Phi(F(x'^i))-\lambda\sum_{i=1}^M v_i\Phi(x'^i)\right\|^2}
{\left\|\sum_{i=1}^M v_i\Phi(F(x'^i))\right\|^2}.
\end{IEEEeqnarray}
\end{definition}

\begin{definition}[Tolerance for the squared relative adjoint residual]\label{def_0307202260713}\leavevmode\par
For the held-out set $\mathrm E_{\mathrm{spec}}=\{(x'^i,F(x'^i))\}_{i=1}^M$, let
\[
g_v=\sum_{i=1}^M v_i\mathcal K(\cdot,x'^i),\qquad v\in\mathbb{C}^M.
\]
For $\varepsilon\geq0$, the candidate pair $(\lambda,g_v)$ satisfies a squared relative residual tolerance $\varepsilon$ if $g_v\neq0$, the denominator in Definition~\ref{def_210720260931} is positive, and
\[
\rho(\lambda,v;\Phi,\mathrm E_{\mathrm{spec}})\le\varepsilon.
\]
Because $\rho$ is a squared ratio, this condition bounds the relative residual norm by $\sqrt{\varepsilon}$. A tolerance $\varepsilon$ on the relative norm itself corresponds to $\rho\leq\varepsilon^2$.

The sampled states determine which kernel sections construct $g_v$. Under exact closure, their weighted sums give the RKHS norm of the residual for the entire function, as shown in Proposition~\ref{prop_koopman_eigenfunction}. Thus a zero residual establishes the adjoint eigenrelation for a nonzero $g_v$ whenever the denominator is positive. If $K_{F,\mathbb C}^*g_v=0$, this relative ratio is undefined. In particular, an exact eigenfunction with eigenvalue zero must be assessed with the unnormalized residual or with a positive normalization such as $\|g_v\|_{\mathcal H_\Phi^{\mathbb C}}^2$.
\end{definition}

\begin{remark}[Representing a spectral candidate with kernel sections]\label{rem_spectal_validation}
Suppose $(\lambda,\widetilde w)$ is an eigenpair of $\widetilde B_\Phi^\top$, with $\widetilde w\neq0$. It determines the candidate function $f_{Q_\Phi\widetilde w}$. To represent this function using kernel sections at the evaluation states $x'^1,\cdots,x'^M$, choose the minimum-norm solution of the following optimization problem:
\begin{equation}
\widehat v(\widetilde w;\Phi,\mathrm E_{\mathrm{spec}})
\in\argmin_{v\in\mathbb{C}^M}
\left(
\left\|\sum_{i=1}^M v_i\Phi(x'^i)-Q_\Phi\widetilde w\right\|^2
+\beta_{\mathrm{rep}}\|v\|^2
\right),\qquad \beta_{\mathrm{rep}}\ge0.
\end{equation}
For $\beta_{\mathrm{rep}}>0$, the minimizer is unique. For $\beta_{\mathrm{rep}}=0$, the minimum-norm convention selects a unique least-squares solution. The normalized squared error in approximating the candidate coefficient vector is
\begin{equation}
\eta_{\mathrm{rep}}(\widehat v;\widetilde w,\Phi,\mathrm E_{\mathrm{spec}})
:=
\frac{\left\|\sum_{i=1}^M\widehat v_i\Phi(x'^i)-Q_\Phi\widetilde w\right\|^2}
{\|Q_\Phi\widetilde w\|^2}.
\end{equation}
The denominator is positive because $\|Q_\Phi\widetilde w\|=\|\widetilde w\|>0$. If $\eta_{\mathrm{rep}}=0$, the kernel sections represent the candidate exactly: $g_{\widehat v}=f_{Q_\Phi\widetilde w}$. If $\eta_{\mathrm{rep}}>0$, the residual is calculated for the approximating function $g_{\widehat v}$. A candidate with $g_{\widehat v}=0$ is excluded from the relative residual diagnostic.

Exact closure also makes the weighted successor sum independent of the weights used to represent a given function. To see this, if $\sum_i u_i\Phi(x'^i)=0$, then
\[
\sum_i u_i\Phi(F(x'^i))
=B^\top\sum_i u_i\Phi(x'^i)=0.
\]
If this implication fails, two sets of weights can represent the same current function but produce different successor sums. In that case, the finite-sum diagnostic depends on the chosen weights and sample, so it cannot be assigned to that function as the residual of the stated RKHS adjoint.
\end{remark}

\subsection{Exact Linear Closure and Koopman-Invariant Coordinate Functions}
\label{sec:finite_koopman_closure}
The statistical bounds concern average prediction error under an evaluation distribution. Exact deterministic closure is much stronger: \emph{the same matrix must propagate the coordinates at every state in the specified domain}. For $\Phi:\mathcal X\rightarrow\Delta_C$, exact one-step closure under $B\in\mathbb{R}^{C\times C}$ means
\begin{IEEEeqnarray}{rCl}
\label{eq_250720260737}
\Phi(F(x))&=&B^\top\Phi(x),\qquad x\in\mathcal X.
\end{IEEEeqnarray}
For any fitted matrix $\widehat B$, define the mean squared one-step coordinate-prediction error under the evaluation distribution $\mathbb P_x$ by
\begin{equation}
R_{\mathbb{P}_x}(\Phi,\widehat B)
:=\mathbb{E}_{x\sim\mathbb P_x}\!\left[\|\Phi(F(x))-\widehat B^\top\Phi(x)\|^2\right],
\end{equation}
and the minimum root-mean-square linear prediction error by
\begin{equation}
\mathcal D_{\mathbb{P}_x}(\Phi)
:=\inf_{B\in\mathbb{R}^{C\times C}} R_{\mathbb{P}_x}(\Phi,B)^{1/2}.
\end{equation}
As in Section~\ref{sec:prediction_limits}, $R_{\mathbb P_x}$ measures the risk of a specified matrix; $\mathcal D_{\mathbb P_x}$ is the unconstrained analogue of $\mathcal D_\kappa$. The corresponding finite-dimensional linear prediction space is closed in $L^2(\mathbb P_x;\mathbb R^C)$, so least-squares projection attains this infimum even for linearly dependent coordinates. Zero risk gives equality only $\mathbb P_x$-almost surely, whereas the spectral results require pointwise deterministic closure over the entire stated domain.

For an independent $t$-step evaluation set $\mathrm E_t$, define
\begin{IEEEeqnarray}{rCl}
\widehat R_{\mathrm E_t}(\Phi,\widehat B)
&:=&\frac1M\sum_{i=1}^M
\left\|\Phi(F^t(x'^i))-\bigl((\widehat B)^\top\bigr)^t\Phi(x'^i)\right\|^2,\\
\widehat\nu_{\mathrm E_t}(\Phi,\widehat B)
&:=&\frac1M\sum_{i=1}^M
\inf_{y\in\Delta_C}
\left\|\bigl((\widehat B)^\top\bigr)^t\Phi(x'^i)-y\right\|^2.
\end{IEEEeqnarray}
The first diagnostic measures multistep coordinate-prediction error. The second measures the squared distance of predictions from the simplex. It is zero exactly when every evaluated prediction is nonnegative and sums to one.

For the spectral diagnostic, the theoretical construction uses a basis $Q_\Phi$ of the full feature span $V_\Phi$. A basis estimated from sampled features may identify only part of that space; Appendix~\ref{sec:exp_spectral} explains the numerical construction. Start from a fitted reduced-matrix eigenpair $(\widehat\lambda,\widehat w)$ with $\widehat w\neq0$ satisfying
\begin{equation}
\left(Q_\Phi^\top\widehat BQ_\Phi\right)^\top\widehat w
=\widehat\lambda\widehat w.
\end{equation}
On $\mathrm E_{\mathrm{spec}}$, choose the minimum-norm minimizer $\widehat v$ of the representation objective and compute
\begin{IEEEeqnarray}{rCl}
\widehat v(\widehat w;\Phi,\mathrm E_{\mathrm{spec}})
&\in&\mathop{\arg\min}_{v\in\mathbb{C}^M}
\left(\left\|\sum_{i=1}^M v_i\Phi(x'^i)-Q_\Phi\widehat w\right\|^2+\beta_{\mathrm{rep}}\|v\|^2\right),\qquad \beta_{\mathrm{rep}}\ge0,\\
\eta_{\mathrm{rep}}(\widehat v;\widehat w,\Phi,\mathrm E_{\mathrm{spec}})
&=&\frac{\left\|\sum_{i=1}^M\widehat v_i\Phi(x'^i)-Q_\Phi\widehat w\right\|^2}{\|Q_\Phi\widehat w\|^2},\\
\rho(\widehat\lambda,\widehat v;\Phi,\mathrm E_{\mathrm{spec}})
&=&\frac{\left\|\sum_{i=1}^M\widehat v_i\Phi(F(x'^i))-\widehat\lambda\sum_{i=1}^M\widehat v_i\Phi(x'^i)\right\|^2}
{\left\|\sum_{i=1}^M\widehat v_i\Phi(F(x'^i))\right\|^2}.
\end{IEEEeqnarray}
The representation error is defined because $\widehat w\neq0$. The residual ratio additionally requires a nonzero represented function and a nonzero weighted successor vector, as specified in Definitions~\ref{def_210720260931} and~\ref{def_0307202260713}. The quantity $\eta_{\mathrm{rep}}$ measures how accurately the sampled feature vectors represent the candidate coefficient. The quantity $\rho$ compares the corresponding weighted successor sum with $\widehat\lambda$ times the weighted current sum. Under exact deterministic closure, $\rho$ equals the stated RKHS-adjoint residual. Otherwise, it is an empirical measure of agreement between these two sums.
\begin{problem}[Linear prediction for a fixed soft abstraction]
\label{problem_central}
Fix the reference construction and a chosen map $\Phi\in\mathcal G_{\mathcal S}$, or its specified numerical implementation. Fit $B\in\mathbb R^{C\times C}$ to reduce the expected one-step prediction error:
\begin{equation}
\inf_{B\in\mathbb{R}^{C\times C}}
R_{\mathbb{P}_x}(\Phi,B).
\label{eq:conditional_population_objective}
\end{equation}
For a fitted pair $(\Phi,\widehat B)$, we distinguish a population accuracy requirement from empirical prediction and spectral criteria:
\begin{enumerate}[label=(R\arabic*)]

\item \textbf{Population one-step prediction accuracy.}
For a prescribed tolerance $\varepsilon_{\mathrm{cl}}>0$,
\begin{equation}
\mathcal D_{\mathbb{P}_x}(\Phi)
\le
R_{\mathbb{P}_x}(\Phi,\widehat B)^{1/2}
\le
\varepsilon_{\mathrm{cl}}.
\end{equation}

\item \textbf{Multistep prediction and simplex checks.}
For a prescribed horizon $t$ and tolerances
$\varepsilon_{\mathrm{t-step}},\varepsilon_{\mathrm{simplex}}>0$,
an evaluation set $\mathrm E_t$ independent of model construction and
fitting satisfies
\begin{IEEEeqnarray}{rCl}
\widehat R_{\mathrm E_t}(\Phi,\widehat B)^{1/2}
&\le&
\varepsilon_{\mathrm{t-step}},\\
\widehat\nu_{\mathrm E_t}(\Phi,\widehat B)^{1/2}
&\le&
\varepsilon_{\mathrm{simplex}}.
\end{IEEEeqnarray}

\item \textbf{Spectral diagnostics.}
For a candidate reduced-matrix eigenpair
$(\widehat\lambda,\widehat w)$ and tolerances
$\varepsilon_{\mathrm{rep}},\varepsilon_{\mathrm{spec}}>0$,
an independent spectral set $\mathrm E_{\mathrm{spec}}$ satisfies
\begin{IEEEeqnarray}{rCl}
\eta_{\mathrm{rep}}
(\widehat v;\widehat w,\Phi,\mathrm E_{\mathrm{spec}})
&\le&
\varepsilon_{\mathrm{rep}},\\
\rho
(\widehat\lambda,\widehat v;\Phi,\mathrm E_{\mathrm{spec}})
&\le&
\varepsilon_{\mathrm{spec}}.
\end{IEEEeqnarray}
\end{enumerate}
Criterion (R1) specifies population prediction risk, whereas (R2) measures finite-sample prediction error and simplex feasibility. Population interpretation of (R2) requires a separate statistical guarantee. Criterion (R3) evaluates the representation error and a weighted-feature residual; its tolerances apply to squared ratios and do not independently certify Koopman invariance.

Representation selection must also account for coordinate variation. In the exact normalized-affinity family, $\omega\downarrow0$ makes the coordinates uniformly approach $C^{-1}\mathbf 1_C$, for which identity-matrix prediction has vanishing error. The forecasting selection protocol therefore considers variation as well as error; Problem~\ref{problem_central} is formulated for the resulting fixed map.
\end{problem}
\begin{remark}[Selecting and fixing the representation]\label{rem:outer_representation_selection}
Model selection compares admissible reference constructions $a\in\mathfrak A$. Each construction specifies a class count $C_a$, a class-label map, a labelled sample $\mathcal S^{(a)}$, and fixed within-class subpartitions; these determine $\mathcal G_{\mathcal S^{(a)}}$. Problem~\ref{problem_central} fixes both the construction and soft map. For each tested parameter setting, construct a map from this family or its specified numerical implementation. Denote the resulting map by
\[
\Phi_{a,\omega,\tau}:\mathcal X\to\Delta_{C_a}.
\]
Its closure matrix $\widehat B_{a,\omega,\tau}$ is fitted using only the corresponding
fitting data. The experimental procedure then applies the stated validation rule to
the finite collection
\[
\left\{
\bigl(
\Phi_{a,\omega,\tau},
\widehat B_{a,\omega,\tau}
\bigr):
(a,\omega,\tau)\in\mathfrak C
\right\},
\]
where $\mathfrak C$ is the finite candidate grid. We denote this
operation abstractly by
\begin{equation}
(\widehat a,\widehat\omega,\widehat\tau)
=
\operatorname{Select}_{\mathrm{val}}
\left\{
\bigl(
\Phi_{a,\omega,\tau},
\widehat B_{a,\omega,\tau}
\bigr):
(a,\omega,\tau)\in\mathfrak C
\right\}.
\label{eq:outer_representation_selection}
\end{equation}
$\operatorname{Select}_{\mathrm{val}}$ denotes the selection rule used in an experiment. For example, it may retain candidates with similar prediction errors and then select the one with the greatest coordinate variation. The documented numerical extension and normalization, when used, are included in each candidate map. Once
$(\widehat a,\widehat\omega,\widehat\tau)$ has been selected, the
corresponding reference-class map, reference matrices, soft coordinate map, and
closure matrix are held fixed.

Changing the class count alters both output dimension and target; changing the soft map can alter the target even at fixed dimension. Thus representation-selection scores concern representation-specific tasks. Predictor comparisons isolate the fitted method only when every method is evaluated against the same selected coordinate map.
\end{remark}
\begin{definition}[Fitted linear coordinate model]
For the discrete-time system $(\mathcal X,F)$, the fitted model is the pair
\begin{equation}
\mathcal M=(\Phi,\widehat B),
\qquad
\Phi:\mathcal X\to\Delta_C,\quad \widehat B\in\mathbb{R}^{C\times C},
\end{equation}
where $\Phi$ is a fixed map from the chosen reference construction, including its specified numerical implementation when used.
For $h\ge1$, its abstraction-space predictor is
\begin{equation}
\mathcal P_{\mathcal M}^{(h)}(x)
:=\bigl((\widehat B)^\top\bigr)^h\Phi(x).
\end{equation}
When $\Phi$ uses KAHM reconstruction affinities, this model combines the KAHM coordinate map with a fitted linear predictor. An arbitrary real matrix need not preserve the simplex. The identity-initialized NLMS procedure used in the experiments preserves $\widehat B\mathbf1_C=\mathbf1_C$ in exact arithmetic, so its predictions sum to one. Their individual components can still be negative. Section~\ref{sec:successor_prototypes} and Appendix~\ref{app:supporting_diagnostics} explain the resulting simplex checks. The projection comparison evaluates how enforcing simplex membership changes prediction error.
\end{definition}

\begin{remark}[Scope of the KAHM adaptation]
The reconstruction and score constructions are cited in Section~\ref{sec:mathematical_background}; the affinity modification is defined in~\eqref{eq_220720260633}. Our contributions concern the resulting coordinate-prediction bounds and their conditional operator interpretation, not the original KAHM construction.
\end{remark}

% ===== END sections/05_abstraction_closure.tex =====

% ===== BEGIN sections/06_geometry_estimation.tex =====
\section{Reconstruction Scores, Prediction Bounds, and Matrix Fitting}
\label{sec:adaptive_closure_estimation}
We now distinguish three questions: how closely the soft coordinates agree with their reference labels, how much successor coordinates vary within a reference class, and how accurately a particular prediction matrix has been fitted. Reconstruction-score margins bound deviation from the reference labels under the stated normalization assumptions; transition dynamics determine the within-class successor variation; and fitting and finite updates affect the error of a particular prediction matrix. The resulting bounds distinguish conclusions about the representation itself from those that concern only a fitted predictor. The initial decomposition assumes deterministic successors; Section~\ref{sec:prediction_limits} already covers stochastic transition pairs.
\subsection{Class-conditional successor means and matrix fitting}
\label{sec:successor_prototypes}
In this subsection, $F:\mathcal X\to\mathcal X$ is deterministic and the evaluation law is $\mathbb P_x$. Write
\begin{equation}
 z(x):=\Phi(F(x))\in\Delta_C.
\end{equation}
We assume that every retained reference class has positive probability under the evaluation law:
\[
 \pi_c:=\mathbb{P}_y(y_c=1)>0,\qquad c=1,\cdots,C.
\]
Here $\mathbb{P}_y$ is the label distribution induced by $\mathbb{P}_x$ and the fixed reference-class map. For each class $c$, average the next-state coordinates over states currently labelled $c$:
\begin{equation}
 z_*^c:=\mathbb{E}_{x\sim\mathbb P_x}\!\left[z(x)\mid y_c=1\right]\in\Delta_C,
 \qquad c=1,\cdots,C.
 \label{eq:successor_prototype}
\end{equation}
This class-conditional mean belongs to $\Delta_C$ because it is an expectation of simplex-valued vectors. Define the within-class successor residual by
\begin{equation}
 \xi_{\mathrm{res}}(x):=z(x)-z_*^{c(x)}.
 \label{eq:successor_residual_direct}
\end{equation}
By construction,
\begin{equation}
 \mathbb{E}_{x\sim\mathbb P_x}\!\left[\xi_{\mathrm{res}}(x)\mid y_c=1\right]=0,
 \qquad c=1,\cdots,C.
\end{equation}
Thus $\xi_{\mathrm{res}}$ measures how a state's successor coordinates differ from the mean for its current reference class. Collect the class-conditional successor means as columns of $B_*^\top$:
\begin{equation}
 B_*^\top:=\left[z_*^1\ \cdots\ z_*^C\right].
 \label{eq:prototype_matrix}
\end{equation}
Since $y(x)=\mathrm e^{c(x)}$, we have $B_*^\top y(x)=z_*^{c(x)}$. Adding and subtracting $B_*^\top\Phi(x)$ gives the exact decomposition
\begin{align}
 \Phi(F(x))
 &=B_*^\top y(x)+\xi_{\mathrm{res}}(x)\nonumber\\
 &=B_*^\top\Phi(x)+B_*^\top\bigl(y(x)-\Phi(x)\bigr)+\xi_{\mathrm{res}}(x).
 \label{eq_110820261955}
\end{align}
The term $B_*^\top\Phi(x)$ is a soft-weighted combination of class mean successors; $B_*^\top(y(x)-\Phi(x))$ corrects the discrepancy from the reference label, and $\xi_{\mathrm{res}}(x)$ is the residual successor variation. In particular, each transition pair used to fit the predictor satisfies
\begin{equation}
 z^i=B_*^\top\Phi(x^i)+B_*^\top\bigl(y^i-\Phi(x^i)\bigr)+\xi_{\mathrm{res}}(x^i).
 \label{eq_290720261556}
\end{equation}
The rows of $B_*$ are the mean successor vectors under the evaluation distribution. They are nonnegative and each sums to one. The matrix minimizes mean squared error for prediction from hard labels. The optimal matrix for prediction from $\Phi(x)$ may differ because soft coordinates can retain differences between states that share the same hard label.

The normalized least-mean-squares update in~\eqref{eq_230720260912} preserves unit row sums when initialized at the identity matrix, in exact arithmetic. Its entries can nevertheless become negative. The fitted matrix is therefore not necessarily a stochastic transition matrix. Denote the $j$th column of $B_*$ by $\alpha_*^j:=(B_*)_{:,j}$. For $N_{\mathrm{ep}}\in\mathbb{Z}_+$ epochs over $N$ training pairs, define
\begin{equation}
 i_\ell:=1+((\ell-1)\bmod N),\qquad \ell=1,\cdots,N_{\mathrm{ep}}N.
 \label{eq:nlms_cyclic_index}
\end{equation}
This indexing repeats the fixed sample order at each epoch. The normalized least-mean-squares (NLMS) update~\citep{Hassibi1996HinfLMS} for output coordinate $j$ is
\begin{IEEEeqnarray}{rCl}
\label{eq_081120261005}
\widehat\alpha^j_\ell
&=&\widehat\alpha^j_{\ell-1}
+\beta\frac{(z^{i_\ell})_j-\Phi(x^{i_\ell})^\top\widehat\alpha^j_{\ell-1}}
 {1+\beta\|\Phi(x^{i_\ell})\|^2}\,\Phi(x^{i_\ell}),
\qquad \ell=1,\cdots,N_{\mathrm{ep}}N,
\end{IEEEeqnarray}
where $\widehat\alpha^j_0$ is the initial estimate and $0<\beta<1$ is the step size. Equivalently, all output coordinates can be updated together:
\begin{equation}
\label{eq_230720260912}
\widehat B_\ell = \widehat B_{\ell-1}
 +\frac{\beta}{1+\beta\|\Phi(x^{i_\ell})\|^2}\,\Phi(x^{i_\ell})\left(z^{i_\ell}-\widehat B_{\ell-1}^{\top}\Phi(x^{i_\ell})\right)^\top,
 \qquad \ell=1,\cdots,N_{\mathrm{ep}}N.
\end{equation}
After the last epoch, the fitted matrix is
\begin{equation}
 \widehat B:=\widehat B_{N_{\mathrm{ep}}N},
\end{equation}
and it predicts the next-state coordinates as $\widehat B^\top\Phi(x)$. This normalized least-mean-squares (NLMS) recursion and the finite-update identities in Section~\ref{sec:nlms_energy} apply to arbitrary training-transition sequences, including dependent samples. If $\widehat B_0\mathbf1_C=\mathbf1_C$, then $\widehat B_\ell\mathbf1_C=\mathbf1_C$ at every update in exact arithmetic. Indeed, the coordinates of $z^{i_\ell}$ sum to one, and so do those of the prediction whenever the previous matrix has unit row sums. The residual therefore sums to zero and the update leaves all row sums unchanged.
\subsection{Finite-update NLMS error analysis}
\label{sec:nlms_energy}
A poor forecast by a fitted matrix need not imply that every matrix is inaccurate. To isolate the additional contribution of fitting, we analyze how one \emph{hypothetical} normalized least-mean-squares (NLMS) update would change the fitted matrix's squared Frobenius distance to the class-mean reference matrix $B_*$. This is a parameter-error identity, not another directly computable lower confidence bound on the population optimum. Define
\begin{align}
 s(x)&:=\|\Phi(x)\|^2,
 &u_\beta(x)&:=\frac{\beta}{1+\beta s(x)},\label{eq:nlms_gain}\\
 \zeta_{\mathrm{fit}}(x)&:=\Phi(F(x))-\widehat B^\top\Phi(x),
 &\zeta_{\mathrm{proto}}(x)&:=\Phi(F(x))-B_*^\top\Phi(x).\label{eq:nlms_residuals}
\end{align}
Here $\zeta_{\mathrm{fit}}$ is the fitted model's prediction residual. The quantity $\zeta_{\mathrm{proto}}$ is the prediction residual when the class-mean matrix $B_*$ acts on the soft coordinates. The subscript labels this reference matrix; it does not imply that $B_*$ is optimal for soft-coordinate inputs. Equation~\eqref{eq_110820261955} gives
\begin{equation}
 \zeta_{\mathrm{proto}}(x)=B_*^\top\bigl(y(x)-\Phi(x)\bigr)+\xi_{\mathrm{res}}(x).
 \label{eq:prototype_residual}
\end{equation}
For the hypothetical additional update, let
\begin{align}
 \widehat B^+(x)&:=\widehat B+u_\beta(x)\Phi(x)\zeta_{\mathrm{fit}}(x)^\top,\label{eq:hypothetical_update}\\
 D_\beta(x)&:=\|B_*-\widehat B^+(x)\|_F^2-\|B_*-\widehat B\|_F^2.
 \label{eq:nlms_energy_change}
\end{align}
The additional update is hypothetical, and evaluation holds $\widehat B$ fixed. The sign of $D_\beta(x)$ records the change in squared Frobenius parameter distance to $B_*$: negative values indicate a decrease and positive values an increase. Define the expected absolute change $\varepsilon_{\mathrm{en}}^2$ and two constants depending on the step size:
\begin{equation}
 \varepsilon_{\mathrm{en}}^2:=\mathbb{E}_{x\sim\mathbb P_x}\!\left[|D_\beta(x)|\right],
 \qquad c_\beta:=1+\frac{\beta}{2+\beta},
 \qquad k_\beta:=\frac{1+\beta}{\sqrt{\beta(2+\beta)}}.
 \label{eq:energy_remainder_constants}
\end{equation}
All these quantities are finite for a fixed finite matrix $\widehat B$: the feature vectors and class-conditional means are bounded, and $0<\beta<1$.

\begin{proposition}[Exact NLMS parameter-error identity]\label{proposition_110820262057}
For $0<\beta<1$, any fitted matrix $\widehat B$, and every $x\in\mathcal X$, the additional update~\eqref{eq:hypothetical_update} satisfies
\begin{equation}
 \label{eq:nlms_exact_energy}
 D_\beta(x)=-u_\beta(x)\bigl(2-u_\beta(x)s(x)\bigr)\|\zeta_{\mathrm{fit}}(x)\|^2 +2u_\beta(x)\langle\zeta_{\mathrm{proto}}(x),\zeta_{\mathrm{fit}}(x)\rangle.
\end{equation}
Consequently,
\begin{equation}
 \|\zeta_{\mathrm{fit}}(x)\|\leq c_\beta\|\zeta_{\mathrm{proto}}(x)\|
          +k_\beta\sqrt{|D_\beta(x)|}.
 \label{eq:nlms_pointwise_bound}
\end{equation}
\end{proposition}
\begin{proof}
Appendix~\ref{app:energy_identity} expands the squared Frobenius norm in~\eqref{eq:nlms_energy_change} to obtain the identity and then solves the resulting quadratic inequality for $\|\zeta_{\mathrm{fit}}(x)\|$.
\end{proof}

Proposition~\ref{proposition_110820262057} bounds the fitted prediction error using the error of $B_*$ on soft coordinates and the change in matrix distance caused by a hypothetical update. To average this bound under $\mathbb P_x$, define
\begin{equation}
 \begin{split}
 \varepsilon_{\mathrm{est}}^2
 &:=\mathbb{E}_{x\sim\mathbb P_x}\!\left[\|(B_*-\widehat B)^\top\Phi(x)\|^2\right],\\
 \sigma_{\mathrm{res}}^2
 &:=\mathbb{E}_{x\sim\mathbb P_x}\!\left[\|\xi_{\mathrm{res}}(x)\|^2\right].
 \end{split}
 \label{eq:estimation_residual_scales}
\end{equation}
The quantity $\varepsilon_{\mathrm{est}}$ is the root-mean-square difference between the predictions of $\widehat B$ and $B_*$ on soft coordinates. This discrepancy can reflect fitting error. It can also reflect a difference between the optimal soft-coordinate predictor and the class-mean matrix that persists at the population level. The quantity $\sigma_{\mathrm{res}}$ measures root-mean-square variation of successor coordinates about their reference-class means:
\begin{equation}
 \sigma_{\mathrm{res}}^2
 =\sum_{c=1}^C\pi_c\,\mathbb{E}_{x\sim\mathbb P_x}\!\left[\|z(x)-z_*^c\|^2\mid y_c=1\right].
 \label{eq:within_regime_successor_dispersion}
\end{equation}
For a vector-valued function $g$, we use
$\|g\|_{L^2(\mathbb{P}_x)}:=\bigl(\mathbb{E}_{x\sim\mathbb P_x}\!\left[\|g(x)\|^2\right]\bigr)^{1/2}$.
This notation expresses the population prediction error as
$R_{\mathbb{P}_x}(\Phi,\widehat B)^{1/2}=\|\zeta_{\mathrm{fit}}\|_{L^2(\mathbb{P}_x)}$.

\begin{proposition}[Upper bounds on one-step prediction error]\label{proposition_120820261330}
Let
\[
 a_\Phi^2:=\mathbb{E}_{x\sim\mathbb P_x}\!\left[\|y(x)-\Phi(x)\|^2\right].
\]
Then
\begin{align}
 R_{\mathbb{P}_x}(\Phi,\widehat B)^{1/2}
 &\leq\|B_*\|_2a_\Phi+\sigma_{\mathrm{res}}+\varepsilon_{\mathrm{est}},
 \label{eq:direct_closure_bound}\\
 R_{\mathbb{P}_x}(\Phi,\widehat B)^{1/2}
 &\leq c_\beta\bigl(\|B_*\|_2a_\Phi+\sigma_{\mathrm{res}}\bigr)
       +k_\beta\varepsilon_{\mathrm{en}}.
 \label{eq:energy_closure_bound}
\end{align}
The first bound adds the effect of deviation from hard labels, within-class successor variation, and the prediction difference between $\widehat B$ and $B_*$. The second scales the first two terms by $c_\beta$ and replaces the last with $k_\beta\varepsilon_{\mathrm{en}}$, derived from the hypothetical update. The exact mean squared error includes cross terms between the residual components, whereas both bounds follow from the $L^2$ triangle inequality.
\end{proposition}
\begin{proof}
Subtracting $\widehat B^\top\Phi(x)$ from~\eqref{eq_110820261955} yields
\[
 \zeta_{\mathrm{fit}}(x)=(B_*-\widehat B)^\top\Phi(x)
       +B_*^\top\bigl(y(x)-\Phi(x)\bigr)+\xi_{\mathrm{res}}(x).
\]
The $L^2(\mathbb{P}_x)$ triangle inequality bounds the norm of this sum by the sum of the three component norms. The middle component is at most $\|B_*\|_2a_\Phi$, proving~\eqref{eq:direct_closure_bound}. Applying the same triangle inequality to~\eqref{eq:nlms_pointwise_bound}, and then using~\eqref{eq:prototype_residual}, gives~\eqref{eq:energy_closure_bound}. Appendix~\ref{app:closure_risk_estimation} provides the complete derivation.
\end{proof}

These inequalities describe the error of a fixed fitted matrix; rates for $\varepsilon_{\mathrm{est}}$ or $\varepsilon_{\mathrm{en}}$ as functions of training epochs require a separate training analysis. In~\eqref{eq:energy_closure_bound}, the update term $\varepsilon_{\mathrm{en}}$ vanishes when $D_\beta(x)=0$ for $\mathbb{P}_x$-almost every state. The absolute value in~\eqref{eq:energy_remainder_constants} measures the magnitude of the parameter-error change, retaining increases and decreases that can cancel in a signed average. Since $\varepsilon_{\mathrm{en}}$ depends on the unknown population class means in $B_*$, using it in a computable stopping criterion would require estimating these quantities.

\subsection{Deviation from Reference Labels and Coordinate Second Moments}
\label{sec:association_fidelity}
A representation agrees closely with the reference partition when it places most of its weight on the labelled class. The relevant quantity is \emph{deviation from the one-hot label}, not the error of a procedure that converts soft labels to hard assignments. The next definitions quantify this deviation and relate it to coordinate second moments.

\begin{definition}[Squared deviation from reference labels]\label{def:association_fidelity}
For each reference class $c$, let
\begin{equation}
 \pi_c:=\mathbb{P}_y(y_c=1),\qquad
 m_c:=\mathbb{E}_{x\sim\mathbb P_x}\!\left[(\Phi(x))_c^2\right].
 \label{eq:regime_population_mass}
\end{equation}
Under the standing assumption $\pi_c>0$, define
\begin{equation}
 \epsilon_c^2:=\mathbb{E}_{x\sim\mathbb P_x}\!\left[(1-(\Phi(x))_c)^2\mid y_c=1\right],
 \qquad A_{\mathrm{fid}}:=\sum_{c=1}^C\pi_c\epsilon_c^2.
 \label{eq:association_fidelity}
\end{equation}
Thus $0\leq\epsilon_c\leq1$ and $0\leq A_{\mathrm{fid}}\leq1$. The subscript ``fid'' refers to fidelity to the reference labels; smaller values mean better agreement. This is a squared coordinate-deviation measure, not a hard-classification error or the probability that a class prediction is wrong.
\end{definition}
The quantity $m_c$ is the uncentered second moment of coordinate $c$ over all states. Its value reflects both the coordinate's mean and its variation; a constant positive coordinate also has $m_c>0$. The class probability $\pi_c$ is defined by the reference-label distribution. $\epsilon_c$ is the root-mean-square shortfall of $\Phi_c(x)$ from $1$ among states labelled $c$. The reference-class map remains fixed when these quantities are evaluated. By the law of total expectation, $A_{\mathrm{fid}}$ can equivalently be written as \begin{equation}
 A_{\mathrm{fid}}=\mathbb{E}_{x\sim\mathbb P_x}\!\left[(1-(\Phi(x))_{c(x)})^2\right].
 \label{eq:association_fidelity_pointwise}
\end{equation}
\begin{theorem}[Assignment error, coordinate second moments, and approximation error]
\label{theorem_300320262124}\leavevmode\par
For any measurable simplex-valued soft coordinate map under the standing reference-class assumptions, and for every reference class $c$,
\begin{equation}
 m_c\geq\pi_c(1-\epsilon_c)^2.
 \label{eq:fidelity_mass_bound}
\end{equation}
Equivalently, defining
\[
 d_c:=\pi_c(2\epsilon_c-\epsilon_c^2),
\]
the bound can be written as $m_c\geq\pi_c-d_c$, where $d_c\geq0$. Moreover,
\begin{equation}
 a_\Phi^2
 =\mathbb{E}_{x\sim\mathbb P_x}\!\left[\|\Phi(x)-y(x)\|^2\right]
 \leq2A_{\mathrm{fid}}.
 \label{eq:fidelity_approximation_bound}
\end{equation}
\end{theorem}

\begin{proof}
Conditional on $y_c=1$, the $L^2$ triangle inequality gives
\[
1\leq\left(\mathbb{E}_{x\sim\mathbb P_x}\!\left[(\Phi(x))_c^2\mid y_c=1\right]\right)^{1/2}+\epsilon_c.
\]
Since $\epsilon_c\leq1$, squaring and retaining the class-$c$ contribution to the unconditional second moment yields
$m_c\geq\pi_c(1-\epsilon_c)^2=\pi_c-d_c$, where
$d_c=\pi_c\epsilon_c(2-\epsilon_c)\geq0$.
For $y(x)=\mathrm e^c$, simplex membership gives
\[
\|\Phi(x)-\mathrm e^c\|^2
=(1-\Phi_c(x))^2+\sum_{k\ne c}\Phi_k(x)^2
\leq2(1-\Phi_c(x))^2.
\]
Taking expectations proves $a_\Phi^2\leq2A_{\mathrm{fid}}$.
\end{proof}
\begin{proposition}[Sensitivity to reference-label deviation]\label{prop:prototype_diameter}\leavevmode\par
Define the largest distance between class mean successor vectors by
\begin{equation}
 L_{\mathrm{proto}}:=\max_{1\leq c,k\leq C}\|z_*^c-z_*^k\|.
 \label{eq:prototype_diameter}
\end{equation}
Then
\begin{align}
 \left(\mathbb{E}_{x\sim\mathbb P_x}\!\left[\|B_*^\top(y(x)-\Phi(x))\|^2\right]\right)^{1/2}
 &\leq L_{\mathrm{proto}}\sqrt{A_{\mathrm{fid}}},
 \label{eq:prototype_transport_bound}\\
 L_{\mathrm{proto}}&\leq\min\left(\sqrt{2}\|B_*\|_2,\sqrt{2}\right).
 \label{eq:prototype_diameter_comparison}
\end{align}
For $C=1$, $L_{\mathrm{proto}}=0$.
\end{proposition}
\begin{proof}
For $y(x)=\mathrm e^c$, simplex normalization gives
\begin{align}
B_*^\top(\mathrm e^c-\Phi(x))
&=\sum_{k\ne c}\Phi_k(x)(z_*^c-z_*^k),
\label{eq:prototype_difference_identity}\\
\|B_*^\top(\mathrm e^c-\Phi(x))\|
&\leq L_{\mathrm{proto}}(1-\Phi_c(x)).
\label{eq:prototype_pointwise_sensitivity}
\end{align}
Squaring and averaging proves~\eqref{eq:prototype_transport_bound}.
For $c\ne k$, the distance between means is at most
$\|B_*\|_2\|\mathrm e^c-\mathrm e^k\|=\sqrt2\|B_*\|_2$.
It is also at most the simplex diameter $\sqrt2$.
For $C=1$, both the assignment term and $L_{\mathrm{proto}}$ vanish.
\end{proof}
Equation~\eqref{eq:prototype_difference_identity} shows why differences between class mean successors matter: weight assigned to class $k$ instead of class $c$ changes the prediction by a multiple of $z_*^c-z_*^k$. Moving weight between classes with identical mean successors leaves the class-mean prediction unchanged. The diameter bound satisfies
$L_{\mathrm{proto}}\sqrt{A_{\mathrm{fid}}}\leq\|B_*\|_2\sqrt{2A_{\mathrm{fid}}}$. It therefore improves or matches the bound obtained by using $a_\Phi\leq\sqrt{2A_{\mathrm{fid}}}$ in Proposition~\ref{proposition_120820261330}. If $a_\Phi$ is known exactly, however, the term $\|B_*\|_2a_\Phi$ may be smaller. Using~\eqref{eq:prototype_transport_bound} in the proof of Proposition~\ref{proposition_120820261330} gives
\begin{align}
 R_{\mathbb{P}_x}(\Phi,\widehat B)^{1/2}
 &\leq L_{\mathrm{proto}}\sqrt{A_{\mathrm{fid}}}+\sigma_{\mathrm{res}}+\varepsilon_{\mathrm{est}},
 \label{eq:fidelity_direct_risk}\\
 R_{\mathbb{P}_x}(\Phi,\widehat B)^{1/2}
 &\leq c_\beta\bigl(L_{\mathrm{proto}}\sqrt{A_{\mathrm{fid}}}+\sigma_{\mathrm{res}}\bigr)
       +k_\beta\varepsilon_{\mathrm{en}}.
 \label{eq:fidelity_energy_risk}
\end{align}
The term $L_{\mathrm{proto}}\sqrt{A_{\mathrm{fid}}}$ bounds the prediction change caused by replacing a hard label with soft weights. It is small when the weights approach the label or the class mean successors approach one another. The remaining terms measure within-class successor variation and either prediction differences from $B_*$ or the effect of the hypothetical update. 
\subsection{Reconstruction-score margins and prediction-error bounds}
\label{sec:kahm_geometric_fidelity}
If a state is reconstructed appreciably better by its reference class than by any competing class, the normalized soft weight of that class should be near one. We quantify this intuition using the \emph{reconstruction-score margin} introduced in Section~\ref{sec:intro_kahm_overview}: the log-affinity of the assigned reference class minus the largest competing log-affinity. The resulting bound on deviation from reference labels can then be substituted into the general prediction inequalities. It does \emph{not} directly constrain variation among successor states or the error of a fitted matrix. For $C\geq2$, define each positive affinity and its logarithmic margin against competing reference classes by
\begin{align}
 a_c(x)&:=1-\mathcal T_c(x)+\tau,
 &G_c(x)&:=\log a_c(x)-\max_{k\ne c}\log a_k(x).
 \label{eq:kahm_affinity_margin}
\end{align}
Because $\tau\leq a_c(x)\leq1+\tau$, the logarithms are well defined on the admissible state region. A positive margin $G_c(x)$ means that reference class $c$ has a smaller reconstruction score, and hence a larger affinity, than every competitor.
\begin{theorem}[Reconstruction-score margins bound deviation from reference labels]
\label{thm:kahm_margin_fidelity}
Fix the reference-class map, the KAHM reconstruction maps, $\tau>0$, and a finite exponent $\omega>0$. For $C\geq2$ and thresholds $g_c>0$, define
\begin{align}
 \eta_c&:=\mathbb{P}_x\!\left(G_c(x)<g_c\mid y_c=1\right),
 &\vartheta_c&:=\frac{(C-1)e^{-\omega g_c}}
 {1+(C-1)e^{-\omega g_c}},
 \label{eq:kahm_margin_parameters}\\
 b_c&:=(1-\eta_c)\vartheta_c^2+\eta_c,
 &A_{\mathrm{geom}}&:=\sum_{c=1}^C\pi_cb_c.
 \label{eq:kahm_geometric_fidelity}
\end{align}
Then, for every reference class $c$,
\begin{align}
 \epsilon_c^2&\leq b_c,
 &m_c&\geq\pi_c(1-\eta_c)(1-\vartheta_c)^2,
 \label{eq:kahm_margin_fidelity_mass}\\
 A_{\mathrm{fid}}&\leq A_{\mathrm{geom}}\leq1,
 &a_\Phi^2&\leq2A_{\mathrm{geom}}.
 \label{eq:kahm_margin_aggregate}
\end{align}
The probability $\eta_c$ measures how often states labelled $c$ have margin below $g_c$. For the remaining states, $1-\Phi_c(x)\leq\vartheta_c$. Thus $b_c$ averages the squared-loss bound $\vartheta_c^2$ where the margin holds and the bound $1$ where it fails. Weighting these class bounds by $\pi_c$ gives $A_{\mathrm{geom}}$. The subscript ``geom'' labels this bound obtained from reconstruction-score margins; it does not denote a metric or imply coordinate invariance. A known upper bound $\overline\eta_c\in[\eta_c,1]$ can replace the unknown $\eta_c$ in these bounds. For $C=1$, normalization directly gives $\Phi\equiv1$ and $A_{\mathrm{fid}}=a_\Phi=0$.
\end{theorem}
\begin{proof}
Normalization gives
\begin{equation}
\Phi_c(x)=\frac{1}{1+S_c(x)},\qquad
S_c(x)=\sum_{k\ne c}\left(\frac{a_k(x)}{a_c(x)}\right)^\omega.
\label{eq:kahm_competing_affinity_sum}
\end{equation}
On the event $\mathcal E_c:=\{G_c\geq g_c\}$, each competing affinity ratio $a_k/a_c$ is at most $e^{-g_c}$. Hence
$S_c\leq(C-1)e^{-\omega g_c}$ and $1-\Phi_c\leq\vartheta_c$.
On its complement, $0\leq1-\Phi_c\leq1$. Splitting the conditional expectation over these events gives
$\epsilon_c^2\leq(1-\eta_c)\vartheta_c^2+\eta_c=b_c$.
Also,
\[
m_c\geq\mathbb{E}_{x\sim\mathbb P_x}\!\left[\Phi_c(x)^2\mathbbm1_{\{y_c=1\}\cap\mathcal E_c}\right]
\geq\pi_c(1-\eta_c)(1-\vartheta_c)^2.
\]
Averaging $b_c\in[0,1]$ and applying Theorem~\ref{theorem_300320262124} proves~\eqref{eq:kahm_margin_aggregate}.
Increasing $\eta_c$ enlarges the assignment-error upper bounds and decreases the second-moment lower bound. Replacing it by $\overline\eta_c\geq\eta_c$ therefore preserves all conclusions. For $C=1$, normalization gives $\Phi\equiv1$.
\end{proof}
\begin{corollary}[Reconstruction-based sufficient conditions for a margin]
\label{cor:kahm_reconstruction_margin}
Let $C\geq2$. For each reference class $c$, let
$0\leq\overline\eta_c\leq1$. Suppose that, conditional on $y_c=1$,
an event of probability at least $1-\overline\eta_c$ satisfies
\begin{equation}
 \mathcal T_c(x)\leq t_c^-<t_c^+
       \leq\min_{k\ne c}\mathcal T_k(x),
 \qquad 0\leq t_c^-<t_c^+\leq1.
 \label{eq:kahm_folding_separation}
\end{equation}
Then the hypotheses of Theorem~\ref{thm:kahm_margin_fidelity} are
satisfied with
\begin{equation}
 g_c=\log\frac{1-t_c^-+\tau}{1-t_c^++\tau}>0,
 \qquad
 \eta_c\leq\overline\eta_c.
 \label{eq:kahm_folding_to_log_margin}
\end{equation}

The score separation can be checked through reconstruction distances and angles. On the same event, require a small distance and angular error for the labelled class, and a sufficiently large distance for every competing class:
\begin{equation}
 \|x-\mathcal A_c(x)\|\leq r_c^-,
 \qquad
 \mathcal T^{\mathrm{Cos}}_c(x)\leq\upsilon_c,
 \qquad
 \min_{k\ne c}\|x-\mathcal A_k(x)\|\geq r_c^+,
 \label{eq:kahm_reconstruction_separation}
\end{equation}
where $r_c^-,r_c^+\geq0$ and $\upsilon_c\in[0,1]$, provided that
\begin{equation}
 t_c^-:=
 \sqrt{\frac{(1-e^{-r_c^-})^2+\upsilon_c^2}{2}}
 <t_c^+:=
 \frac{1-e^{-r_c^+}}{\sqrt{2}}.
 \label{eq:kahm_reconstruction_to_folding}
\end{equation}
\end{corollary}

\begin{proof}
On the stated score-separation event, $a_c\geq1-t_c^-+\tau$ and
$a_k\leq1-t_c^++\tau$ for every $k\ne c$. Thus
$G_c\geq\log((1-t_c^-+\tau)/(1-t_c^++\tau))>0$, and margin failure has conditional probability at most $\overline\eta_c$.
For the reconstruction-level condition, monotonicity of $1-e^{-r}$ gives
\[
\mathcal T_c(x)\leq
\sqrt{\bigl((1-e^{-r_c^-})^2+\upsilon_c^2\bigr)/2}=t_c^-,
\qquad
\mathcal T_k(x)\geq(1-e^{-r_c^+})/\sqrt2=t_c^+.
\]
The second inequality discards the nonnegative squared angular contribution. These inequalities establish the required score separation.
\end{proof}
Because the KAHM reconstruction $\mathcal A_k(x)$ lies in the affine hull of the reference samples, the competing reconstruction error admits a simple lower bound. Let $H_k$ denote the affine hull of the rows of $\mathbf X^k$. Then
\begin{equation}
 \|x-\mathcal A_k(x)\|
 \geq
 \operatorname{dist}(x,H_k)
 :=
 \inf_{h\in H_k}\|x-h\|,
 \label{eq:kahm_affine_hull_certificate}
\end{equation}
since $\mathcal A_k(x)\in H_k$. If $H_k=\mathbb R^n$, this lower bound is zero. The reconstruction-based certificate in Corollary~\ref{cor:kahm_reconstruction_margin} combines control of labelled-class reconstruction with lower bounds on competing-class reconstruction errors. Combining~\eqref{eq:kahm_margin_aggregate} with
\eqref{eq:fidelity_direct_risk}-\eqref{eq:fidelity_energy_risk} gives
\begin{align}
 T_{\mathrm{geom}}
 &:=L_{\mathrm{proto}}\sqrt{A_{\mathrm{geom}}}
    +\sigma_{\mathrm{res}},
 \label{eq:geometric_risk_core}\\
 R_{\mathbb{P}_x}(\Phi,\widehat B)^{1/2}
 &\leq
 \min\left(
 T_{\mathrm{geom}}+\varepsilon_{\mathrm{est}},\,
 c_\beta T_{\mathrm{geom}}
 +k_\beta\varepsilon_{\mathrm{en}}
 \right).
 \label{eq:geometry_closure_bound}
\end{align}
Larger reconstruction-score margins can reduce $A_{\mathrm{geom}}$. The resulting error contribution is multiplied by $L_{\mathrm{proto}}$, the largest distance between class mean successors. Controlling the full prediction bound also requires bounds on within-class successor variation $\sigma_{\mathrm{res}}$ and on the matrix-dependent terms $\varepsilon_{\mathrm{est}}$ and $\varepsilon_{\mathrm{en}}$.
\begin{remark}[Numerical score extensions]\label{rem:numerical_score_scope}
Theorem~\ref{thm:kahm_margin_fidelity} applies to measurable bounded scores with normalized affinities, independently of the reconstruction algorithm. Write $\widetilde{\mathcal T}_c$ for the fixed measurable scores produced by the numerical rules in Section~\ref{sec:exp_protocol}, with values in $[0,1]$. Let $\widetilde\Phi$ denote the implemented coordinate map. Outside the measurable fallback set $\mathcal U$, it normalizes these scores as in~\eqref{eq_220720260633}. On $\mathcal U$, where the summed powered affinities are at most $10^{-12}$, it instead returns uniform coordinates, as specified in~\eqref{eq:implemented_affinity_normalization}. The threshold rule can also activate where the reconstruction is well defined. The theorem applies with margins computed from $\widetilde{\mathcal T}_c$ and
\[
\widetilde\eta_c
:=\mathbb P_x\bigl(\{\widetilde G_c<g_c\}\cup\mathcal U\mid y_c=1\bigr).
\]
The margin argument applies on $\{\widetilde G_c\geq g_c\}\setminus\mathcal U$; the assignment loss is bounded by one on its complement. Consequently, $\widetilde\eta_c$ accounts for the full probability of margin failure or fallback use. The RKHS construction and class-mean risk bounds also apply when every quantity is defined using $\widetilde\Phi$. Applying the affine-hull bound~\eqref{eq:kahm_affine_hull_certificate} additionally requires verifying that the implemented reconstruction belongs to the specified hull, including under any denominator replacement.
\end{remark}

\subsection{Minimum prediction error and successor variation}
\label{sec:attainable_closure}
The lower certificate rules out an error tolerance when its value is sufficiently large, but it does not generally show whether that tolerance is attainable. We therefore complement Proposition~\ref{prop:population_obstruction} with an upper bound obtained from the class-mean matrix when it satisfies the norm limit. Substituting the KAHM reconstruction-score margin bound then expresses both sides in terms of label deviation controlled by these scores. Finally, the four-state example holds the coordinates fixed while changing successor behavior, making the predictive obstruction easy to isolate.
\begin{proposition}[Attainable error with bounded matrix norm]
\label{prop:attainable_closure}
For $\kappa\geq0$, define
\[
\mathcal D_\kappa(\Phi)
:=\inf_{\|B\|_2\leq\kappa}R_{\mathbb P_x}(\Phi,B)^{1/2},
\qquad (t)_+:=\max(t,0).
\]
If $\kappa\geq\|B_*\|_2$, then
\begin{equation}
(\sigma_{\mathrm{res}}-\kappa a_\Phi)_+
\leq\mathcal D_\kappa(\Phi)
\leq\sigma_{\mathrm{res}}+L_{\mathrm{proto}}\sqrt{A_{\mathrm{fid}}}.
\label{eq:two_sided_closure}
\end{equation}
For the reconstruction-score construction, this implies
\begin{equation}
(\sigma_{\mathrm{res}}-\kappa\sqrt{2A_{\mathrm{geom}}})_+
\leq\mathcal D_\kappa(\Phi)
\leq\sigma_{\mathrm{res}}+L_{\mathrm{proto}}\sqrt{A_{\mathrm{geom}}}.
\label{eq:two_sided_geometric_closure}
\end{equation}
The lower bounds hold for every norm limit $\kappa\geq0$. The upper bounds use $B_*$ as a feasible predictor and therefore require $\kappa\geq\|B_*\|_2$. Each row of $B_*$ belongs to $\Delta_C$, which gives $\|B_*\|_2\leq\sqrt C$. Thus $\kappa=\sqrt C$ always includes this predictor in the permitted matrix class.
\end{proposition}
\begin{proof}
Conditional centering of $\xi_{\mathrm{res}}=z-B_*^\top y$ gives, for every $B$,
\[
\begin{aligned}
\mathbb{E}_{x\sim\mathbb P_x}\!\left[\|z(x)-B^\top y(x)\|^2\right]
&=\sigma_{\mathrm{res}}^2+
\mathbb{E}_{x\sim\mathbb P_x}\!\left[\|(B_*-B)^\top y(x)\|^2\right]\\
&\geq\sigma_{\mathrm{res}}^2.
\end{aligned}
\]
The reverse triangle inequality therefore yields
$R_{\mathbb P_x}(\Phi,B)^{1/2}\geq
\sigma_{\mathrm{res}}-\|B\|_2a_\Phi$.
Taking the infimum over $\|B\|_2\leq\kappa$ proves the lower bound.
For the upper bound, use the feasible matrix $B_*$ and~\eqref{eq:fidelity_direct_risk}, whose discrepancy term is then zero.
Theorem~\ref{thm:kahm_margin_fidelity} gives~\eqref{eq:two_sided_geometric_closure}.
Finally, $\|B_*\|_2\leq\|B_*\|_F\leq\sqrt C$.
\end{proof}

When $B_*$ satisfies the spectral-norm limit and assignment error is small, both sides of the two-sided inequality approach $\sigma_{\mathrm{res}}$, so the optimal soft-coordinate error is close to the optimal hard-label error. The comparison applies to one fixed representation, since altering the coordinates may also change $\sigma_{\mathrm{res}}$. In particular, if
$\sigma_{\mathrm{res}}>\varepsilon+\kappa\sqrt{2A_{\mathrm{geom}}}$,
every linear prediction matrix with spectral norm at most $\kappa$ has root-mean-square error greater than $\varepsilon$.
Conversely, $\sigma_{\mathrm{res}}+L_{\mathrm{proto}}\sqrt{A_{\mathrm{geom}}}\leq\varepsilon$
guarantees such a matrix exists whenever $\kappa\geq\|B_*\|_2$.
These are sufficient conditions for excluding or attaining the specified tolerance within the norm-constrained matrix class. They need not determine the exact minimum, and a particular fitted matrix may have a larger error than the optimum. The norm limit is essential to the lower bound: it prevents arbitrarily large amplification of small coordinate differences between states with the same hard label.

\paragraph{A four-state reconstruction example: identical inputs, conflicting successors.}\label{ex:kahm_conflict}
Here we make the Introduction's thought experiment exact using admissible KAHM reconstructions. Two states in the first class receive identical current coordinates, but the chosen dynamics send them to different successor-coordinate vectors. No forecast based only on their current coordinates can tell the two successors apart. Let $a=(1,1)^\top$, $b=(1,-1)^\top$, $d=(-1,1)^\top$, and $e=(-1,-1)^\top$.
Assign probability $1/4$ to each state in $\mathcal X=\{a,b,d,e\}$. Use reference datasets $\mathbf X^1=[a\ b]^\top$ and $\mathbf X^2=[d\ e]^\top$, with labels $c(a)=c(b)=1$ and $c(d)=c(e)=2$.
Use one KAHM per class. Within each class, the first state coordinate is constant and the second takes values $\pm1$. Algorithm~\ref{algorithm_encoding_matrix} therefore retains the vertical direction. The encoded samples are distinct and have covariance $2$, so both datasets are admissible.
Their Gaussian kernel matrices are both $K=\left[\begin{smallmatrix}1&k\\k&1\end{smallmatrix}\right]$, with $k=e^{-1}$. The two datasets differ only in the sign of their first column. This sign change leaves the squared residuals in the Appendix~\ref{app:kahm_construction} regularization rule unchanged, so both use the same $\lambda>0$.
At either encoded reference point, the regression outputs sum to $(1+k)/(1+k+\lambda)>0$, so normalization is well defined. At the point with kernel vector $(1,k)^\top$, the outputs are $(K+\lambda I)^{-1}(1,k)^\top$. Dividing by their sum gives the affine weights and the factor $q$ in the reconstruction formula. The reconstructions of all four states are
\[
q:=\frac{(1-k)(1+k+\lambda)}{(1+k)(1-k+\lambda)}\in(0,1),\qquad
\mathcal A_1(x)=(1,qx_2)^\top,\quad
\mathcal A_2(x)=(-1,qx_2)^\top.
\]
All four states and reconstructions are nonzero. For each state's labelled class and the other class, respectively, the reconstruction distances and cosine similarities are
\[
1-q,\quad\sqrt{4+(1-q)^2};
\qquad
\frac{1+q}{\sqrt{2(1+q^2)}},\quad
\frac{q-1}{\sqrt{2(1+q^2)}}.
\]
Thus Definition~\ref{def_space_folding_measure} gives fixed scores $T_{\mathrm{in}}<T_{\mathrm{out}}$, with the same strict margin at every state. For any fixed $\tau>0$, let
\[
\chi:=\frac{1-T_{\mathrm{in}}+\tau}{1-T_{\mathrm{out}}+\tau}>1,
\qquad p:=\frac{\chi^\omega}{1+\chi^\omega}.
\]
The resulting KAHM coordinates are $\Phi(a)=\Phi(b)=v=(p,1-p)^\top$ and $\Phi(d)=\Phi(e)=u=(1-p,p)^\top$.
As $\omega\to\infty$, $A_{\mathrm{fid}}=(1-p)^2\to0$: the soft coordinates approach the reference labels. Under the uniform distribution on the four states, their mean is
\[
\bar\phi:=\mathbb{E}_{x\sim\mathbb P_x}\!\left[\Phi(x)\right]
=(1/2,1/2)^\top.
\]
Their total coordinate variance tends to a positive limit:
\[
\mathbb{E}_{x\sim\mathbb P_x}\!\left[\|\Phi(x)-\bar\phi\|^2\right]
=(2p-1)^2/2\longrightarrow1/2.
\]
Under identity dynamics, these coordinates have exact closure with $B=I_2$. Holding the representation fixed, change the state map to $F(a)=a$ and $F(b)=F(d)=F(e)=d$. Although $a$ and $b$ have identical current coordinates, their successor coordinates differ; every coordinate-only predictor assigns them the same output, whose least-squares optimum is $(v+u)/2$. The optimal prediction for $d$ and $e$ is $u$.
The matrix
\[
B^\top=\begin{pmatrix}1-p/2&(1-p)/2\\p/2&(1+p)/2\end{pmatrix}
\]
realizes these predictions and has $\|B\|_F^2=p^2-p+3/2\leq3/2$. Consequently,
\begin{equation}
\mathcal D_{\sqrt2}(\Phi)^2
=\inf_B R_{\mathbb P_x}(\Phi,B)
=\frac{\|v-u\|^2}{8}
=\frac{(2p-1)^2}{4}
=\sigma_{\mathrm{res}}^2.
\label{eq:analytic_class_conflict}
\end{equation}
This minimum squared error is two thirds of the optimal constant-predictor risk, $3(2p-1)^2/8$.
The equality $\Phi(a)=\Phi(b)$ persists for every $\omega$, while their successor coordinates differ under $F$. Increasing $\omega$ sharpens the assignments while preserving this shared representation. As the coordinates approach hard labels, assignment error vanishes while minimum prediction risk tends to $1/4$.

\begin{table}[tbp]
\centering
\caption{Analytical limits for one KAHM representation under two dynamics maps at $p=31/32$, with $\kappa=\sqrt2$. Both maps have assignment loss $1/1024$ and coordinate variance $225/512$. The last two columns give population limits, not additional numerical evaluations.}
\label{tab:analytic_obstruction}
\small
\begin{tabularx}{\linewidth}{@{}Ycccc@{}}
\toprule
Dynamics & $\sigma_{\mathrm{res}}$ & Optimal RMSE & Label limit & Centered limit \\
\midrule
Identity: $F(x)=x$ & $0$ & $0$ & $0$ & $0$ \\
Conflicting successors: $F(a)=a$; $F(b)=F(d)=F(e)=d$ & $15/32$ & $15/32$ & $13/32$ & $15/32$ \\
\bottomrule
\end{tabularx}
\end{table}

\FloatBarrier
\paragraph{An informative certificate for the KAHM example.}
For the four-state KAHM construction, choose $\omega=\log(31)/\log\chi$, so $p=31/32$.
Then $A_{\mathrm{fid}}=1/1024$, $\sigma_{\mathrm{res}}^2=225/1024$, and the exact optimal root-mean-square error is $15/32$.
For the illustrative \emph{possible} balanced count vector $(n_a,n_b,n_d,n_e)=(1024,1024,1024,1024)$, $M=4096$, $\widehat f=1/1024$, $\widehat s=225/1024$, and $\widehat t=0$. At $\delta=0.05$ the pointwise bounds are $0.303065$ (label), $0.128038$ (centered), $0.336172$ (prototypes $v,u$ fixed in advance), and $0.329043$ (ERM). Knowing $\tau_\Phi=0$ yields the separate special-case bound $\sqrt{V_\delta}=0.421053$. The population optimum is $15/32$. These are deterministic calculations on an illustrative count vector, not sampled campaign replicates or jointly calibrated claims.
The $95\%$ claim concerns repeated i.i.d. sampling, not a conditional guarantee given balanced counts. Corollary~\ref{cor:certificate_consistency} implies $L_{\sqrt2,\delta}\to13/32$, which falls below the true optimum $15/32$ by $1/16$, while the population centered bound is exact. Table~\ref{tab:analytic_obstruction} contrasts these limits. The archived Monte Carlo counts are compared separately in Section~\ref{sec:certificate_four_state_evaluation}.

\begin{remark}[One-step risk for stochastic successors]\label{rem:stochastic_risk}
Let $(x,x^+)$ have fixed joint evaluation distribution $\mathbb P_{x,x^+}$, with current-state marginal $\mathbb P_x$, and write $z=\Phi(x^+)$. Define
\[
m(x):=\mathbb{E}_{x^+\sim\mathbb P(\cdot\mid x)}\!\left[\Phi(x^+)\right].
\]
The vector $m(x)$ is the mean next-state coordinate vector conditional on the full current state. In this expectation, $x$ is the conditioning random variable; only the successor is averaged over its conditional distribution. The prediction error has two parts: error in predicting this mean and random variation around it. The standard variance decomposition gives \citep{colbrook2024stochastic}
\[
\begin{aligned}
\mathbb{E}_{(x,x^+)\sim\mathbb P}\!\left[\|z-B^\top\Phi(x)\|^2\right]
&=\mathbb{E}_{x\sim\mathbb P_x}\!\left[\|m(x)-B^\top\Phi(x)\|^2\right]\\
&\quad+\mathbb{E}_{(x,x^+)\sim\mathbb P}\!\left[\|z-m(x)\|^2\right].
\end{aligned}
\]
Proposition~\ref{prop:attainable_closure} also holds for this joint-law risk. By conditioning first on the current state and then on its label, the class mean can be written in either form:
\[
\begin{aligned}
z_*^c
&=\mathbb{E}_{(x,x^+)\sim\mathbb P}\!\left[z\mid y_c=1\right]\\
&=\mathbb{E}_{x\sim\mathbb P_x}\!\left[m(x)\mid y_c=1\right].
\end{aligned}
\]
Consequently,
\[
\sigma_{\mathrm{res}}^2
=\mathbb{E}_{x\sim\mathbb P_x}\!\left[\|m(x)-z_*^{c(x)}\|^2\right]+\mathbb{E}_{(x,x^+)\sim\mathbb P}\!\left[\|z-m(x)\|^2\right].
\]
The same conditional-centering argument proves both bounds. Here successor variation has two sources: different states in one class may have different conditional mean successors, and the successor may vary randomly even at a fixed state. This decomposition describes the stochastic experiments' one-step prediction errors; the spectral identities retain their deterministic-transition assumptions.
\end{remark}

\subsection{Independent evaluation and finite-sample bounds}
\label{sec:fidelity_evaluation}
Let $\mathcal F_{\mathrm{fit}}$ collect all information used to construct the representation, fit the matrix, and select the model, including random choices. Formally, it is the $\sigma$-algebra generated by those data and choices. Conditioning on it lets us treat the reference classes, coordinates, and fitted matrix as fixed. Draw evaluation states $x'^{1},\ldots,x'^{M}$ independently from $\mathbb P_x$, independently of $\mathcal F_{\mathrm{fit}}$, and define $y'^{i}:=y(x'^{i})$. The evaluation sample
\begin{equation}
 \mathcal V
 :=
 \bigl((x'^{i},y'^{i})\bigr)_{i=1}^{M},
 \qquad
 \mathbb P_{\mathcal V}:=(\mathbb P_{x,y})^{\otimes M}
 \label{eq:fidelity_evaluation_sample}
\end{equation}
has conditional distribution $\mathbb P_{\mathcal V}$ given $\mathcal F_{\mathrm{fit}}$. Its entries are $M$ independent and identically distributed state-reference-label realizations. For $y'^{i}=\mathrm e^{c_i}$, define
\begin{equation}
 \widehat A_{\mathrm{fid}}:=\frac1M\sum_{i=1}^M(1-(\Phi(x'^{i}))_{c_i})^2,
 \qquad  U_\delta
:=
\min\left(
1,\,
\widehat A_{\mathrm{fid}}
+\sqrt{\frac{\log(1/\delta)}{2M}}
\right),\quad 0<\delta<1.
 \label{eq:fidelity_confidence_bound}
\end{equation}
Conditional on $\mathcal F_{\mathrm{fit}}$, the summands in
$\widehat A_{\mathrm{fid}}$ are independent, take values in $[0,1]$,
and have common mean $A_{\mathrm{fid}}$. Hoeffding's
inequality~\citep{hoeffding1963} therefore gives
\begin{equation}
 \Pr_{\mathcal V\sim\mathbb P_{\mathcal V}}\!\left(
 A_{\mathrm{fid}}\leq U_\delta
 \right)\geq1-\delta.
 \label{eq:fidelity_confidence_event}
\end{equation}
The probability is over repeated evaluation samples, with the fitted objects held fixed. We call $A_{\mathrm{fid}}\leq U_\delta$ the \emph{confidence event}. Because $\widehat A_{\mathrm{fid}}$ estimates a single loss averaged across all classes, one confidence bound applies to that average. The radius requires evaluation pairs that are mutually independent and independent of the fitting and model-selection data. Appendix~\ref{app:fidelity_concentration} gives the corresponding formulation for independent trajectories whose observations may be dependent within each trajectory.
\begin{theorem}[Closure-risk bounds with a finite-sample assignment term]
\label{theorem_one_step_closure_risk_bound}
Assume deterministic transitions and a measurable simplex-valued map $\Phi$, with $\pi_c>0$ for every retained reference class. Let $0<\beta,\delta<1$. Conditional on $\mathcal F_{\mathrm{fit}}$, let $\mathcal V$ contain $M$ i.i.d. state-reference-label pairs with law $\mathbb P_{x,y}$, as in~\eqref{eq:fidelity_evaluation_sample}. With conditional probability at least $1-\delta$ over this evaluation sample, the following bound holds simultaneously for all $\gamma,\eta>0$:
\begin{equation}
 R_{\mathbb{P}_x}(\Phi,\widehat B)\leq{}
 (1+\eta)c_\beta^2
 \left[
 (1+\gamma)L_{\mathrm{proto}}^2U_\delta
 +(1+\gamma^{-1})\sigma_{\mathrm{res}}^2
 \right]+(1+\eta^{-1})k_\beta^2\varepsilon_{\mathrm{en}}^2.
 \label{eq:finite_sample_risk_family}
\end{equation}
\end{theorem}
\begin{proof}
On the event $A_{\mathrm{fid}}\leq U_\delta$, which has conditional probability at least $1-\delta$,~\eqref{eq:fidelity_energy_risk} gives
\[
R_{\mathbb P_x}(\Phi,\widehat B)^{1/2}
\leq c_\beta\bigl(L_{\mathrm{proto}}\sqrt{U_\delta}+\sigma_{\mathrm{res}}\bigr)
+k_\beta\varepsilon_{\mathrm{en}}.
\]
Apply $(a+b)^2\leq(1+t)a^2+(1+t^{-1})b^2$ first to the two outer terms with $t=\eta$, then inside the parentheses with $t=\gamma$. This gives~\eqref{eq:finite_sample_risk_family}. The confidence event is the same for every $\gamma,\eta>0$, so their values may be chosen after observing the assignment-loss estimate.
\end{proof}
In the following result, $T_\delta$ denotes the prediction-error term defined below; it is distinct from the input-variance bound in~\eqref{eq:input_variance_confidence}.

\begin{result}[Compact one-step closure risk bound]
\label{result_210820160745}
Under the assumptions of
Theorem~\ref{theorem_one_step_closure_risk_bound}, define
\begin{align}
 T_\delta
 &:=L_{\mathrm{proto}}\sqrt{U_\delta}
   +\sigma_{\mathrm{res}},
 \label{eq:risk_bound_core}\\
 \mathcal B_\delta
 &:=\min\left(
 T_\delta+\varepsilon_{\mathrm{est}},\,
 c_\beta T_\delta+k_\beta\varepsilon_{\mathrm{en}}
 \right).
 \label{eq:risk_bound_radius}
\end{align}
Then, conditional on $\mathcal F_{\mathrm{fit}}$, with probability at
least $1-\delta$ over the independent evaluation sample,
\begin{equation}
 R_{\mathbb{P}_x}(\Phi,\widehat B)
 \leq\mathcal B_\delta^2.
 \label{eq:repaired_risk_bound}
\end{equation}
Moreover,
\begin{equation}
 \inf_{\gamma,\eta>0}
 \Bigl\{
 (1+\eta)c_\beta^2
 \bigl[
 (1+\gamma)L_{\mathrm{proto}}^2U_\delta
 +(1+\gamma^{-1})\sigma_{\mathrm{res}}^2
 \bigr]
 +(1+\eta^{-1})k_\beta^2\varepsilon_{\mathrm{en}}^2
 \Bigr\}
 =
 \left(
 c_\beta T_\delta+k_\beta\varepsilon_{\mathrm{en}}
 \right)^2.
 \label{eq:optimized_energy_risk}
\end{equation}
\end{result}
\begin{proof}
On the event $A_{\mathrm{fid}}\leq U_\delta$, substituting into
\eqref{eq:fidelity_direct_risk}-\eqref{eq:fidelity_energy_risk} gives both bounds defining $\mathcal B_\delta$. Their minimum proves~\eqref{eq:repaired_risk_bound}.
For~\eqref{eq:optimized_energy_risk}, apply
$\inf_{t>0}[(1+t)a+(1+t^{-1})b]=(\sqrt a+\sqrt b)^2$
first to the $\gamma$ terms and then to the $\eta$ terms. Appendix~\ref{app:compact_risk} includes the cases $a=0$ or $b=0$.
The confidence event holds simultaneously for all positive values of these auxiliary parameters, so taking the infima preserves its probability guarantee.
\end{proof}
Result~\ref{result_210820160745} combines finite-sample control of assignment loss with population class means and matrix-dependent quantities. To obtain a computable upper bound on prediction error, we estimate class means before evaluation and hold them fixed. Independent evaluation pairs then bound each term in the error decomposition.

\begin{proposition}[Risk bound from independent evaluation pairs]
\label{prop:evaluable_risk}\leavevmode\par
Fix $\Phi$, $\widehat B$, and $A^\top=[a^1\ \cdots\ a^C]$ with $a^c\in\Delta_C$, using only construction, fitting, and selection data. Each column $a^c$ of $A^\top$ may be an estimate, obtained from training data, of the mean successor coordinates for class $c$. Draw $M$ independent and identically distributed (i.i.d.) realizations $(x^i,x^{+,i})$ from the fixed joint evaluation distribution $\mathbb P_{x,x^+}$, independently of the fitted objects. Set $z=\Phi(x^+)$, $\phi^i=\Phi(x^i)$, $z^i=\Phi(x^{+,i})$, and $c_i=c(x^i)$. Define
\begin{align*}
L_A&:=\max_{c,k}\|a^c-a^k\|,
&d_A&:=\max_c\|a^c-\widehat B^\top\mathrm e^c\|,\\
\widehat f&:=M^{-1}\sum_i(1-(\phi^i)_{c_i})^2,
&\widehat v&:=M^{-1}\sum_i\|z^i-a^{c_i}\|^2,\\
\widehat e&:=M^{-1}\sum_i\|(A-\widehat B)^\top\phi^i\|^2,
&\alpha_\delta&:=\sqrt{\frac{\log(3/\delta)}{2M}}.
\end{align*}
For $0<\delta<1$, put
\begin{align*}
U_f&:=\min(1,\widehat f+\alpha_\delta),
&U_v&:=\min(2,\widehat v+2\alpha_\delta),\\
U_e&:=\min(d_A^2,\widehat e+d_A^2\alpha_\delta).
\end{align*}
Conditional on the fitted objects, with probability at least $1-\delta$,
\begin{equation}
\left(\mathbb{E}_{(x,x^+)\sim\mathbb P}\!\left[\|z-\widehat B^\top\Phi(x)\|^2\right]\right)^{1/2}
\leq L_A\sqrt{U_f}+\sqrt{U_v}+\sqrt{U_e}.
\label{eq:evaluable_risk}
\end{equation}
For $x^+=F(x)$, the left-hand side is $R_{\mathbb P_x}(\Phi,\widehat B)^{1/2}$.
\end{proposition}
\begin{proof}
Write $y=\mathrm e^{c(x)}$, $\phi=\Phi(x)$, and decompose
\[
z-\widehat B^\top\phi
=A^\top(y-\phi)+(z-A^\top y)+(A-\widehat B)^\top\phi.
\]
The first component has norm at most $L_A(1-\phi_{c(x)})$, as in Proposition~\ref{prop:prototype_diameter}.
The summands defining $\widehat f,\widehat v,\widehat e$ lie in $[0,1]$, $[0,2]$, and $[0,d_A^2]$, respectively; the last bound follows by convexity of the norm. Hoeffding's inequality, with failure probability $\delta/3$ for each loss, bounds their expectations by $U_f,U_v,U_e$ simultaneously with probability at least $1-\delta$. If $d_A=0$, the last loss vanishes identically. The $L^2$ triangle inequality proves~\eqref{eq:evaluable_risk}.
\end{proof}
All terms are computable from the fixed fitted objects and evaluation pairs. The term $L_A\sqrt{U_f}$ bounds the prediction change caused by replacing hard labels with soft coordinates. The term $\sqrt{U_v}$ bounds root-mean-square deviations of successor coordinates from estimated class means, including class-mean estimation error. The term $\sqrt{U_e}$ bounds the root-mean-square prediction difference between $\widehat B$ and $A$. For independent trajectories, replace each loss by its within-trajectory average and $M$ by the number of trajectories, using the evaluation law specified in Appendix~\ref{app:fidelity_concentration}.

\begin{remark}[Direct evaluation of a fixed predictor]
\label{rem:direct_risk_upper}
The decomposition in Proposition~\ref{prop:evaluable_risk} helps explain error sources. A direct upper confidence bound is simpler when only the fitted predictor's total risk is needed. Fix $\Phi$ and $\widehat B$ before drawing $M$ i.i.d. transition pairs. Define
\[
H_{\widehat B}:=\max_{1\leq j,c\leq C}
\|\mathrm e^j-\widehat B^\top\mathrm e^c\|^2,
\qquad
\widehat R_M:=\frac1M\sum_{i=1}^M
\|z^i-\widehat B^\top\phi^i\|^2.
\]
Every loss is at most $H_{\widehat B}$. Indeed,
$z-\widehat B^\top\phi=\sum_{j,c}z_j\phi_c(\mathrm e^j-\widehat B^\top\mathrm e^c)$ is a convex combination, and the squared norm is convex. Hoeffding's inequality therefore gives, with probability at least $1-\delta$,
\begin{equation}
R_{\mathbb P}(\Phi,\widehat B)
\leq\min\left\{H_{\widehat B},\,
\widehat R_M+H_{\widehat B}
\sqrt{\frac{\log(1/\delta)}{2M}}\right\}.
\label{eq:direct_risk_upper}
\end{equation}
The square root of this bound gives an upper confidence bound on RMSE for one predictor fixed before evaluation, including under stochastic successors. It is computed directly from the fitted predictor and evaluation transitions. Joint comparison with the lower certificate requires allocating the total failure probability between the bounds. Empirical evaluation of these two fitted-risk upper bounds is outside the scope of the reported campaigns.
\end{remark}

\begin{corollary}[Koopman candidate residual bound]
\label{cor:repaired_koopman_bound}\leavevmode\par
Under the assumptions of Result~\ref{result_210820160745}, and on the confidence event in~\eqref{eq:fidelity_confidence_event}, let
$(\lambda,w)\in\mathbb{C}\times
(\mathbb{C}^C\setminus\{\mathbf 0_C\})$ satisfy
\[
 \widehat B w=\lambda w,
\]
and let $f_{w}$ be defined by~\eqref{eq_220820261131}. Then
\begin{equation}
 \mathbb{E}_{x\sim\mathbb P_x}
 \!\left[
 \left|
 f_{w}(F(x))
 -\lambda f_{w}(x)
 \right|^2
 \right]
 \leq
 \|w\|^2\mathcal B_\delta^2.
 \label{eq:repaired_eigenfunction_bound}
\end{equation}
Moreover, $f_w$ is nonzero as a function on $\mathcal X$ if and only if $P_\Phi^{\mathbb C}(w)\ne0$. Positivity of $\|f_w\|_{L^2(\mathbb P_x)}$ further requires that $f_w$ be nonzero on a set of positive evaluation probability. The left-hand side of~\eqref{eq:repaired_eigenfunction_bound} is the mean squared residual in the composition equation
\[
 f_w\circ F=\lambda f_w.
\]
Boundedness of the composition operator on $L^2(\mathbb P_x)$ and spectral proximity of $\lambda$ require separate operator-theoretic conditions.
\end{corollary}

\begin{proof}
Since $\widehat Bw=\lambda w$,
\[
f_w(F(x))-\lambda f_w(x)
=w^\top\bigl(\Phi(F(x))-\widehat B^\top\Phi(x)\bigr).
\]
Cauchy-Schwarz, expectation, and~\eqref{eq:repaired_risk_bound} give~\eqref{eq:repaired_eigenfunction_bound}.
Proposition~\ref{proposition_070720261027} gives
$\|f_w\|_{\mathcal H_\Phi^{\mathbb C}}^2=\|P_\Phi^{\mathbb C}(w)\|^2$, proving the nonzero-function criterion.
\end{proof}
This corollary uses an eigenvector of the full fitted matrix $\widehat B$ and controls the absolute mean squared residual in $f_w(F(x))=\lambda f_w(x)$. The function's magnitude under the evaluation law affects the residual scale, so a relative-error assessment additionally requires normalization by a positive $L^2(\mathbb P_x)$ norm. The adjoint diagnostic in Section~\ref{sec:finite_koopman_closure} starts from the transposed reduced matrix $(Q_\Phi^\top\widehat BQ_\Phi)^\top$ and tests a different relation. The kernel-section representation check and the held-out adjoint diagnostic therefore provide complementary information.
\begin{corollary}[Bound on deviation from one-hot reference labels]
\label{cor:repaired_association_risk}
On the confidence event in~\eqref{eq:fidelity_confidence_event},
\begin{equation}
 a_\Phi^2
 =
 \mathbb{E}_{x\sim\mathbb P_x}
 \!\left[\|y(x)-\Phi(x)\|^2\right]
 \leq2U_\delta.
 \label{eq:repaired_association_risk}
\end{equation}
\end{corollary}
\begin{proof}
Theorem~\ref{theorem_300320262124} gives
\[
 a_\Phi^2\leq2A_{\mathrm{fid}}.
\]
On the confidence event in~\eqref{eq:fidelity_confidence_event},
$A_{\mathrm{fid}}\leq U_\delta$. Therefore,
\[
 a_\Phi^2\leq2A_{\mathrm{fid}}\leq2U_\delta,
\]
which proves~\eqref{eq:repaired_association_risk}.
\end{proof}
\paragraph{Independent evaluation of the reconstruction-score margin condition.}
Independent evaluation states allow us to estimate how often reconstruction-score margins fall below specified thresholds. For $C\geq2$, choose thresholds $g_c>0$ using only $\mathcal F_{\mathrm{fit}}$, and let $A_{\mathrm{geom}}$ be the margin-based upper bound on assignment loss in~\eqref{eq:kahm_geometric_fidelity}. For a state with reference label $y=\mathrm e^c$, define
\begin{equation}
 W_g(x,y)
 :=
 \vartheta_c^2+
 (1-\vartheta_c^2)
 \mathbbm{1}_{\{G_c(x)<g_c\}}.
 \label{eq:geometric_evaluation_loss}
\end{equation}
For the exact normalized-affinity map, this loss equals $\vartheta_c^2$ when the required margin holds and $1$ when it fails. In each case, it upper-bounds the squared shortfall of the labelled coordinate. The proof of Theorem~\ref{thm:kahm_margin_fidelity} shows that
\[
 (1-(\Phi(x))_c)^2
 \leq W_g(x,y)\leq1.
\]
Moreover,
\[
 \mathbb{E}_{x\sim\mathbb P_x}
 \!\left[W_g(x,y(x))\right]
 =A_{\mathrm{geom}}.
\]
Define
\[
 \widehat A_{\mathrm{geom}}
 :=
 \frac1M\sum_{i=1}^M
 W_g(x'^{i},y'^{i})
\]
and
\begin{equation}
 U_{\mathrm{geom},\delta}
 :=
 \min\left(
 1,\,
 \widehat A_{\mathrm{geom}}
 +\sqrt{\frac{\log(1/\delta)}{2M}}
 \right).
 \label{eq:geometric_confidence_bound}
\end{equation}
Conditional on $\mathcal F_{\mathrm{fit}}$, the variables
$W_g(x'^{i},y'^{i})$ are independent, take values in $[0,1]$, and have common mean $A_{\mathrm{geom}}$. Hoeffding's inequality therefore gives
\[
 \Pr_{\mathcal V\sim\mathbb P_{\mathcal V}}\!\left(
 A_{\mathrm{geom}}
 \leq U_{\mathrm{geom},\delta}
 \right)
 \geq1-\delta.
\]
Since $A_{\mathrm{fid}}\leq A_{\mathrm{geom}}$, it follows on the same event that
\[
 A_{\mathrm{fid}}
 \leq A_{\mathrm{geom}}
 \leq U_{\mathrm{geom},\delta}.
\]
Consequently, any bound derived using the event $A_{\mathrm{fid}}\leq U_\delta$ also holds on the event $A_{\mathrm{geom}}\leq U_{\mathrm{geom},\delta}$ after replacing $U_\delta$ by $U_{\mathrm{geom},\delta}$. For the same evaluation sample and confidence level, the pointwise inequality
\[
 (1-(\Phi(x'^{i}))_{c_i})^2
 \leq W_g(x'^{i},y'^{i})
\]
implies
$\widehat A_{\mathrm{fid}}\leq\widehat A_{\mathrm{geom}}$ and hence
$U_\delta\leq U_{\mathrm{geom},\delta}$. Directly evaluating assignment loss gives an upper bound at least as tight as the margin-based bound. The margin-based bound also explains why the loss is small when the labelled class reconstructs most states better than its competitors. For the implemented coordinates in Remark~\ref{rem:numerical_score_scope}, the corresponding evaluation loss must also count use of the uniform fallback. Define
\begin{equation}
\widetilde W_g(x,\mathrm e^c)
:=\begin{cases}
\vartheta_c^2,&\widetilde G_c(x)\geq g_c\text{ and }x\notin\mathcal U,\\
1,&\text{otherwise}.
\end{cases}
\label{eq:implemented_geometric_evaluation_loss}
\end{equation}
Outside $\mathcal U$, the margin argument applies to the implemented scores. On $\mathcal U$, the bound $1$ covers the squared assignment loss. Thus
$(1-\widetilde\Phi_c(x))^2\leq\widetilde W_g(x,\mathrm e^c)\leq1$.
Applying Hoeffding's inequality to the independent evaluation values
$\widetilde W_g(x'^{i},y'^{i})$ gives, with conditional probability at least $1-\delta$,
\[
\mathbb E_{x\sim\mathbb P_x}
\!\left[(1-\widetilde\Phi_{c(x)}(x))^2\right]
\leq
\min\left\{
1,\,
\frac1M\sum_{i=1}^M\widetilde W_g(x'^{i},y'^{i})
+\sqrt{\frac{\log(1/\delta)}{2M}}
\right\}.
\]
This upper bound can replace $U_\delta$ in results whose coordinate-dependent quantities are all evaluated using the fixed implemented map $\widetilde\Phi$. Omitting $\mathcal U$ can give an invalid geometric upper bound even when the reconstruction-score margin is positive. The direct certificate in Proposition~\ref{prop:exclusion_certificate} already uses the actual coordinate vectors and needs no such modification.

Thresholds or representations chosen after inspecting evaluation outcomes require fresh evaluation data or a confidence guarantee that holds simultaneously over the choices considered. Reconstruction-score margins bound the deviation from hard labels. Distances between class mean successors then bound how much this deviation changes predictions made by the class-mean matrix. Proposition~\ref{prop:evaluable_risk} provides a computable counterpart to Result~\ref{result_210820160745}, using class means estimated before evaluation and independent transition pairs.

% ===== END sections/06_geometry_estimation.tex =====

% ===== BEGIN sections/03_certificate_evaluation.tex =====
\section{Evaluating Prediction-Error Certificates}
\label{sec:exp_certificate}
We ask two empirical questions. \emph{First, do the proposed lower confidence bounds stay below the known optimum, and when are they informative?} An analytical four-state construction permits a direct check and a matched comparison of three structural bounds with an ERM-based comparator using 117,000 archived evaluation datasets. These are repeated samples from fixed designs, not 117,000 separately trained representations. \emph{Second, can the simpler label-based bound reveal a prediction limit for coordinates learned from dynamical data?} We use exploratory data to choose one Van der Pol representation, freeze it, and then evaluate it on fresh independent pairs at prespecified norm limits. The population optimum is unknown in this second setting, so its certificate-to-optimum gap cannot be assessed. We did not numerically test the centered or prototype variants on the learned high-dimensional coordinates.

\subsection{Quantities, trial units, and interpretation}
\label{sec:certificate_reporting}
Unless otherwise specified, the label-based certificates use Proposition~\ref{prop:exclusion_certificate} with $\delta=0.05$. One replicate dataset contains $M$ independent current-successor pairs. From these pairs we compute the assignment loss $\widehat f$, the within-class successor variance $\widehat s$, their confidence corrections $F_\delta$ and $V_\delta$, and the resulting bound $L_{\kappa,\delta}$. A \emph{positive certificate} means $L_{\kappa,\delta}>0$. It does not necessarily exclude a practically useful nonzero error tolerance: to exclude $\varepsilon$, the bound must satisfy $L_{\kappa,\delta}>\varepsilon$. A zero certificate is inconclusive: it is not evidence that a sufficiently accurate matrix exists. The quantity being bounded is $\mathcal D_\kappa(\Phi)$, the smallest population RMSE over all matrices whose spectral norm is at most $\kappa$.

For comparison with a fitted matrix $\widehat B$ held fixed during evaluation, we report its empirical coordinate-space root-mean-square error
\begin{equation}
\widehat{\mathrm{RMSE}}(\widehat B)
=\left[\frac1M\sum_{i=1}^{M}
\|\Phi(x^{+,i})-\widehat B^{\top}\Phi(x^i)\|^2\right]^{1/2}.
\label{eq:certificate_empirical_rmse}
\end{equation}
This RMSE uses the Euclidean distance between dimensionless coordinate vectors, without dividing by target magnitude or by the number of coordinates. The forecasting experiments additionally report the relative squared error $E_h$ and centered score $R_h^2$. When $\|\widehat B\|_2\leq\kappa$, we also calculate $\widehat{\mathrm{RMSE}}(\widehat B)-L_{\kappa,\delta}$. This difference compares the sample error of one fitted predictor with a lower confidence bound on the population optimum. A positive difference does not establish coverage of that optimum: checking coverage requires its population value. These fitted-predictor differences describe empirical errors rather than the coverage probability of the lower confidence bound.

To estimate the frequencies of certificate positivity or error exclusion, we treat each independently generated evaluation dataset as one replicate. We report two-sided, pointwise $95\%$ Clopper--Pearson intervals for the corresponding repeated-evaluation probability~\citep{clopper1934}. These intervals are conditional on the fixed representation and design; they do not quantify variation across newly trained maps. The empirical 5th and 95th percentiles of certificate values, computed by linear interpolation, are descriptive quantiles rather than confidence intervals.

For one evaluation sample, Proposition~\ref{prop:exclusion_certificate} gives a simultaneous guarantee for all norm limits $\kappa$. The $95\%$ guarantees across different datasets, representations, or evaluation laws are pointwise; no joint $95\%$ guarantee over all reported experiments is asserted. Inspecting several norm limits on the same dataset does not create additional independent trials. Results under the two evaluation laws are therefore reported separately.

\subsection{Controlled validation with a known population optimum}
\label{sec:certificate_four_state_evaluation}
The controlled problem has four states, $a=(1,1)^\top$, $b=(1,-1)^\top$, $d=(-1,1)^\top$, and $e=(-1,-1)^\top$. The hard classes are $\{a,b\}$ and $\{d,e\}$. For a parameter $p\in[1/2,1]$, define $\Phi(a)=\Phi(b)=(p,1-p)^\top$ and $\Phi(d)=\Phi(e)=(1-p,p)^\top$. Thus two states in each class have identical current coordinates. We compare identity dynamics with a map that gives different successors to states in the same reference class, $F(a)=a$ and $F(b)=F(d)=F(e)=d$. Section~\ref{ex:kahm_conflict} constructs this coordinate family using KAHM and derives its optimal prediction error. Current states are sampled independently from the uniform law on these four states, and successors follow the specified map. Retaining the four state counts is sufficient to reconstruct all sample statistics, because both the coordinate vectors and successors are fixed for each state. The sampled counts vary across replicates; the analytical illustration in Section~\ref{ex:kahm_conflict} uses fixed balanced counts.

The controlled example isolates the obstruction: the two states in the first class have identical current coordinates, so their different successors force an irreducible error. To investigate how this behavior changes with label agreement, we use 13 fixed design cases: identity dynamics at $p=31/32$ and the nonidentity transition map at
\[
p\in\left\{\frac45,\frac56,\frac{17}{20},\frac{13}{15},\frac9{10},
\frac{14}{15},\frac{19}{20},\frac{29}{30},\frac{31}{32},
\frac{59}{60},\frac{99}{100},1\right\}.
\]
Each case uses nine sample sizes,
\[
M\in\{128,512,1024,2048,4096,8192,32768,131072,524288\},
\]
and 1,000 replicate datasets per case and sample size. Thus the design has $13\times9=117$ settings and 117,000 datasets. The norm limit is $\kappa=\sqrt2$. Each dataset is assessed at six error tolerances: $0.05$, $0.10$, $0.20$, $0.30$, $0.40$, and $0.45$. Exact rational design values distinguish a tolerance equal to the optimum from one strictly above or below it. Appendix~\ref{app:certificate_provenance} identifies the reported configuration and the records needed to reproduce the dataset-level checks.

Under identity dynamics the optimum is zero. Under that nonidentity transition map, for $p\in[1/2,1]$, the analytical quantities are
\begin{equation}
\mathcal D_{\sqrt2}(\Phi)=p-\frac12,
\qquad A_{\mathrm{fid}}=(1-p)^2,
\qquad \ell_{\sqrt2}(\Phi)=\left(3p-\frac52\right)_+.
\label{eq:four_state_validation_limits}
\end{equation}
In this family, current coordinates are constant within each class, so $\tau_\Phi=0$ and the class-centered population bound obeys $\widetilde\ell_{\sqrt2}(\Phi)=\sigma_{\mathrm{res}}=p-1/2=\mathcal D_{\sqrt2}(\Phi)$ under the nonidentity transition map. The equality follows from the population calculation and does not depend on the size of the evaluation sample. Finite KAHM exponents realize the interior values $1/2<p<1$ in this construction. The endpoint maps arise only as limits, and $p=1$ serves as the hard-assignment control. These formulas distinguish three situations without estimating the optimum from data. At $p=4/5$, the optimum is $0.30$, although the limiting certificate is zero. At $p=31/32$, the optimum is $15/32$ and the limiting certificate is $13/32$. Their difference, $1/16$, persists after sampling uncertainty vanishes. At $p=14/15$, the limiting certificate is exactly $0.30$. Convergence to this boundary does not by itself establish that a finite-sample certificate will strictly exceed $0.30$.

Figure~\ref{fig:four_state_analytical_limits} compares the exact minimum RMSE with the label-based population lower bound in~\eqref{eq:four_state_validation_limits} as a function of $p$.
\begin{figure}[htbp]
\centering
\includegraphics[width=0.88\linewidth]{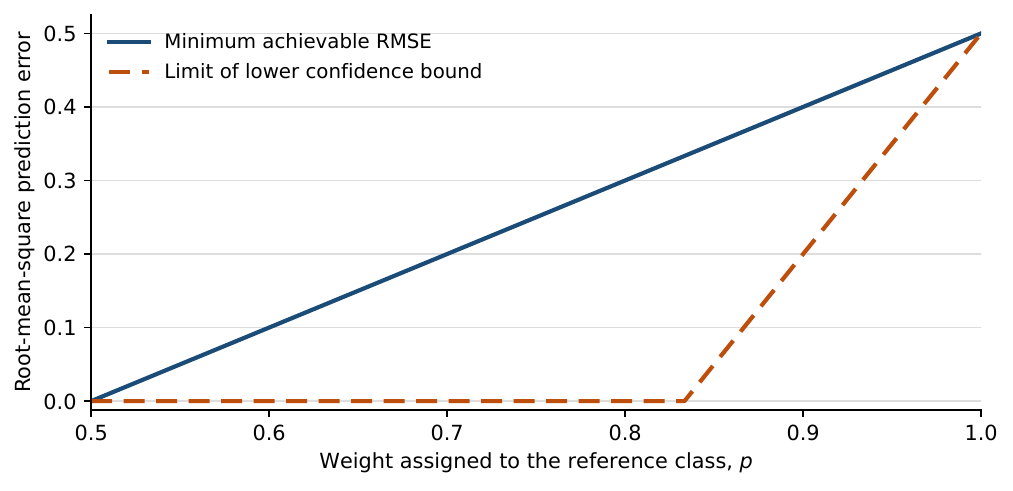}
\caption{Exact minimum RMSE (solid) and the \emph{label-based} population lower bound (dashed) from~\eqref{eq:four_state_validation_limits}, at $\kappa=\sqrt2$. The centered population lower bound coincides analytically with the solid curve and is not separately plotted. The dashed bound vanishes for $p\leq5/6$ although the optimum is positive for $p>1/2$.}
\Description{The optimal root-mean-square error rises linearly from zero at p equal to one half to one half at p equal to one. The limiting lower confidence bound remains zero until p exceeds five sixths and then increases to the same endpoint.}
\label{fig:four_state_analytical_limits}
\end{figure}

Across the 117,000 fixed-design evaluation datasets, independent numerical verification found no lower confidence bound above the analytically known optimum and no erroneous tolerance exclusion. All 49,271 zero certificates were included in this analysis. In each of the 117 case-by-sample-size settings, no violations occurred in 1,000 replicates; the corresponding pointwise two-sided $95\%$ exact binomial interval for the violation probability is approximately $[0,0.00368]$. Zero observed violations therefore do not imply a zero underlying violation probability. Appendix~\ref{app:certificate_provenance} documents the independent recomputation from state counts, rational and high-precision checks, numerical-tolerance tests, and reproduction records. Knowledge of the analytical optimum also permits separate assessment of sampling uncertainty and intrinsic looseness of the population bound.

\paragraph{Structural certificates versus empirical minimum-risk bounds.}
The four bounds were evaluated on every archived state-count sample. At $\kappa=\sqrt2$, the ERM-based comparator uses the exact identity $\widehat R^*_{\sqrt2,M}=\widehat s$ and therefore needs no numerical optimization. For $95\%$ joint confidence across the four bounds on one prespecified dataset, we assign $\delta=0.05/4$ to each method. The allocation does not give simultaneous coverage across all cases, sample sizes, and replicates. Table~\ref{tab:four_state_erm_historical} summarizes 1,000 archived replicates per $M$ at $p=31/32$, where the true population optimum is $15/32$.
\begin{table}[tbp]
\centering\small
\caption{Means of four RMSE lower bounds over 1,000 archived replicate datasets per $M$ ($p=31/32$, $\kappa=\sqrt2$). Bonferroni allocation: $\delta=0.0125$ per method, with joint $95\%$ confidence within one replicate only. The last column reports the observed fraction of replicates in which the prototype bound exceeds the ERM bound.}
\label{tab:four_state_erm_historical}
\begin{tabular}{@{}r r r r r r@{}}
\toprule
$M$ & Label & Centered & Prototype & ERM & Proto $>$ ERM\\
\midrule
1,024 & 0.1289 & 0.0000 & \textbf{0.1468} & 0.0000 & 100\%\\
4,096 & 0.2825 & 0.0951 & \textbf{0.3126} & 0.2913 & 100\%\\
8,192 & 0.3228 & 0.1635 & \textbf{0.3593} & 0.3529 & 100\%\\
32,768 & 0.3671 & 0.2616 & 0.4144 & \textbf{0.4148} & 0\%\\
\bottomrule
\end{tabular}
\end{table}
The four-state design specifies prototypes $h^1=(p,1-p)^\top$ and $h^2=(1-p,p)^\top$ before evaluation, so their within-class approximation error is zero. This advantage need not transfer to prototypes estimated for learned representations. The centered finite-sample correction remains conservative despite its exact population limit, and the ERM-based bound becomes tighter than the fixed-prototype bound at larger $M$. None of the four bounds exceeded the known population optimum in the archived replicates. Reproduction identifiers are given in Appendix~\ref{app:certificate_provenance}.

\subsection{Pilot analysis and selection of the dynamical evaluation}
\label{sec:certificate_pilot_diagnosis}
The pilot uses the affinity exponent $\omega$ introduced in Section~\ref{sec:intro_kahm_overview}. Increasing this exponent sharpens differences in relative reconstruction affinity. However, the numerical implementation additionally applies the absolute normalization threshold in~\eqref{eq:implemented_affinity_normalization}; increasing $\omega$ can therefore activate the uniform-coordinate fallback. The exact construction and score domain are specified in Sections~\ref{sec:mathematical_background} and~\ref{sec:method}.

The dynamical pilot used a fitted Duffing representation with $C=10$, $\omega=2$, and a Van der Pol representation with $C=25$, $\omega=4$. For each system, the pilot considered stochastic dynamics with the process-noise level used in training and dynamics with that process noise removed. The corresponding Van der Pol evaluation laws are specified in Section~\ref{sec:certificate_vdp_protocol}. The pilot used sample sizes $128$, $512$, and $2048$, with eight replicates per setting. All 96 datasets, comprising 86,016 retained pairs, were included. Every certificate was zero at all four tested norm limits, $\{1,\|\widehat B\|_2,\sqrt2,2\|\widehat B\|_2\}$.

The pilot results motivated an evaluation of larger assignment exponents. If the confidence corrections are removed while the observed statistics are held fixed, positivity requires
\begin{equation}
\kappa<\widehat\kappa_{\mathrm{crit}}
:=\sqrt{\frac{\widehat s}{2\widehat f}},
\qquad \widehat f>0.
\label{eq:pilot_critical_budget}
\end{equation}
Table~\ref{tab:certificate_pilot_diagnosis} applies this calculation to the mean pilot statistics at $M=2048$. The fitted matrices have norms well above these estimated critical values. The calculation explains why increasing the sample size alone was unlikely to yield informative certificates for the pilot representations at the fitted norms: even after omitting the confidence corrections, the pilot estimates gave nonpositive bounds. It is a planning calculation based on estimated quantities, not a population guarantee.

\begin{table}[tbp]
\centering
\small
\caption{Pilot evaluation of the fitted representations at $M=2048$, with eight replicate datasets per row. Bars denote means across replicates. The final column gives the planning value $\sqrt{\overline{\widehat s}/(2\overline{\widehat f})}$. The displayed inputs are rounded and may not reproduce the final digits of that column.}
\label{tab:certificate_pilot_diagnosis}
\begin{tabular}{@{}llrrrr@{}}
\toprule
System & Evaluation law & $\|\widehat B\|_2$ & $\overline{\widehat f}$ & $\overline{\widehat s}$ & Critical norm \\
\midrule
Duffing & Matched stochastic & 1.005768 & 0.780756 & 0.000217 & 0.011796 \\
Duffing & Deterministic & 1.005768 & 0.780787 & 0.000222 & 0.011914 \\
Van der Pol & Matched stochastic & 1.078133 & 0.360983 & 0.017542 & 0.155876 \\
Van der Pol & Deterministic & 1.078133 & 0.361292 & 0.017328 & 0.154858 \\
\bottomrule
\end{tabular}
\end{table}

To identify a representation for which the certificate might be informative, the exploratory study varied the assignment exponent while keeping reference centers and KAHM models fixed. Each candidate prediction matrix was fitted using training data only. On that basis, we selected the Van der Pol map with $\omega=128$ and then evaluated it on newly generated, independent pairs. For Duffing, the pilot estimate of the population lower bound remained nonpositive at the fitted norm for the tested exponents through $\omega=2048$. The subsequent evaluation therefore focused on the selected Van der Pol representation. Consequently, the certificate evaluation describes the selected Van der Pol representation under the two specified transition laws.

Exploratory (pilot) pairs were used to choose the map and plan the evaluation. The representation, refitted matrix, two transition laws, sample size, replicate count, and four norm limits were then fixed \emph{before} generating the independent evaluation pairs. We use \emph{confirmatory} only in this restricted sense: testing the prespecified certificate on new data for the selected map, not demonstrating performance across newly trained representations. Appendix~\ref{app:certificate_provenance} identifies the protocol and numerical-verification records.

\subsection{Fixed representation and sampling of independent transition pairs}
\label{sec:certificate_vdp_protocol}
The selected Van der Pol representation has $C=25$, $\omega=128$, and $\tau=10^{-6}$. The state-space $K$-means centers and class-specific KAHM reconstruction models are reused without refitting. Coordinate weights are recomputed with exponent $128$ in place of $4$, after which the linear prediction matrix is refitted. This changes both the inputs to the predictor and its successor-coordinate targets.

The evaluated map includes the absolute threshold in~\eqref{eq:implemented_affinity_normalization}: it returns uniform coordinates whenever the sum of powered affinities is at most $10^{-12}$. This rule can be triggered even when all powered affinities are representable and strictly positive. Consequently, the effect of changing $\omega$ includes both a change in relative weights and a possible change in which states receive uniform coordinates. The reported summaries do not give the frequency of this fallback on training states, current evaluation states, or successors.

The matrix is refitted using the three training trajectories, with seeds $0,1,2$ and 1,200 transition pairs per trajectory, giving 3,600 training pairs in total. Normalized least mean squares (NLMS) uses step size $\beta=0.1$, identity initialization, and 20 epochs in the fixed sample order, without projection. The fitted matrix is held fixed throughout evaluation and has
\[
\|\widehat B\|_2=1.0203452494452896.
\]
The hard reference label of an evaluation state is its nearest stored $K$-means center in Euclidean state coordinates, with the smallest index used to break a tie.

Both evaluation laws use the Van der Pol dynamics $\dot q=p$, $\dot p=(1-q^2)p-q$, with state $(q,p)$ and step size $\Delta t=0.02$. Initial states follow $\mathcal N((2,0)^\top,0.05^2I_2)$. Under stochastic evaluation, independent Gaussian process noise with standard deviation $10^{-5}$ per coordinate is added after every fourth-order Runge-Kutta step. This is the same process-noise setting as in training; the tables call this law \emph{matched stochastic}. Under deterministic evaluation, the process noise is removed while the initial state remains random. Both laws use the same fixed representation and fitted matrix.

For each retained pair $i$, we independently draw an initial condition, a trajectory-noise stream, and a simulation time $t_i$ uniformly from $\{0,\ldots,1199\}$. We simulate that trajectory through step $t_i+1$ and retain only $(x_{t_i}^i,x_{t_i+1}^i)$, denoted by $(x^i,x^{+,i})$. We then start a new trajectory with new random draws for the next pair. Thus each trajectory contributes exactly one evaluation pair. The joint evaluation law is a uniform mixture over simulation times under the specified initialization and dynamics, and the sampling design gives independent pairs. Consecutive transitions from one trajectory are not treated as independent evaluation observations.

Each replicate contains $M=16{,}384$ pairs. With 32 replicates per law, the design comprises 64 datasets and 1,048,576 retained pairs. The reported analysis includes all 64 scheduled datasets. Appendix~\ref{app:certificate_provenance} specifies the seed allocation and identifies the records required to reproduce the sampling and numerical checks.

\subsection{Confirmatory results across norm limits}
\label{sec:certificate_vdp_results}
At the fitted matrix's spectral norm, all 32 evaluation replicates under each Van der Pol transition law produce positive lower confidence bounds. This indicates a nonzero certified obstruction for the fixed map and norm-constrained predictor class, not a measured gap to the unknown population optimum. The median certificate is $0.084727$ under matched stochastic evaluation and $0.084269$ under deterministic evaluation; the corresponding minima are $0.076333$ and $0.075401$. Table~\ref{tab:certificate_vdp_primary} gives the primary summaries, and Table~\ref{tab:certificate_vdp_budgets} shows how the certificate weakens for larger norm limits.

For either law, 32 positive outcomes in 32 replicates give a pointwise $95\%$ exact binomial interval of $[0.8911,1]$. This interval concerns the probability of obtaining a positive certificate when the same fixed design is evaluated again.

\begin{table}[tbp]
\centering
\small
\caption{Van der Pol evaluation using independent pairs at the fitted norm $\kappa=1.0203452494452896$, with $\delta=0.05$ and $M=16{,}384$ pairs per replicate. Each law has 32 independent replicate datasets. Brackets give the descriptive 5th and 95th percentiles across replicates, rather than confidence limits. Empirical RMSE is defined in~\eqref{eq:certificate_empirical_rmse}.}
\label{tab:certificate_vdp_primary}
\begin{tabular}{@{}lcrrr@{}}
\toprule
Evaluation law & Positive $L$ & Median $L$ [5th, 95th] & Minimum $L$ & Median RMSE \\
\midrule
Matched stochastic & 32/32 & 0.084727 [0.078444, 0.089910] & 0.076333 & 0.276233 \\
Deterministic & 32/32 & 0.084269 [0.077720, 0.092126] & 0.075401 & 0.275797 \\
\bottomrule
\end{tabular}
\end{table}

The assignment loss and within-class successor variance explain why the bound is positive at the fitted norm. Under matched stochastic evaluation, their means across replicates are $\overline{\widehat f}=0.020818$ and $\overline{\widehat s}=0.116479$. Under deterministic evaluation, the corresponding means are $0.020879$ and $0.116581$. At $M=16{,}384$, the amount subtracted from $\widehat s$ is $\sqrt{2\log(2/0.05)/M}\simeq0.02122$. For every replicate, the resulting lower bound on successor variation satisfies $\sqrt{V_\delta}>\kappa\sqrt{2F_\delta}$. Each reported certificate is calculated on an individual replicate before aggregation. Applying the nonlinear formula to mean statistics would not produce the reported mean certificate.

\begin{table}[tbp]
\centering
\small
\caption{All four prespecified norm limits in the Van der Pol study. Counts give positive certificates among 32 replicates per law. The fitted-matrix column applies under both laws: $\widehat B$ lies outside the class with $\kappa=1$ and inside the other three classes.}
\label{tab:certificate_vdp_budgets}
\begin{tabular}{@{}lcccrr@{}}
\toprule
 & Fitted matrix & \multicolumn{2}{c}{Positive $L$} & \multicolumn{2}{c}{Median $L$} \\
\cmidrule(lr){3-4}\cmidrule(l){5-6}
Norm limit & permitted? & Stochastic & Deterministic & Stochastic & Deterministic \\
\midrule
$1$ & No & 32/32 & 32/32 & 0.089186 & 0.088660 \\
$\|\widehat B\|_2$ & Yes & 32/32 & 32/32 & 0.084727 & 0.084269 \\
$\sqrt2$ & Yes & 11/32 & 8/32 & 0 & 0 \\
$2\|\widehat B\|_2$ & Yes & 0/32 & 0/32 & 0 & 0 \\
\bottomrule
\end{tabular}
\end{table}

Increasing $\kappa$ permits stronger amplification of small coordinate differences and weakens the lower bound. At $\kappa=\sqrt2$, 11 of 32 matched stochastic replicates and 8 of 32 deterministic replicates have positive certificates. Their pointwise $95\%$ exact binomial intervals for positivity probability are $[0.1857,0.5319]$ and $[0.1146,0.4340]$, respectively. At twice the fitted norm, all certificates are zero; the corresponding interval is $[0,0.1089]$ under either law. Thus positivity at the fitted norm does not imply positivity for a substantially larger matrix class.

The fitted matrix's empirical RMSE exceeds $L_{\kappa,\delta}$ in all 192 admissible comparisons. These comparisons are not independent because each evaluation dataset contributes results at three norm limits. At the fitted norm, the smallest observed RMSE-minus-certificate differences are $0.186606$ for matched stochastic evaluation and $0.182506$ for deterministic evaluation. They quantify sample performance of the fitted predictor relative to the certificate, not the certificate's gap from the unknown minimum population RMSE $\mathcal D_\kappa(\Phi)$; its coverage cannot be evaluated in this design.

Changing $\omega$ from $4$ to $128$ changes both the predictor inputs and the successor targets and can also change activation of the uniform-coordinate fallback. The positive certificates therefore concern the coordinate-prediction task defined by the $\omega=128$ map. The close medians under the two evaluation laws describe this particular design and do not establish that process noise has no effect. Positivity under deterministic evaluation shows that the bound can detect a prediction limitation even when the successor is determined by the current state: different states assigned to the same reference class can still have different successor coordinates.

% ===== END sections/03_certificate_evaluation.tex =====

% ===== BEGIN sections/07_selected_experiments.tex =====
\section{Forecasting Experiments}
\label{sec:experiments}
A nearly constant representation can give deceptively small forecast errors because there is little variation to predict. We therefore analyze coordinate variation alongside forecasting accuracy. We also distinguish comparisons in which every method predicts the \emph{same} output representation from comparisons in which each method predicts its own coordinates. Experiments on oscillators and fixed randomized behavior policies assess how predictive performance changes with forecasting horizon. These experiments evaluate fitted predictors, not the best population error over an entire norm-bounded class. Appendix~\ref{app:protocol} details the fitting and data-selection protocols; Appendix~\ref{app:supporting_diagnostics} reports simplex and spectral diagnostics; and Appendix~\ref{app:certificate_provenance} contains the independent certificate-evaluation records.

\subsection{Protocol, comparison targets, and error measures}
\label{sec:exp_protocol}\label{sec:algorithm}
For the oscillator studies, the state is $x=(q,p)$, with
\begin{equation}
\begin{aligned}
\text{Duffing:}\quad &\dot q=p,&\dot p&=-0.25p+q-q^3,&\Delta t&=0.03,\\
\text{Van der Pol:}\quad &\dot q=p,&\dot p&=(1-q^2)p-q,&\Delta t&=0.02.
\end{aligned}
\label{eq:exp_oscillators}
\end{equation}
Numerical trajectories use fourth-order Runge-Kutta integration with additive process noise after each step. The Duffing reference diagnostics use noise standard deviation $10^{-4}$ per coordinate; Van der Pol and the expanded Duffing search use $10^{-5}$. The Duffing reference and common-target studies start near $(0.7,0)$. The expanded Duffing search uses seeds $0,1,2$ with respective initial centers $(-1.6,0.4)$, $(-1.2,-0.6)$, and $(-0.4,0.9)$. Each initial state is its center plus $\mathcal N(0,0.05^2I_2)$ noise. Van der Pol starts near $(2,0)$ with the same initial-state noise covariance. New evaluation trajectories use separate random seeds under the initialization law specified for their study. A horizon counts integration steps: 200 steps correspond to 6 model time units for Duffing and 4 for Van der Pol. Appendix~\ref{app:protocol} gives the fitting, selection, numerical-score, and episode conventions.

All reported studies use the observed state-coordinate scales without additional standardization. For each candidate representation, $K$-means reference classes and KAHM reconstruction models are fitted to training states. The coordinate map remains fixed during subsequent prediction-matrix fitting. NLMS uses step size $0.1$, identity initialization, and 20 epochs in the fixed sample order, without projection unless stated otherwise. Choosing a different representation changes both the predictor inputs and the coordinates it must predict.

For horizon $h$, let $I_h$ contain starts whose targets remain in the same trajectory or episode. Each start identifies both its episode and its time index; in the prediction metrics, $t$ denotes this episode--time pair. Write $z_{t,h}=\Phi(x_{t+h})$ and let $\widehat z_{t,h}$ be the method's predicted coordinates. We use
\begin{equation}
E_h=\frac{\sum_{t\in I_h}\|z_{t,h}-\widehat z_{t,h}\|^2}{\sum_{t\in I_h}\|z_{t,h}\|^2},
\qquad
R_h^2=1-\frac{\sum_{t\in I_h}\|z_{t,h}-\widehat z_{t,h}\|^2}{\sum_{t\in I_h}\|z_{t,h}-\overline z_h\|^2},
\label{eq:exp_metrics}
\end{equation}
where $\overline z_h$ is the mean of the evaluation targets at horizon $h$. The relative squared error $E_h$ describes forecast error relative to the uncentered target magnitude; it is a ratio of squared errors, \emph{not} an RMSE. The centered score $R_h^2$ compares the same prediction error with the error of always predicting $\overline z_h$. A score of one is perfect, zero gives the same error as that constant reference, and a negative score gives a larger error. Because $\overline z_h$ is calculated from the evaluation targets, it serves as a descriptive reference rather than a constant estimated from training data. The centered score is undefined when every target vector is identical. Both scores are dimensionless. The certificate instead reports Euclidean coordinate RMSE without scaling by target magnitude or variation.

With a fixed output map, every method predicts the same coordinate vector, permitting direct comparison of prediction errors. When each method predicts its own representation, both the target and its variation can change. Those comparisons characterize representation-specific predictability; comparison on a shared prediction task requires a common target.

Most forecasting tables summarize three observed units rather than a broad distribution of learned representations. In Tables~\ref{tab:exp_selection}, \ref{tab:exp_trajectory_count}, and~\ref{tab:exp_policy}, dispersion describes variation across fitted runs under the specified seed protocols; Table~\ref{tab:exp_external} instead reports evaluation trajectories conditional on one fit per method. Some settings contributed to model selection, and some control runs share reset seeds. Table captions and Appendix~\ref{app:protocol} specify these distinctions. The certificate replicates based on independent transition pairs are analyzed separately.

\begin{algorithm}[tbp]
\caption{Construction, fitting, and assessment of a soft coordinate map}
\label{alg:kahkm}
\begin{algorithmic}[1]
\Require Training data; selection data if used; candidate class counts and exponents; evaluation protocol with documented data roles
\For{each candidate representation}
\State Fit the reference classes and KAHM models on training states.
\State Construct $\Phi$ and fit a linear prediction matrix on training transitions.
\State Measure prediction error and coordinate variation on the selection data.
\EndFor
\State Apply the stated selection rule and fix the selected map and matrix.
\State For a certificate, draw fresh independent pairs using the separately fixed design.
\State Compute the class statistics and certificate without refitting the map or matrix.
\State For a common-target forecasting comparison, evaluate every method against the same fixed output map.
\State Record whether forecasting evaluation data were also used in selection.
\Ensure Results with explicit sampling units, data roles, and comparison targets
\end{algorithmic}
\end{algorithm}

\subsection{Coordinate variation and the choice of prediction target}
\label{sec:exp_abstraction}
A soft-coordinate vector can be close to uniform for most states. Its squared magnitude is then close to $1/C$, even when it varies little across states. Consequently, a small $E_h$ can coexist with poor prediction of that limited variation. To measure variation in the representation separately from prediction error, define the sample mean and covariance of $N$ coordinate vectors by
\[
\bar\phi=N^{-1}\sum_{i=1}^N\Phi(x^i),\qquad
\Sigma_\Phi=N^{-1}\sum_{i=1}^N(\Phi(x^i)-\bar\phi)(\Phi(x^i)-\bar\phi)^\top.
\]
The total coordinate variance is $V_{\mathrm{coord}}=\operatorname{tr}(\Sigma_\Phi)$. Simplex vectors have squared norm at most one, which gives $V_{\mathrm{coord}}\leq1-\|\bar\phi\|^2$. We therefore report the normalized coordinate variance
\begin{equation}
\rho_{\mathrm{var}}=\frac{\operatorname{tr}(\Sigma_\Phi)}{1-\|\bar\phi\|^2},
\label{eq:retained_variation}
\end{equation}
with value zero when the denominator is zero. For a constant nonvertex vector, the numerator is zero and the denominator is positive, so the ratio is again zero. The statistic lies between zero and one and measures variance relative to the largest possible variance of simplex vectors having the same mean.

To distinguish variation across many active directions from variation concentrated in one direction, let $\lambda_1,\ldots,\lambda_r$ be the positive covariance eigenvalues and put $p_j=\lambda_j/\sum_k\lambda_k$. The entropy effective rank and its normalization are
\begin{equation}
r_{\mathrm{eff}}=\exp\!\left(-\sum_j p_j\log p_j\right),\qquad
\rho_{\mathrm{rank}}=\frac{r_{\mathrm{eff}}}{C-1}.
\label{eq:retained_effective_rank}
\end{equation}
They are set to zero for zero covariance, including $C=1$. Centered simplex vectors lie in a $(C-1)$-dimensional subspace. The effective rank equals $r$ when all positive eigenvalues are equal and approaches one when a single eigenvalue dominates. It summarizes the sample covariance spectrum, whereas $r_\Phi$ denotes the dimension of the full coordinate function space.

The Van der Pol selection study without added observation noise retains the process noise specified in Section~\ref{sec:exp_protocol}. It ranks each candidate by the mean centered error $1-R_h^2$ over $h\in\{1,10,50,100,200\}$ and three fitted runs. Let $k_*$ be the candidate with the smallest mean error. The one-standard-error set contains every candidate whose mean error is at most the mean for $k_*$ plus the estimated standard error for $k_*$. All candidates therefore use the same threshold. Among the candidates in that set, the rule selects the candidate with the largest $\rho_{\mathrm{var}}$ and uses $\rho_{\mathrm{rank}}$ to resolve any remaining tie. With only three runs, this rule is a model-selection heuristic. Table~\ref{tab:exp_selection} gives all three candidates in the resulting set, from which $(C,\omega)=(25,6)$ is selected. Each row has a different coordinate target, and all values are selection-stage statistics.

\begin{table}[tbp]
\centering\small
\caption{The complete one-standard-error set for Van der Pol selection without added observation noise. Centered error is the per-run mean of $1-R_h^2$ over five horizons, summarized as mean $\pm$ standard error across three fitted runs. Variation diagnostics are averaged across runs. All values were used in model selection.}
\label{tab:exp_selection}
\begin{tabular}{@{}rrrcrrrc@{}}
\toprule
$C$ & $\omega$ & Mean $R^2$ & Centered error & $\rho_{\mathrm{var}}$ & $r_{\mathrm{eff}}$ & $\rho_{\mathrm{rank}}$ & Selected \\
\midrule
25 & 6 & 0.8885 & $0.1115\pm0.0034$ & 0.3545 & 14.26 & 0.5942 & Yes \\
30 & 6 & 0.8865 & $0.1135\pm0.0030$ & 0.3219 & 15.58 & 0.5373 & No \\
20 & 4 & 0.8865 & $0.1135\pm0.0063$ & 0.2720 & 10.42 & 0.5482 & No \\
\bottomrule
\end{tabular}
\end{table}

A second selection procedure averages performance over five levels of added training observation noise, $\sigma_{\mathrm{obs}}\in\{0,0.005,0.01,0.02,0.05\}$. At each noise level, a run receives the score $\tfrac14(1-R_{50}^2)+\tfrac14(1-R_{100}^2)+\tfrac12(1-R_{200}^2)$; the procedure then averages over the five noise levels. Appendix~\ref{app:protocol} specifies how training-coordinate standard deviations scale the perturbations. This procedure selects $(25,4)$ as the only candidate in its one-standard-error set. The smallest mean error is $0.31165$, with estimated standard error $0.00445$. This setting is used in the common-target forecasting comparison in Section~\ref{sec:exp_external} and differs from the certificate representation with $\omega=128$. The perturbation seed depends on the candidate, so candidate comparisons also change the realized training perturbations.

Table~\ref{tab:exp_ablation} compares ways to construct coordinate maps, with each map defining its own prediction target. It fixes $(C,\omega)=(15,12)$ for Duffing and $(20,4)$ for Van der Pol, and compares KAHM affinities with hard, radial-basis-function (RBF), and distance-based assignments to centroids. The RBF bandwidth and distance scale are calculated from training distances as specified in Appendix~\ref{app:protocol}. The Duffing diagnostic reuses the trajectory used to select its configuration and therefore provides a post-selection comparison.

\begin{table}[tbp]
\centering
\caption{Comparison of coordinate constructions at fixed oscillator configurations. Entries are means over three runs; sample standard deviations are supplied in the released per-run summary files. Each map is evaluated in its own coordinate space using the same NLMS estimator. Thus the table measures the predictability of each map's own coordinates.}
\label{tab:exp_ablation}
\small
\setlength{\tabcolsep}{5pt}
\begin{tabular}{lrrrrrr}
\toprule
& \multicolumn{3}{c}{Duffing, $C=15$, $\omega=12$}
& \multicolumn{3}{c}{Van der Pol, $C=20$, $\omega=4$}\\
\cmidrule(lr){2-4}\cmidrule(lr){5-7}
Coordinate map & $E_1$ & $R_1^2$ & $E_{200}$ & $E_1$ & $R_1^2$ & $E_{200}$\\
\midrule
KAHM affinities & $5.74\!\times\!10^{-6}$ & 0.9942 & $5.44\!\times\!10^{-4}$ & $1.89\!\times\!10^{-3}$ & 0.9976 & 0.180\\
$K$-means RBF & $6.87\!\times\!10^{-5}$ & 0.9950 & $2.45\!\times\!10^{-2}$ & $3.01\!\times\!10^{-2}$ & 0.9672 & 0.819\\
$K$-means distance & $1.92\!\times\!10^{-3}$ & 0.9924 & 0.583 & $9.21\!\times\!10^{-2}$ & 0.9001 & 1.07\\
$K$-means hard & $2.12\!\times\!10^{-2}$ & 0.9411 & 0.393 & 0.112 & 0.8792 & 1.09\\
\bottomrule
\end{tabular}
\end{table}

For the reported configurations, KAHM affinities produce the smallest 200-step relative squared error among the listed coordinate constructions. Because each construction defines its own target, this is evidence about representation-specific predictability, not superiority on a common prediction task. The shared-target analysis is presented in Section~\ref{sec:exp_external}.

\subsection{Learning from additional trajectories}
\label{sec:exp_robustness}
This experiment changes the number of complete training trajectories while keeping the evaluation trajectories fixed. For each training count, models are fitted using three disjoint seed blocks beginning at $0$, $1000$, and $2000$. Within a block, the compared configurations use the same training trajectories, with 900 transitions per trajectory. Evaluation uses separate trajectories $100,101,102$, with 800 transitions each. The configurations are $(25,4)$, selected across observation-noise levels; $(25,6)$, selected without added observation noise; and the reference $(20,4)$. Each map is reconstructed and each matrix refitted from the available training data. Consequently, increasing the number of trajectories changes both representation construction and matrix estimation.

For each fitted run, define the average centered error over the reported horizons by
\[
J:=\frac{1}{5}\sum_{h\in\{1,10,50,100,200\}}(1-R_h^2).
\]
Use $J_{\mathrm{noise}}$ for the configuration selected across noise levels, $J_{\mathrm{clean}}$ for the configuration selected without added observation noise, and $J_{\mathrm{ref}}$ for the reference. The paired differences are $\Delta J_{\mathrm{clean-noise}}:=J_{\mathrm{clean}}-J_{\mathrm{noise}}$ and $\Delta J_{\mathrm{ref-noise}}:=J_{\mathrm{ref}}-J_{\mathrm{noise}}$. Positive differences favor the configuration selected across noise levels. Pairing refers to the common training trajectories within a seed block; each fitted map still defines its own coordinate target. The quantity $J$ averages centered error and differs from the magnitude-normalized squared error $E_h$ in~\eqref{eq:exp_metrics}.

\begin{table}[tbp]
\centering
\caption{Van der Pol prediction as the number of training trajectories increases. The configuration selected across observation-noise levels, the configuration selected without added observation noise, and the reference configuration are $(C,\omega)=(25,4)$, $(25,6)$, and $(20,4)$, respectively. The $R^2$ columns give means over five horizons, summarized as mean $\pm$ standard deviation across three training-seed blocks. The last two columns give paired differences in centered error $J$, as mean $\pm$ standard error. These are descriptive comparisons of representation-specific targets.}
\label{tab:exp_trajectory_count}
\small
\setlength{\tabcolsep}{3.2pt}
\resizebox{\linewidth}{!}{%
\begin{tabular}{rccccc}
\toprule
Training trajectories &
Across noise levels $R^2$ &
No added noise $R^2$ &
Reference $R^2$ &
$\Delta J_{\mathrm{clean-noise}}$ &
$\Delta J_{\mathrm{ref-noise}}$\\
\midrule
1 &
$0.8745\pm0.0055$ &
$0.8768\pm0.0134$ &
$0.8795\pm0.0081$ &
$-0.0022\pm0.0055$ &
$-0.0049\pm0.0077$\\

2 &
$0.8876\pm0.0050$ &
$0.8840\pm0.0047$ &
$0.8889\pm0.0125$ &
$0.0036\pm0.0030$ &
$-0.0013\pm0.0064$\\

3 &
$0.8930\pm0.0062$ &
$0.8936\pm0.0065$ &
$0.8982\pm0.0044$ &
$-0.0005\pm0.0033$ &
$-0.0052\pm0.0051$\\

5 &
$0.8985\pm0.0029$ &
$0.8861\pm0.0027$ &
$0.8958\pm0.0070$ &
$0.0124\pm0.0024$ &
$0.0028\pm0.0036$\\

8 &
$0.9021\pm0.0100$ &
$0.8891\pm0.0045$ &
$0.9077\pm0.0030$ &
$0.0130\pm0.0032$ &
$-0.0056\pm0.0070$\\
\bottomrule
\end{tabular}%
}
\end{table}

The ordering changes across the five training counts. At eight trajectories, mean $R^2$ over the five horizons is $0.9021$ for the configuration selected across noise levels, $0.8891$ for the configuration selected without added observation noise, and $0.9077$ for the reference. These three-fit comparisons describe prediction of each map's own coordinates under the specified local initialization laws. They show no consistent ordering across all training counts.

\subsection{Comparing predictors on the same abstract target}
\label{sec:exp_external}
Direct coordinate prediction starts from $\Phi(x_t)$ and repeatedly applies the fitted matrix. A state-space predictor forecasts $\widehat x_{t+h}$ and then encodes that forecast as $\Phi(\widehat x_{t+h})$. Both outputs are compared with the same target, $\Phi(x_{t+h})$. Holding this output map fixed makes the errors directly comparable within a system.

We compare direct coordinate prediction with dynamic mode decomposition (DMD) on the observed state and extended dynamic mode decomposition (EDMD) on polynomial state features. The EDMD dictionaries include every monomial through degree two or three, including the constant. The polynomial degrees and ridge penalty $10^{-8}$ are prespecified. Each method is fitted once using three trajectories of 1,200 transitions with seeds $0,1,2$. It is then evaluated on three trajectories of 800 transitions with seeds $100,101,102$. For multistep EDMD, polynomial features are computed at the initial state and propagated by powers of the fitted feature matrix; the two state coordinates are then extracted. The common output map is $(10,2)$ for Duffing, selected in the expanded Duffing search, and $(25,4)$ for Van der Pol, selected using the centered criterion across training observation-noise levels. Forecast states are encoded using the numerical conventions in Appendix~\ref{app:protocol}.

\begin{table}[tbp]
\centering
\caption{Oscillator prediction in a fixed KAHM output representation.
Within each system all methods share the same target:
$(C,\omega)=(10,2)$ for Duffing and $(25,4)$ for Van der Pol.
Entries are mean $R_h^2\pm$ standard deviation over three evaluation
trajectories for one fitted model per method. The DMD/EDMD polynomial
degrees and ridge regularization are prespecified.}
\label{tab:exp_external}
\small
\setlength{\tabcolsep}{4pt}
\begin{tabular}{llccc}
\toprule
System & Method & $R_1^2$ & $R_{50}^2$ & $R_{200}^2$\\
\midrule
Duffing
& Coordinate NLMS
& $0.9835\pm0.0003$
& $0.5056\pm0.0167$
& $-0.3155\pm0.0552$\\
& State DMD
& $0.9895\pm0.0021$
& $-0.7108\pm0.1433$
& $-15.0419\pm1.7772$\\
& State EDMD, degree 2
& $0.9984\pm0.0015$
& $0.9975\pm0.0019$
& $0.9546\pm0.0134$\\
& State EDMD, degree 3
& $0.9990\pm0.0017$
& $0.9961\pm0.0009$
& $0.9774\pm0.0114$\\
\midrule
Van der Pol
& Coordinate NLMS
& $0.9963\pm0.0002$
& $0.9020\pm0.0024$
& $0.7724\pm0.0044$\\
& State DMD
& $0.9574\pm0.0015$
& $0.0909\pm0.0037$
& $0.2051\pm0.0019$\\
& State EDMD, degree 2
& $0.9578\pm0.0014$
& $0.0777\pm0.0054$
& $0.2061\pm0.0027$\\
& State EDMD, degree 3
& $1.0000\pm0.0000$
& $0.5339\pm0.0176$
& $0.4075\pm0.0106$\\
\bottomrule
\end{tabular}
\end{table}

There is no uniformly best forecasting method in these comparisons. With a common Duffing output target at 200 steps, degree-three EDMD achieves mean $R^2=0.9774$, against $-0.3155$ for direct coordinate prediction. For the selected Van der Pol target at the same horizon, direct coordinate prediction instead achieves $0.7724$, against $0.4075$ for degree-three EDMD. These are within-system comparisons on a common target. The three trajectories measure evaluation variability for one fitted model per method, and the displayed $1.0000\pm0.0000$ is a rounded summary rather than a claim of exact prediction.

\subsection{Prediction under fixed randomized behavior policies}
\label{sec:exp_policy}
CartPole-v1, MountainCar-v0, and Acrobot-v1 test coordinate prediction under fixed behavior policies \citep{towers2024gymnasium}. The behavior policies are fixed rather than learned in this study. At each step, the prescribed action is replaced by a uniformly sampled action with probability $0.05$ for CartPole and MountainCar or $0.10$ for Acrobot. The predictor sees the current observation but not the randomly selected action. It must therefore predict a noisy realized successor rather than an action-conditioned outcome. This remains true when the environment is deterministic conditional on state and action. The study evaluates prediction under these policies; it does not evaluate policy learning or control performance.

Table~\ref{tab:exp_policy} gives the reported settings and horizon scores, where a horizon counts environment transitions within one episode. The CartPole and MountainCar runner specifies fixed configurations $(C,\omega)=(20,1)$ and $(100,1)$, respectively. The released sensitivity grid does not establish how these settings were chosen; Appendix~\ref{app:protocol} distinguishes that grid from the reported runs. Acrobot uses a search with seed $0$ that minimizes $E_{50}$, followed by runs at the selected setting. All three tasks report fits with base training seeds $0,1,2$. Adjacent base seeds generate overlapping episode-reset ranges, so these fits are dependent. The study did not reserve an untouched final test set. These results are exploratory diagnostics under the stated policies, with an additional selection dependence for Acrobot. Appendix~\ref{app:protocol} gives the episode and reset-seed conventions.

\begin{table}[tbp]
\centering
\caption{Exploratory coordinate-prediction diagnostics at the reported
configurations under fixed randomized behavior policies. Entries are mean
$\pm$ standard deviation over three fitted runs with base seeds
$0,1,2$; Acrobot standard deviations use divisor $3$.
The $R^2$ scores compare prediction error with always predicting the mean evaluation-target vector. Overlapping episode reset seeds introduce dependence
across the base-seed runs.}
\label{tab:exp_policy}
\small
\setlength{\tabcolsep}{4pt}
\begin{tabular}{lcrrrr}
\toprule
Task & $(C,\omega)$ & $E_1$ & $R_1^2$ & $E_{50}$ & $R_{50}^2$\\
\midrule
CartPole & $(20,1)$ & $0.0716\pm0.0061$ & $-0.133\pm0.189$ & $0.0768\pm0.0115$ & $-0.235\pm0.146$\\
MountainCar & $(100,1)$ & $0.00386\pm0.00058$ & $0.937\pm0.006$ & $0.1305\pm0.0047$ & $-1.111\pm0.079$\\
Acrobot & $(150,0.5)$ & $0.0096\pm0.0014$ & $0.7566\pm0.0307$ & $0.0577\pm0.0019$ & $-0.3860\pm0.2369$\\
\bottomrule
\end{tabular}
\end{table}

MountainCar and Acrobot have positive mean one-step $R^2$, while all three tasks have negative mean $R_{50}^2$. A negative $R_{50}^2$ means the model incurs \emph{more} squared error than always predicting the evaluation-set mean coordinate vector. Reporting $E_h$ alongside this centered score shows why a small magnitude-normalized error need not imply good prediction of changes in the representation. These outcomes depend on the specified representation, fixed policy, predictor, and horizon.

Acrobot also compares predictors against one common KAHM target. Table~\ref{tab:exp_acrobot} reports direct coordinate prediction and affine, quadratic, and cubic regressions in observation space. The polynomial predictors recompute features from each predicted observation before taking the next step, whereas the oscillator EDMD models propagate a feature vector by matrix powers. Each predicted Acrobot observation coordinate is clipped to $[-5,5]$ at every step. This is a numerical stabilization rule: it does not ensure physically valid observations, including the sine-cosine constraints. The reported results apply to predictors with this clipping rule.

\begin{table}[tbp]
\centering\small
\caption{Acrobot predictors evaluated against the same KAHM target, $C=150$ and $\omega=0.5$. Entries are means $\pm$ standard deviations over three fitted runs; the stored standard deviations use divisor $3$. Polynomial state predictions use the specified per-step clipping rule.}
\label{tab:exp_acrobot}
\begin{tabular}{@{}lcccc@{}}
\toprule
Predictor & $E_1$ & $R_1^2$ & $E_{50}$ & $R_{50}^2$ \\
\midrule
Direct coordinates, NLMS & $0.0096\pm0.0014$ & $0.7566\pm0.0307$ & $0.0577\pm0.0019$ & $-0.3860\pm0.2369$ \\
Affine state predictor & $0.0096\pm0.0003$ & $0.7535\pm0.0325$ & $0.0663\pm0.0150$ & $-0.5856\pm0.3776$ \\
Quadratic state predictor & $0.0059\pm0.0001$ & $0.8475\pm0.0208$ & $0.0806\pm0.0163$ & $-0.8703\pm0.0928$ \\
Cubic state predictor & $0.0030\pm0.0003$ & $0.9247\pm0.0076$ & $0.0790\pm0.0083$ & $-0.9422\pm0.5819$ \\
\bottomrule
\end{tabular}
\end{table}
Projection onto the simplex does not change the reported direct-predictor scores at the displayed precision. Although quadratic and cubic state predictors have smaller one-step errors, direct coordinate prediction has the smallest mean 50-step relative squared error in this comparison. All four predictors nevertheless have negative mean $R_{50}^2$, indicating worse 50-step performance than the constant evaluation-mean reference.

\paragraph{Visualization of dominant coordinate groups.}
Figure~\ref{fig:exp_regimes} shows one MountainCar fit by coloring visited states with the index of their largest coordinate. These indices are nominal labels and can differ from the hard reference classes used in the certificate. The figure describes how the coordinates group visited states.
\begin{figure}[htbp]
\centering
\includegraphics[width=0.77\linewidth]{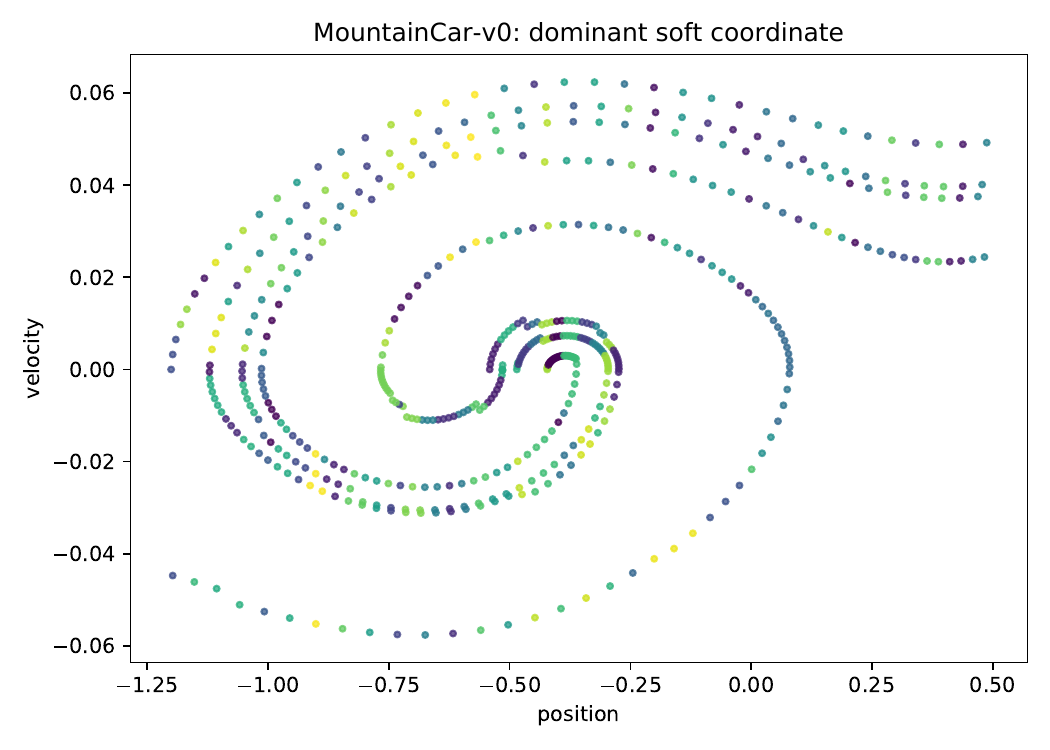}
\caption{A descriptive view of the MountainCar abstraction at $C=100$, $\omega=1$, using training seed $0$ and evaluation seed $2000$. Colors identify the largest soft coordinate along visited position-velocity states. The 100 coordinate indices are nominal labels, so color differences carry no quantitative meaning. The colors describe coordinate assignments on visited trajectories and do not measure class separation or predictive sufficiency.}
\Description{MountainCar position and velocity samples colored by the index of their largest soft coordinate.}
\label{fig:exp_regimes}
\end{figure}
\subsection{What the experiments establish about predictive abstractions}
\label{sec:exp_scope}
These studies answer a forecasting question, not a certificate question. Low error must be interpreted alongside coordinate variation, and a fair comparison of prediction algorithms requires the same target coordinates. On the systems and horizons tested, performance depends strongly on both choices. By contrast, the certificate experiments evaluate lower bounds on the best population error over a constrained class. Neither set of numerical results establishes exact Koopman closure. Detailed data roles and provenance, including immutable repository revisions, are documented in Appendices~\ref{app:protocol} and~\ref{app:certificate_provenance}.
\FloatBarrier

% ===== END sections/07_selected_experiments.tex =====

% ===== BEGIN sections/08_discussion.tex =====
\section{Discussion}
\label{sec:discussion}
The main lesson is that a representation can agree with the reference classes used to build it and yet retain too little information for accurate forecasting. Our certificates quantify prediction limits for a fixed representation and evaluation law within a specified bounded-norm linear predictor class. Reconstruction-score margins explain label agreement; successor and input variation determine the structural prediction bound; and Koopman theory describes the more restrictive case of exact deterministic closure. These questions are related, but none of the three properties alone implies the others.

\subsection{What the lower confidence bound establishes}
\label{sec:discussion_abstraction}
The certificate evaluates one representation for a specified prediction task. Its inputs are a fixed coordinate map, a reference partition, a joint current-successor distribution, a spectral-norm limit, and independent evaluation data. Its output is a lower confidence bound on the best population coordinate-prediction error in the permitted matrix class. If the bound exceeds the chosen tolerance, changing the matrix-fitting algorithm alone cannot meet that tolerance while the representation, distribution, and norm limit remain fixed.

Class-conditioned successor variation characterizes the best label-only predictor; soft inputs can encode additional information within a class, subject to the matrix-norm limit. The class-centered bound uses measured variation in current coordinates within reference classes, whereas the label-based bound substitutes their deviation from one-hot reference labels. These results do not establish which alternative representation, nonlinear predictor, or control policy would succeed.

\subsection{Using a certificate to assess a representation}
\label{sec:discussion_practical}
A practitioner can use the bound as a diagnostic rather than a new training objective. First fix the reference partition, coordinate map, evaluation law, target error tolerance, and allowable matrix norm without using the evaluation observations. Next compute the class-conditional statistics on independent transitions and report the certificate with its confidence level. If the bound exceeds the tolerance, retraining a matrix within the same norm-bounded class cannot achieve the desired population RMSE; the plausible next steps include changing the representation, enriching the predictor class, or reconsidering the permitted sensitivity. If the bound is zero or below tolerance, the test is inconclusive: an accurate predictor may or may not exist. In either case, coordinate-variation summaries help detect trivial nearly constant representations that might otherwise appear easy to forecast.

The bound does not identify which improved representation will work, and changing the norm limit changes the question being certified. Comparing several candidate representations using the same evaluation pairs also requires a confidence guarantee covering that comparison, or an independent final evaluation sample. These restrictions are part of the practical meaning of the result rather than technical details that can be ignored after calculating the bound.

\subsection{Sampling uncertainty and the gap in the population bound}
\label{sec:discussion_design}
The four-state example shows why validity and informativeness are different. For the original label-based certificate with $\delta=0.05$, numerical verification found no bound above the known optimum across the 117,000 four-state evaluation datasets, but 49,271 bounds were zero. More data reduce confidence corrections without necessarily making a structural bound equal to the optimum. At $p=31/32$, the label-based population bound is $13/32$ against the exact optimum $15/32$; the centered population bound equals $15/32$ because within-class input variation vanishes. At $p=4/5$, the label-based population bound is zero although the optimum is $0.30$, while the centered population limit again equals the optimum. Finite-sample tightness must therefore be assessed separately from these population identities.

The Van der Pol representation with $\omega=128$ was selected using earlier exploratory results, before its independent certificate evaluations. Confirmatory evaluation yielded positive certificates at the fitted matrix norm in all 32 datasets under each law, but substantially fewer at larger spectral-norm limits. Positivity is therefore specific to the norm-constrained class.

The replicate datasets describe repeated evaluation of one fixed map, not variability across fitted representations. Adaptive representation selection requires fresh evaluation or appropriate simultaneous coverage.

\subsection{Interpreting the forecasting experiments}
\label{sec:discussion_prediction}
The exponent change used for the Van der Pol certificate also changes its input and successor coordinates and may activate the uniform-coordinate fallback. Its prediction task is therefore distinct from the forecasting task with a different exponent. Coordinate-variation statistics identify nearly constant targets, while common-target comparisons isolate predictor performance conditional on one fixed output map. In these studies, polynomial EDMD attains higher long-horizon scores on Duffing, whereas direct coordinate prediction attains higher scores for the selected Van der Pol target. The fixed-policy tasks have negative mean 50-step $R^2$ at all reported configurations. These findings do not establish uniform superiority or control benefits. Moreover, the fitted Van der Pol RMSE does not determine the gap between the certificate and the unknown population optimum.

\subsection{Exact closure and spectral interpretation}
\label{sec:discussion_spectral}
Under exact deterministic closure, the scalar functions generated by the coordinates form a finite-dimensional invariant space for the Koopman operator. Removing coefficient vectors that represent zero gives the correct reduced matrix representation and nonzero eigenfunctions. The adjoint is determined by the specified RKHS inner product. These operator statements apply when closure holds on the full stated domain.

The reported numerical spectral quantities instead compare weighted sums of features on evaluation data. Small values may arise from cancellation or from the selected weights. Under the stochastic protocols, they remain descriptive sample diagnostics unless additional assumptions justify a population operator interpretation. Appendix~\ref{app:supporting_diagnostics} states their definitions and limitations. The invariance diagnostics and forecast bounds of \citet{conradie2026trustworthy} provide useful directions for a more direct comparison in future work.

\paragraph{From one-step residuals to long-horizon error.}
For deterministic $F$, define the fitted residual
$r(x)=\Phi(F(x))-\widehat B^\top\Phi(x)$.
Repeated substitution gives the exact identity
\begin{equation}
\Phi(F^h(x))-(\widehat B^\top)^h\Phi(x)
=\sum_{j=0}^{h-1}(\widehat B^\top)^{h-1-j}r(F^j(x)).
\label{eq:multistep_residual_identity}
\end{equation}
Thus long-horizon error depends on the matrix powers and the residuals at later states reached by the dynamics. A small average one-step residual under the initial evaluation law need not control residuals under those later state distributions. If $F$ preserves $\mathbb P_x$, the $L^2$ triangle inequality gives
\[
\left(\mathbb E_{x\sim\mathbb P_x}\!\left[
\|\Phi(F^h(x))-(\widehat B^\top)^h\Phi(x)\|^2\right]\right)^{1/2}
\leq R_{\mathbb P_x}(\Phi,\widehat B)^{1/2}
\sum_{j=0}^{h-1}\|\widehat B\|_2^j.
\]
This bound makes the required invariance assumption explicit. It has not been established for the evaluation laws in the reported forecasting studies.

\subsection{Limitations and further evaluation}
\label{sec:discussion_scope}
\textbf{Statistical scope.}
The confidence guarantees are conditional on fixed representation and class maps and require the specified independent evaluation units. Data-dependent representation selection, optional stopping, and distribution shift require new evaluation samples or appropriately extended guarantees. Under stochastic transitions, the lower bound concerns realized-successor RMSE, including transition noise; it does not by itself rule out exact linear evolution of the conditional mean.

\textbf{Comparative scope.}
The four-state study repeats evaluations of a known low-dimensional representation, whereas the Van der Pol result concerns one representation selected using exploratory data. In the analytical design, prespecified exact class prototypes benefit the prototype certificate; this advantage may not hold with learned prototypes, and the ERM-based baseline can be tighter for larger samples. Comparing the additional certificates on the 64 Van der Pol datasets requires coordinate arrays or sufficient statistics not included in the compact public release. Thus the four-state ordering cannot be generalized to learned high-dimensional maps.

\textbf{Experimental scope.}
Several forecasting summaries contain only three fitted models or evaluation trajectories, with some dependent control-seed ranges. The prevalence and impact of the numerical uniform-coordinate fallback have not been measured; a comparison with stable exact normalization would constitute evaluation of a changed representation. The bounds derived from KAHM reconstruction-score margins also require their specified score-domain and normalization conditions or the correction in~\eqref{eq:implemented_geometric_evaluation_loss}. Appendix~\ref{app:certificate_provenance} documents available data and verification records.

\section{Conclusion}
\label{sec:conclusion}
A soft representation can reproduce its reference labels and nevertheless discard distinctions needed for forecasting. We show how to test this possibility without relying on the performance of a particular trained predictor. For fixed soft coordinates, the population bound based on class-conditional centering compares successor variation within reference classes with the distinctions retained by the current coordinates. Independent evaluation transitions provide lower confidence bounds on the minimum population RMSE of all linear predictors with a specified spectral-norm limit. A bound exceeding a target tolerance rules out that accuracy for the entire constrained class; a zero bound leaves the question open.

For coordinates derived from Kernel Affine Hull Machine reconstruction scores, margin inequalities link better within-class reconstruction to stronger reference-label agreement. The analytically exact four-state example demonstrates why that agreement does not guarantee accurate prediction. A reanalysis of 117,000 archived four-state evaluation datasets checks the structural certificates against a known optimum and an ERM-based comparator. Independent Van der Pol evaluations yield positive label-based certificates at the fitted norm for one previously selected map; the optimum in that setting remains unknown. Further forecasting comparisons show why coordinate variation, a common prediction target, and horizon matter when judging learned representations.

If the coordinates evolve by one linear matrix at every state under deterministic dynamics, their nonredundant function span has the stated Koopman-invariant and reproducing-kernel interpretation. This conditional result is distinct from certifying a positive prediction error from finite data. Taken together, the findings offer a practical distinction: \emph{representing reference classes well is not the same as retaining enough information to predict their future}, and the latter can sometimes be ruled out for an entire predictor class before further fitting effort is spent.

% ===== END sections/08_discussion.tex =====

\clearpage
\section*{Statements and Declarations}
\addcontentsline{toc}{section}{Statements and Declarations}
\paragraph{Funding.}
The research reported in this paper has been supported by the Austrian Ministry for Transport, Innovation and Technology, the Federal Ministry for Digital and Economic Affairs, and the State of Upper Austria through the SCCH competence center INTEGRATE (FFG grant no. 892418), part of the FFG COMET Competence Centers for Excellent Technologies Programme.

\paragraph{Competing interests.}
The authors declare that they have no competing interests.

\paragraph{Author contributions.}
Mohit Kumar developed the mathematical theory. Somayeh Kargaran implemented the experiments and maintained the provenance records for the experimental results.

\paragraph{Data availability.}
The publicly released research data and verification records are archived at the \href{https://github.com/MohitKumarRostock/Kernel_Affine_Hull_Koopman_Machine/tree/0bbb2bafa107892e32e221e882326890b14fe36e}{experiment-evidence revision of the project repository}. The directory \path{reproduction/raw_data/canonical/} contains seven compressed NumPy archives with 553 logical input datasets for the forecasting studies, together with metadata, array-content hashes, and file checksums. The four-state evidence archive contains the counts for all 117,000 replicates and their verification records. The four-method comparison code, 117-cell summary CSVs, and compressed replicate-level results (117,000 rows) are archived in a separate immutable repository revision, with SHA-256 checksums for the uncompressed files. The fixed Van der Pol representation and its training and verification records are also distributed. Appendix~\ref{app:certificate_provenance} gives the associated commit identifiers and data paths.

For the confirmatory Van der Pol certificate study, the repository provides the frozen protocol and configuration, 64 replicate-level result summaries, aggregate tables, and independent numerical-verification records. The full evaluation arrays are retained separately from the compact public release; inquiries about access may be directed to the corresponding author. Author-generated data and non-code outputs are licensed under CC BY 4.0 unless otherwise indicated.

\paragraph{Code availability.}
Experimental source code, configurations, numerical verifiers, and forecasting code are archived at the immutable \href{https://github.com/MohitKumarRostock/Kernel_Affine_Hull_Koopman_Machine/tree/0bbb2bafa107892e32e221e882326890b14fe36e}{experiment-evidence revision \texttt{0bbb2bafa107}}. The four-state minimum-risk comparison, its source code, frozen-input reproduction procedure, tests, and derived outputs are archived at the \href{https://github.com/MohitKumarRostock/Kernel_Affine_Hull_Koopman_Machine/tree/78a1a5e7ea9d3a2fecd3b335bbe5a9f3f2237eb2/reproduction/posthoc_four_state_erm}{post-hoc analysis revision \texttt{78a1a5e7ea9d}}. Reproduction commands and expected output hashes for the four-state comparison are provided in \path{reproduction/posthoc_four_state_erm/README_REPRODUCE.md}. The repository also contains general reproduction instructions (\path{REPRODUCIBILITY.md}), dependency specifications (\path{requirements-lock-arm64.txt}), and reference-environment records (\path{provenance/reference_arm64_environment/}). The source code is distributed under the MIT License.

\paragraph{Ethics approval and consent.}
The research described in this manuscript consists of mathematical analysis and computer simulations. It reports no study involving human participants or animals; ethics approval and participant consent are therefore not applicable to the work described.

\paragraph{Use of AI assistance.}
ChatGPT (OpenAI) assisted with literature searches, language revision, and experimental-reporting and provenance descriptions. 

\appendix

% ===== BEGIN appendices/technical_details.tex =====
\section{KAHM Construction Details}\label{app:kahm_construction}
This appendix describes the KAHM construction used in~\eqref{eq_220420251843}, following prior work~\citep{kumar2025operator,kumar2026semantic}. The construction has four components: an encoding matrix, a Gaussian kernel, regularized regression functions, and normalized affine reconstruction weights. Throughout this appendix, the reference matrix $ \mathbf X = \left[
 x^1\ \cdots\ x^N
 \right]^\top
 \in\mathbb R^{N\times n}$ is assumed to contain $N\geq2$ pairwise distinct samples. The encoding returned by Algorithm~\ref{algorithm_encoding_matrix} must also map distinct reference samples to distinct vectors, i.e., $
 \mathbf P_{\mathbf X}x^i
 \neq
 \mathbf P_{\mathbf X}x^j,
 \; i\neq j$.
The encoding must also retain at least one direction, and the encoded reference points must have positive-definite sample covariance.
\begin{itemize}
\item $\mathbf{P}_{\mathbf{X}}
\in\mathbb{R}^{\underline{n}\times n}$,
with
$\underline{n}\in\{1,\cdots,n\}$,
is an encoding matrix such that
$\mathbf{P}_{\mathbf{X}}x$ is an
$\underline n$-dimensional encoding of $x$.
The encoding matrix is computed from the reference samples
$\mathbf X$ using Algorithm~\ref{algorithm_encoding_matrix}. The algorithm orders encoding directions by decreasing sample-covariance eigenvalue. It then checks the sample range of each encoded coordinate. If any range is below $10^{-3}$ and more than one direction remains, it removes the last direction and checks again. Every retained eigenvalue exceeds double-precision machine epsilon $\epsilon_{\rm mach}$. 
\begin{algorithm}
\caption{Determination of Encoding Matrix $\mathbf{P}_{\mathbf{X}}$}
\begin{algorithmic}[1]
\Require $\mathbf X\in\mathbb R^{N\times n}$, $N\geq2$.
\State Let $r_{\mathbf X}$ count sample-covariance eigenvalues $>\epsilon_{\rm mach}$.
\If{$r_{\mathbf X}=0$}
\State \Return inadmissible.
\EndIf
\State $\underline n\gets\min(20,n,N-1,r_{\mathbf X})$.
\State Form the rows of $\mathbf P_{\mathbf X}$ from the transposed eigenvectors for the largest $\underline n$ covariance eigenvalues.
\While{$\underline n>1$ and
$\displaystyle\min_j\bigl[\max_i(\mathbf P_{\mathbf X}x^i)_j-\min_i(\mathbf P_{\mathbf X}x^i)_j\bigr]<10^{-3}$}
\State Remove $\mathbf P_{\mathbf X}$'s last row; $\underline n\gets\underline n-1$.
\EndWhile
\State \Return $\mathbf P_{\mathbf X}$.
\end{algorithmic}
\label{algorithm_encoding_matrix}
\end{algorithm}

For a matrix returned by
Algorithm~\ref{algorithm_encoding_matrix}, the KAHM construction proceeds only when
$\{\mathbf P_{\mathbf X}x^i\}_{i=1}^N$ are pairwise distinct and their sample covariance is positive definite. These are the admissibility requirements imposed on class-specific reference datasets in Definition~\ref{definition_200720261153}.
\item Define the encoded space by
\begin{IEEEeqnarray}{rCl}
 \underline{\mathcal X}
 &:=&
 \{\mathbf P_{\mathbf X}x\mid x\in\mathbb R^n\},
\end{IEEEeqnarray}
and equip it with the positive-definite real-valued kernel
$k_{\mathbf X}:
\underline{\mathcal X}\times\underline{\mathcal X}\rightarrow
\mathbb R$ and its corresponding reproducing kernel Hilbert space
$\mathcal H_{k_{\mathbf X}}(\underline{\mathcal X})$, where
\begin{IEEEeqnarray}{rCl}
\label{eq_090420261651}
 k_{\mathbf X}(\underline x^i,\underline x^j)
 &:=&
 \exp\left(
 -\frac{1}{2\underline n}
 (\underline x^i-\underline x^j)^\top
 \Sigma_{\mathbf X}^{-1}
 (\underline x^i-\underline x^j)
 \right).
\end{IEEEeqnarray}
Here
$\underline x^i,\underline x^j\in\underline{\mathcal X}$ and
$\Sigma_{\mathbf X}\succ0$ is the sample covariance matrix of
\[
 \{\mathbf P_{\mathbf X}x^1,\cdots,
   \mathbf P_{\mathbf X}x^N\}.
\]

Both the encoding covariance and the Gaussian-kernel covariance use denominator $N-1$ in the construction and implementation:
\[
\bar u:=\frac1N\sum_{i=1}^N\mathbf P_{\mathbf X}x^i,
\qquad
\Sigma_{\mathbf X}:=\frac1{N-1}\sum_{i=1}^N
(\mathbf P_{\mathbf X}x^i-\bar u)(\mathbf P_{\mathbf X}x^i-\bar u)^\top.
\]
This convention gives covariance $2$ for the two encoded points $-1$ and $1$ used in Section~\ref{ex:kahm_conflict}. Reference values are stored in single precision; covariance estimation and matrix calculations use double precision.

\item For each $i\in\{1,\cdots,N\}$, the function
$h_{\mathbf X}^i:
\underline{\mathcal X}\rightarrow\mathbb R$, with
$h_{\mathbf X}^i\in
\mathcal H_{k_{\mathbf X}}(\underline{\mathcal X})$,
fits target $1$ at encoded reference point $i$ and $0$ at the other reference points by kernel-regularized least squares:
\begin{IEEEeqnarray}{rCl}
 h_{\mathbf X}^i
 &=&
 \arg\min_{f\in
 \mathcal H_{k_{\mathbf X}}(\underline{\mathcal X})}
 \left(
 \sum_{j=1}^N
 \left|
 \mathbbm{1}_{\{\mathbf P_{\mathbf X}x^i\}}
 (\mathbf P_{\mathbf X}x^j)
 -
 f(\mathbf P_{\mathbf X}x^j)
 \right|^2
 +
 \lambda_{\mathbf X}^*
 \|f\|_{\mathcal H_{k_{\mathbf X}}
 (\underline{\mathcal X})}^2
 \right).
 \IEEEeqnarraynumspace
\end{IEEEeqnarray}
Because the encoded reference points are pairwise distinct,
\[
 \mathbbm{1}_{\{\mathbf P_{\mathbf X}x^i\}}
 (\mathbf P_{\mathbf X}x^j)
 =
 \delta_{ij}.
\]
Thus the target vector for $h_{\mathbf X}^i$ is the $i$th row of $\mathbf I_N$. The regularized solution balances agreement with these targets against the RKHS norm penalty. The regularization parameter
$\lambda_{\mathbf X}^*\in\mathbb R_+$ is given by
\begin{IEEEeqnarray}{rCl}
\label{eq_020920251811}
 \lambda_{\mathbf X}^*
 &=&
 \hat e+\frac{2}{nN}\|\mathbf X\|_F^2,
\end{IEEEeqnarray}
where $\hat e$ is the unique fixed point of $\mathcal R_{\mathbf X}$ in its first argument, with the second argument fixed as shown:
\begin{IEEEeqnarray}{rCl}
\label{eq_090120230831}
 \hat e
 &=&
 \mathcal R_{\mathbf X}
 \left(
 \hat e,\frac{2}{nN}\|\mathbf X\|_F^2
 \right),
\end{IEEEeqnarray}
with
$\mathcal R_{\mathbf X}:
\mathbb R_+\times\mathbb R_+\rightarrow\mathbb R_+$ defined by
\begin{IEEEeqnarray}{rCl}
 \mathcal R_{\mathbf X}(e,\kappa)
 &:=&
 \frac{1}{nN}
 \sum_{j=1}^n
 \left\|
 (\mathbf X)_{:,j}
 -
 \mathbf K_{\mathbf X}
 \left(
 \mathbf K_{\mathbf X}
 +(e+\kappa)\mathbf I_N
 \right)^{-1}
 (\mathbf X)_{:,j}
 \right\|^2.
\end{IEEEeqnarray}
Here $\mathbf K_{\mathbf X}$ is the
$N\times N$ kernel matrix whose $(i,j)$th entry is
\begin{IEEEeqnarray}{rCl}
\label{eq_020920251810}
 (\mathbf K_{\mathbf X})_{ij}
 &:=&
 k_{\mathbf X}
 (\mathbf P_{\mathbf X}x^i,
  \mathbf P_{\mathbf X}x^j).
\end{IEEEeqnarray}
The iterations
\begin{IEEEeqnarray}{rCl}
\label{eq_090420261740}
 e|_{it+1}
 &=&
 \mathcal R_{\mathbf X}
 \left(
 e|_{it},\frac{2}{nN}\|\mathbf X\|_F^2
 \right),
 \qquad it\in\{0,1,\cdots\},
 \\
\label{eq_090420261741}
 e|_0
 &\in&
 \left(
 0,\frac{1}{nN}\|\mathbf X\|_F^2
 \right)
\end{IEEEeqnarray}
converge to $\hat e$. Specifically, Theorem~1 in
\citep{kumar2024kahm} establishes that, for the regularization offset
$\tfrac{2}{nN}\|\mathbf X\|_F^2$, the fixed-point iterations
\eqref{eq_090420261740}-\eqref{eq_090420261741}
converge to the unique fixed point $\hat e$ of
$\mathcal R_{\mathbf X}
(\cdot,\tfrac{2}{nN}\|\mathbf X\|_F^2)$.
The numerical implementation approximates this limit. It starts from
$e_0=\|\mathbf X\|_F^2/(2nN)$, giving
$\lambda_0=5\|\mathbf X\|_F^2/(2nN)$, and stops when the absolute change
in $\lambda$ is at most $10^{-3}$ or after 100 iterations.
The reconstruction uses the inverse matrix computed in the final completed
iteration, before that iteration's final update of $\lambda$.
Using the identity $\mathbbm{1}_{\{\mathbf P_{\mathbf X}x^i\}}
(\mathbf P_{\mathbf X}x^j)=\delta_{ij}$,
the kernel-regularized least-squares solution is
\begin{IEEEeqnarray}{rCl}
\label{eq_010920251549}
 h_{\mathbf X}^i(\cdot)
 &=&
 (\mathbf I_N)_{i,:}
 \left(
 \mathbf K_{\mathbf X}
 +\lambda_{\mathbf X}^*\mathbf I_N
 \right)^{-1}
 \left[
 \begin{IEEEeqnarraybox*}[][c]{,c/c/c,}
 k_{\mathbf X}(\cdot,\mathbf P_{\mathbf X}x^1)
 &
 \cdots
 &
 k_{\mathbf X}(\cdot,\mathbf P_{\mathbf X}x^N)
 \end{IEEEeqnarraybox*}
 \right]^\top.
\end{IEEEeqnarray}
The regression output $h_{\mathbf X}^i(\mathbf P_{\mathbf X}x)$ corresponds to reference sample $i$. Dividing it by the sum of all outputs gives that sample's affine reconstruction weight. The matrix solve in~\eqref{eq_010920251549} depends only on reference data and is performed once during construction. For each new query, we compute its encoding and kernel values, then multiply by the stored regression coefficients. Reusing these coefficients avoids solving a new regression problem for every reconstruction.
\item For $x\in\Omega_{\mathbf X}$, the denominator in~\eqref{eq_220420251843} is nonzero, and
\begin{equation}
\label{eq:kahm_affine_normalization}
\sum_{i=1}^{N}
\frac{h_{\mathbf X}^{i}(\mathbf P_{\mathbf X}x)}
{\sum_{j=1}^{N}h_{\mathbf X}^{j}(\mathbf P_{\mathbf X}x)}=1.
\end{equation}
Thus $\mathcal A_{\mathbf X}[\Omega_{\mathbf X}]\subseteq\mathrm{aff}(\{x^1,\cdots,x^N\})$.
\end{itemize}
\section{Proof of Proposition~\ref{proposition_300620260944}}\label{app:adjoint_kernel_proof}
For every $g\in\mathcal H_{\mathcal K}(\mathcal X)$, the reproducing property and adjoint relation give
\[
\langle g,\mathcal K(\cdot,F(x))\rangle
=g(F(x))=(K_Fg)(x)
=\langle K_Fg,\mathcal K(\cdot,x)\rangle
=\langle g,K_F^*\mathcal K(\cdot,x)\rangle.
\]
Both functions belong to the RKHS. Since their inner products with every $g$ agree, their difference is orthogonal to the whole space and must be zero. This proves~\eqref{eq:adjoint_kernel_section_identity_background}.
\section{Proof of Proposition~\ref{proposition_070720261027}}\label{app:feature_rkhs_proof}
Write $w=w^1+\mathrm i w^2$. Because $\Phi(x)$ is real and lies in $V_\Phi$, the component of $w$ orthogonal to $V_\Phi^{\mathbb C}$ contributes zero to every function value:
\begin{equation}
w^\top\Phi(x)=P_\Phi^{\mathbb C}(w)^\top\Phi(x).
\label{eq_050720261638}
\end{equation}
Thus $f_w=f_{P_\Phi(w^1)}+\mathrm i f_{P_\Phi(w^2)}\in\mathcal H_\Phi^{\mathbb C}$.
The real-space identity $\langle f_u,f_v\rangle=u^\top v$ for $u,v\in V_\Phi$, together with the complex inner product in~\eqref{eq_050720262018}, gives
\[
\langle f_w,f_{w'}\rangle_{\mathcal H_\Phi^{\mathbb C}}
=P_\Phi^{\mathbb C}(w)^\top\overline{P_\Phi^{\mathbb C}(w')}.
\]
Taking $w'=w$ yields the norm formula and shows that the represented function is zero exactly when its projected coefficient is zero.
\section{Proof of Proposition~\ref{proposition_170720262159}}\label{app:projection_proof}
Since $Q_\Phi^\top Q_\Phi=\mathbf I_{r_\Phi}$, the vector $Q_\Phi Q_\Phi^\top w$ lies in $V_\Phi^{\mathbb C}$ and
$Q_\Phi^\top(w-Q_\Phi Q_\Phi^\top w)=0$.
The remainder is therefore orthogonal to $V_\Phi^{\mathbb C}$. Here $Q_\Phi$ is real, so its transpose is also its conjugate transpose.
Uniqueness of orthogonal projection gives~\eqref{eq_170720261001}.
\section{Proof of Proposition~\ref{prop_koopman_eigenfunction}}\label{app:spectral_correspondence_proof}
We first show that exact closure gives a well-defined action on represented functions. We then express the Koopman operator and its adjoint in the orthonormal basis $Q_\Phi$. Finally, the reproducing property turns the adjoint residual into the finite feature sums in the proposition.

Under the exact-closure assumption
\[
 \Phi(F(x))=B^\top\Phi(x),
 \qquad x\in\mathcal X,
\]
the space $V_\Phi$ is invariant under $B^\top$. Indeed,
$B^\top\Phi(x)=\Phi(F(x))\in V_\Phi$ for every $x\in\mathcal X$,
and hence, by linearity,
\[
 B^\top v\in V_\Phi,
 \qquad v\in V_\Phi.
\]
The same invariance holds for the complexification
$V_\Phi^{\mathbb C}$. It also implies that $B$ preserves the redundant coefficient directions: if $u\in V_\Phi^\perp$ and $v\in V_\Phi$, then
\[
\langle Bu,v\rangle=\langle u,B^\top v\rangle=0.
\]
Hence $Bu\in V_\Phi^\perp$. Coefficient vectors that represent the same function therefore continue to represent the same function after multiplication by $B$.

The exact-closure relation implies that composition with $F$ maps $\mathcal H_\Phi^{\mathbb C}$ into itself. Since
$\mathcal H_\Phi^{\mathbb C}$ is finite dimensional, the restricted
complexified Koopman operator is bounded and therefore admits an
adjoint. For any $w\in\mathbb C^C$, the definition of $f_{w}$ and
exact closure give
\begin{align}
 \left(K_{F,\mathbb C}f_{w}\right)(x)
 &=
 f_{w}(F(x)) \nonumber\\
 &=
 w^\top\Phi(F(x)) \nonumber\\
 &=
 w^\top B^\top\Phi(x) \nonumber\\
 &=
 (Bw)^\top\Phi(x) \nonumber\\
 &=
 f_{Bw}(x).
 \label{eq_060720261927}
\end{align}
Thus
\[
 K_{F,\mathbb C}f_{w}
 =
 f_{Bw},
\]
which proves~\eqref{eq_190720260749}. Subtracting $\lambda f_w$ gives
\[
K_{F,\mathbb C}f_w-\lambda f_w=f_{Bw-\lambda w}.
\]
By Proposition~\ref{proposition_070720261027}, this residual is zero if and only if $P_\Phi^{\mathbb C}(Bw-\lambda w)=0$. The same proposition gives $f_w\neq0$ if and only if $P_\Phi^{\mathbb C}(w)\neq0$. Together these prove the necessary and sufficient coefficient conditions in item~\ref{prop_koopman_eigenfunction_1}. In particular, if $Bw=\lambda w$ and $f_w\neq0$, then
\[
 K_{F,\mathbb C}f_{w}
 =
 f_{\lambda w}
 =
 \lambda f_{w}.
\]
Hence $f_w$ is a Koopman eigenfunction with eigenvalue $\lambda$.

For a reduced coefficient $\widetilde w$, projection onto the unique representative gives
\[
P_\Phi^{\mathbb C}(BQ_\Phi\widetilde w)
=Q_\Phi Q_\Phi^\top BQ_\Phi\widetilde w
=Q_\Phi\widetilde B_\Phi\widetilde w.
\]
Consequently,
\[
K_{F,\mathbb C}f_{Q_\Phi\widetilde w}
=f_{Q_\Phi\widetilde B_\Phi\widetilde w}.
\]
The map $\widetilde w\mapsto f_{Q_\Phi\widetilde w}$ is an isometry, because its squared RKHS norm is $\|Q_\Phi\widetilde w\|^2=\|\widetilde w\|^2$. It is therefore injective. The reduced Koopman eigenrelation holds exactly when $\widetilde w\neq0$ and $\widetilde B_\Phi\widetilde w=\lambda\widetilde w$.

We next express the adjoint in the orthonormal coefficient basis $Q_\Phi$. Let $\widetilde w,\widetilde w'\in\mathbb C^{r_\Phi}$. Since $Q_\Phi\widetilde w\in V_\Phi^{\mathbb C}$,
Proposition~\ref{proposition_070720261027} and
$P_\Phi^{\mathbb C}=Q_\Phi Q_\Phi^\top$ give
\begin{align*}
 \left\langle
 K_{F,\mathbb C}f_{Q_\Phi\widetilde w},
 f_{Q_\Phi\widetilde w'}
 \right\rangle_{\mathcal H_\Phi^{\mathbb C}}
 &=
 \left\langle
 f_{BQ_\Phi\widetilde w},
 f_{Q_\Phi\widetilde w'}
 \right\rangle_{\mathcal H_\Phi^{\mathbb C}}\\
 &=
 \left(
 P_\Phi^{\mathbb C}
 (BQ_\Phi\widetilde w)
 \right)^\top
 \overline{Q_\Phi\widetilde w'}.
\end{align*}
Using
$P_\Phi^{\mathbb C}=Q_\Phi Q_\Phi^\top$ and
$\widetilde B_\Phi=Q_\Phi^\top BQ_\Phi$, we obtain
\[
 P_\Phi^{\mathbb C}
 (BQ_\Phi\widetilde w)
 =
 Q_\Phi Q_\Phi^\top BQ_\Phi\widetilde w
 =
 Q_\Phi\widetilde B_\Phi\widetilde w.
\]
Therefore,
\begin{align*}
 \left\langle
 K_{F,\mathbb C}f_{Q_\Phi\widetilde w},
 f_{Q_\Phi\widetilde w'}
 \right\rangle_{\mathcal H_\Phi^{\mathbb C}}
 &=
 \left(
 Q_\Phi\widetilde B_\Phi\widetilde w
 \right)^\top
 \overline{Q_\Phi\widetilde w'}\\
 &=
 \widetilde w^\top
 \widetilde B_\Phi^\top
 \overline{\widetilde w'}\\
 &=
 \widetilde w^\top
 \overline{
 \widetilde B_\Phi^\top\widetilde w'
 }\\
 &=
 \left\langle
 f_{Q_\Phi\widetilde w},
 f_{Q_\Phi\widetilde B_\Phi^\top
 \widetilde w'}
 \right\rangle_{\mathcal H_\Phi^{\mathbb C}},
\end{align*}
where the third equality uses the fact that
$\widetilde B_\Phi$ is real. By the defining property of the adjoint,
\[
 K_{F,\mathbb C}^*
 f_{Q_\Phi\widetilde w'}
 =
 f_{Q_\Phi\widetilde B_\Phi^\top
 \widetilde w'}.
\]
This proves~\eqref{eq_190720260756}. Now suppose
\[
 \widetilde B_\Phi^\top\widetilde w
 =
 \lambda\widetilde w,
 \qquad
 \widetilde w\neq0.
\]
Then~\eqref{eq_190720260756} gives
\[
 K_{F,\mathbb C}^*
 f_{Q_\Phi\widetilde w}
 =
 f_{Q_\Phi\lambda\widetilde w}
 =
 \lambda
 f_{Q_\Phi\widetilde w}.
\]
To establish an eigenfunction, we must also show that the coefficient vector represents a nonzero function.
Because the columns of $Q_\Phi$ are orthonormal,
\[
 \|Q_\Phi\widetilde w\|
 =
 \|\widetilde w\|>0.
\]
Moreover,
$Q_\Phi\widetilde w\in V_\Phi^{\mathbb C}$, so
Proposition~\ref{proposition_070720261027} yields
\[
 \left\|
 f_{Q_\Phi\widetilde w}
 \right\|_{\mathcal H_\Phi^{\mathbb C}}^2
 =
 \|Q_\Phi\widetilde w\|^2
 =
 \|\widetilde w\|^2
 >0.
\]
Hence
$f_{Q_\Phi\widetilde w}\neq0$, and it is an adjoint
Koopman eigenfunction with eigenvalue $\lambda$. Conversely, if this nonzero function satisfies the adjoint eigenrelation, then~\eqref{eq_190720260756} and injectivity of $\widetilde w\mapsto f_{Q_\Phi\widetilde w}$ imply $\widetilde B_\Phi^\top\widetilde w=\lambda\widetilde w$. This proves both directions of item~\ref{prop_koopman_eigenfunction_2}.

Finally, let
\[
 \mathrm E_{\mathrm{spec}}
 =
 \{(x'^i,F(x'^i))\}_{i=1}^M
\]
and $v=(v_1,\cdots,v_M)^\top\in\mathbb C^M$. By the definition of
the induced kernel,
\begin{IEEEeqnarray}{rCCCl}
\label{eq_020720262122}
 g_v
 & = &
 \sum_{i=1}^M
 v_i\mathcal K(\cdot,x'^i)
 & = &
 \left(\Phi(\cdot)\right)^\top
 \sum_{i=1}^M v_i\Phi(x'^i).
\end{IEEEeqnarray}
Using linearity of $K_{F,\mathbb C}^*$ and the kernel-section identity
\eqref{eq_020720261144},
\begin{IEEEeqnarray}{rCCCl}
\label{eq_020720262123}
 K_{F,\mathbb C}^*g_v
 & = &
 \sum_{i=1}^M
 v_i\mathcal K(\cdot,F(x'^i))
 & = &
 \left(\Phi(\cdot)\right)^\top
 \sum_{i=1}^M v_i\Phi(F(x'^i)).
\end{IEEEeqnarray}
Hence
\[
 K_{F,\mathbb C}^*g_v-\lambda g_v
 =
 f_{q},
\]
where
\[
 q
 :=
 \sum_{i=1}^M v_i\Phi(F(x'^i))
 -
 \lambda\sum_{i=1}^M v_i\Phi(x'^i).
\]
Each feature vector appearing in $q$ belongs to $V_\Phi$, and thus
$q\in V_\Phi^{\mathbb C}$. Proposition~\ref{proposition_070720261027}
therefore gives
\[
 \left\|
 K_{F,\mathbb C}^*g_v-\lambda g_v
 \right\|_{\mathcal H_\Phi^{\mathbb C}}^2
 =
 \left\|
 \sum_{i=1}^M v_i\Phi(F(x'^i))
 -
 \lambda\sum_{i=1}^M v_i\Phi(x'^i)
 \right\|^2,
\]
which proves~\eqref{eq_210720260634}. Similarly,
\[
 \left\|
 K_{F,\mathbb C}^*g_v
 \right\|_{\mathcal H_\Phi^{\mathbb C}}^2
 =
 \left\|
 \sum_{i=1}^M v_i\Phi(F(x'^i))
 \right\|^2,
\]
which proves~\eqref{eq_210720260635}.

\section{Proof of Proposition~\ref{proposition_110820262057}}
\label{app:energy_identity}
We first compute the exact change in squared parameter distance. We then turn that identity into a quadratic inequality for the prediction residual and bound its coefficients uniformly over the simplex. Fix $x\in\mathcal X$ and abbreviate
\[
 v:=\Phi(x),\qquad
 s:=\|v\|^2,\qquad
 u:=\frac{\beta}{1+\beta s},\qquad
 Q:=B_*-\widehat B.
\]
Also write
\[
 \zeta_f:=\zeta_{\mathrm{fit}}(x),
 \qquad
 \zeta_p:=\zeta_{\mathrm{proto}}(x).
\]
By~\eqref{eq:nlms_residuals},
\[
 \zeta_f=Q^\top v+\zeta_p.
\]
The hypothetical update~\eqref{eq:hypothetical_update} therefore gives
\[
 B_*-\widehat B^+(x)
 =
 Q-uv\zeta_f^\top.
\]
\paragraph{Parameter-error identity.} Expanding the change in squared Frobenius norm yields
\begin{align}
 D_\beta(x)
 &=
 \|Q-uv\zeta_f^\top\|_F^2-\|Q\|_F^2
 \nonumber\\
 &=
 -2u\langle Q,v\zeta_f^\top\rangle_F
 +u^2\|v\zeta_f^\top\|_F^2
 \nonumber\\
 &=
 -2u\langle Q^\top v,\zeta_f\rangle
 +u^2s\|\zeta_f\|^2
 \nonumber\\
 &=
 -u(2-us)\|\zeta_f\|^2
 +2u\langle\zeta_p,\zeta_f\rangle.
 \label{eq:appendix_energy_expansion}
\end{align}
The third equality uses the rank-one identities
\[
 \langle Q,v\zeta_f^\top\rangle_F
 =
 \langle Q^\top v,\zeta_f\rangle,
 \qquad
 \|v\zeta_f^\top\|_F^2
 =
 \|v\|^2\|\zeta_f\|^2.
\]
For the final equality, substitute
\[
 Q^\top v=\zeta_f-\zeta_p.
\]
This proves~\eqref{eq:nlms_exact_energy}.
\paragraph{Residual bound.}
Define
\[
 a:=u(2-us),
 \qquad
 r:=\|\zeta_f\|.
\]
Since
\[
 us=\frac{\beta s}{1+\beta s}<1,
\]
we have $a>0$. Rearranging
\eqref{eq:appendix_energy_expansion} and applying the
Cauchy-Schwarz inequality gives
\begin{equation}
 ar^2
 \leq
 2u\|\zeta_p\|r+|D_\beta(x)|.
 \label{eq:energy_quadratic}
\end{equation}
Equivalently,
\[
 ar^2-2u\|\zeta_p\|r-|D_\beta(x)|\leq0.
\]
Because $a>0$ and $r\geq0$, $r$ is bounded by the nonnegative root of
the corresponding quadratic equation:
\begin{align}
 r
 &\leq
 \frac{
 u\|\zeta_p\|
 +
 \sqrt{
 u^2\|\zeta_p\|^2+a|D_\beta(x)|
 }
 }{a}
 \nonumber\\
 &\leq
 \frac{2u}{a}\|\zeta_p\|
 +
 \frac{\sqrt{|D_\beta(x)|}}{\sqrt a}
 \nonumber\\
 &=
 \frac{2}{2-us}\|\zeta_p\|
 +
 \frac{\sqrt{|D_\beta(x)|}}{\sqrt a}.
 \label{eq:energy_root_bound}
\end{align}
The second inequality uses $\sqrt{p+q}\leq\sqrt p+\sqrt q$ for $p,q\geq0$. Because $v=\Phi(x)\in\Delta_C$, its coordinates are nonnegative and sum to one. Hence
\[
 1
 =
 \left(\sum_{c=1}^C v_c\right)^2
 \leq
 C\sum_{c=1}^C v_c^2
 =
 Cs,
\]
by the Cauchy-Schwarz inequality, while
\[
 s=\sum_{c=1}^C v_c^2
 \leq
 \left(\sum_{c=1}^C v_c\right)^2
 =1.
\]
Thus
\[
 \frac1C\leq s\leq1.
\]
For the first coefficient in~\eqref{eq:energy_root_bound},
\begin{align}
 \frac{2}{2-us}
 &=
 \frac{2(1+\beta s)}{2+\beta s}
 \nonumber\\
 &=
 1+\frac{\beta s}{2+\beta s}
 \nonumber\\
 &\leq
 1+\frac{\beta}{2+\beta}
 =
 c_\beta.
 \label{eq:energy_factor_bound}
\end{align}
Here the inequality follows from $s\leq1$ and the fact that
\[
 q\mapsto\frac{q}{2+q}
\]
is increasing for $q\geq0$.

For the second coefficient,
\begin{equation}
 a
 =
 \frac{\beta(2+\beta s)}{(1+\beta s)^2}.
 \label{eq:energy_a_expression}
\end{equation}
Define
\[
 \varphi(q):=\frac{2+q}{(1+q)^2},
 \qquad q\geq0.
\]
Then
\[
 \varphi'(q)
 =
 -\frac{3+q}{(1+q)^3}<0,
\]
so $\varphi$ is decreasing on $[0,\infty)$. Since
$0\leq\beta s\leq\beta$, it follows from
\eqref{eq:energy_a_expression} that
\begin{align}
 a
 &=
 \beta\,\varphi(\beta s)
 \nonumber\\
 &\geq
 \beta\,\varphi(\beta)
 =
 \frac{\beta(2+\beta)}{(1+\beta)^2}
 =
 k_\beta^{-2}.
 \label{eq:energy_coercivity_bound}
\end{align}
Consequently,
\[
 \frac{1}{\sqrt a}\leq k_\beta.
\]
Substituting
\eqref{eq:energy_factor_bound} and
\eqref{eq:energy_coercivity_bound} into
\eqref{eq:energy_root_bound} gives
\[
 \|\zeta_{\mathrm{fit}}(x)\|
 \leq
 c_\beta\|\zeta_{\mathrm{proto}}(x)\|
 +
 k_\beta\sqrt{|D_\beta(x)|},
\]
which is~\eqref{eq:nlms_pointwise_bound} and completes the proof.
\section{Proof of Proposition~\ref{proposition_120820261330}}
\label{app:closure_risk_estimation}
The proof uses two ways to compare the fitted predictor with the class-mean matrix: a direct difference of predictions and the one-update identity. Both use the $L^2$ triangle inequality.

For a vector-valued function $g$, use the notation
\[
 \|g\|_{L^2(\mathbb P_x)}
 :=
 \left(
 \mathbb{E}_{x\sim\mathbb P_x}
 \!\left[\|g(x)\|^2\right]
 \right)^{1/2}.
\]
All fitted objects are treated as fixed after conditioning on
$\mathcal F_{\mathrm{fit}}$.
\paragraph{Direct estimation bound.}
Subtracting $\widehat B^\top\Phi(x)$ from
\eqref{eq_110820261955} gives
\begin{equation}
 \zeta_{\mathrm{fit}}(x)
 =
 (B_*-\widehat B)^\top\Phi(x)
 +B_*^\top\bigl(y(x)-\Phi(x)\bigr)
 +\xi_{\mathrm{res}}(x).
 \label{eq:appendix_exact_residual_decomposition}
\end{equation}
Applying the triangle inequality (Minkowski's inequality) in $L^2(\mathbb P_x)$ bounds the residual by the sum of the three component norms:
\begin{align}
 \|\zeta_{\mathrm{fit}}\|_{L^2(\mathbb P_x)}
 &\leq
 \|(B_*-\widehat B)^\top\Phi\|_{L^2(\mathbb P_x)}
 +\|B_*^\top(y-\Phi)\|_{L^2(\mathbb P_x)}
 +\|\xi_{\mathrm{res}}\|_{L^2(\mathbb P_x)}
 \nonumber\\
 &\leq
 \varepsilon_{\mathrm{est}}
 +\|B_*\|_2a_\Phi
 +\sigma_{\mathrm{res}}.
 \label{eq:appendix_direct_risk}
\end{align}
Indeed, by the spectral-norm inequality,
\[
 \|B_*^\top(y(x)-\Phi(x))\|
 \leq
 \|B_*\|_2\|y(x)-\Phi(x)\|
\]
pointwise, and hence
\[
 \|B_*^\top(y-\Phi)\|_{L^2(\mathbb P_x)}
 \leq
 \|B_*\|_2a_\Phi.
\]
The other two terms in
\eqref{eq:appendix_direct_risk} are
$\varepsilon_{\mathrm{est}}$ and
$\sigma_{\mathrm{res}}$ by
\eqref{eq:estimation_residual_scales}.
Since
\[
 \|\zeta_{\mathrm{fit}}\|_{L^2(\mathbb P_x)}
 =
 R_{\mathbb P_x}(\Phi,\widehat B)^{1/2},
\]
equation~\eqref{eq:appendix_direct_risk} proves \eqref{eq:direct_closure_bound}.
\paragraph{Bound using parameter-error change.}
Proposition~\ref{proposition_110820262057} gives the pointwise
inequality
\[
 \|\zeta_{\mathrm{fit}}(x)\|
 \leq
 c_\beta\|\zeta_{\mathrm{proto}}(x)\|
 +k_\beta\sqrt{|D_\beta(x)|}.
\]
Applying Minkowski's inequality gives
\begin{align}
 \|\zeta_{\mathrm{fit}}\|_{L^2(\mathbb P_x)}
 &\leq
 c_\beta
 \|\zeta_{\mathrm{proto}}\|_{L^2(\mathbb P_x)}
 +
 k_\beta
 \left(
 \mathbb{E}_{x\sim\mathbb P_x}
 \!\left[|D_\beta(x)|\right]
 \right)^{1/2}
 \nonumber\\
 &=
 c_\beta
 \|\zeta_{\mathrm{proto}}\|_{L^2(\mathbb P_x)}
 +
 k_\beta\varepsilon_{\mathrm{en}}.
 \label{eq:appendix_energy_risk}
\end{align}
From~\eqref{eq:prototype_residual},
\[
 \zeta_{\mathrm{proto}}(x)
 =
 B_*^\top\bigl(y(x)-\Phi(x)\bigr)
 +\xi_{\mathrm{res}}(x),
\]
so another application of Minkowski's inequality gives
\begin{align*}
 \|\zeta_{\mathrm{proto}}\|_{L^2(\mathbb P_x)}
 &\leq
 \|B_*^\top(y-\Phi)\|_{L^2(\mathbb P_x)}
 +
 \|\xi_{\mathrm{res}}\|_{L^2(\mathbb P_x)}\\
 &\leq
 \|B_*\|_2a_\Phi+\sigma_{\mathrm{res}}.
\end{align*}
Substituting this bound into
\eqref{eq:appendix_energy_risk} yields
\[
 R_{\mathbb P_x}(\Phi,\widehat B)^{1/2}
 \leq
 c_\beta
 \bigl(
 \|B_*\|_2a_\Phi+\sigma_{\mathrm{res}}
 \bigr)
 +
 k_\beta\varepsilon_{\mathrm{en}},
\]
which proves~\eqref{eq:energy_closure_bound}. Both~\eqref{eq:direct_closure_bound} and
\eqref{eq:energy_closure_bound} hold for any fixed fitted matrix $\widehat B$, including one obtained from dependent training transitions.
\paragraph{Class-mean sensitivity and reconstruction margins.}
The assignment contribution is $\|B_*^\top(y-\Phi)\|_{L^2(\mathbb P_x)}$. Replacing this term by any upper bound $q_\Phi$ preserves the direct and energy inequalities in~\eqref{eq:direct_closure_bound} and~\eqref{eq:energy_closure_bound}. In that case,
\begin{align*}
 R_{\mathbb P_x}(\Phi,\widehat B)^{1/2}
 &\leq
 q_\Phi+\sigma_{\mathrm{res}}
 +\varepsilon_{\mathrm{est}},\\
 R_{\mathbb P_x}(\Phi,\widehat B)^{1/2}
 &\leq
 c_\beta
 \bigl(q_\Phi+\sigma_{\mathrm{res}}\bigr)
 +k_\beta\varepsilon_{\mathrm{en}}.
\end{align*}
Proposition~\ref{prop:prototype_diameter} gives
\[
 \|B_*^\top(y-\Phi)\|_{L^2(\mathbb P_x)}
 \leq
 L_{\mathrm{proto}}\sqrt{A_{\mathrm{fid}}}.
\]
Taking
\[
 q_\Phi
 =
 L_{\mathrm{proto}}\sqrt{A_{\mathrm{fid}}}
\]
in these two error inequalities proves \eqref{eq:fidelity_direct_risk}-\eqref{eq:fidelity_energy_risk}. For the bound derived from KAHM reconstruction-score margins, Theorem~\ref{thm:kahm_margin_fidelity} gives
\[
 A_{\mathrm{fid}}\leq A_{\mathrm{geom}}.
\]
Combining this inequality with Proposition~\ref{prop:prototype_diameter} yields
\[
 \|B_*^\top(y-\Phi)\|_{L^2(\mathbb P_x)}
 \leq
 L_{\mathrm{proto}}\sqrt{A_{\mathrm{fid}}}
 \leq
 L_{\mathrm{proto}}\sqrt{A_{\mathrm{geom}}}.
\]
Hence we may take
\[
 q_\Phi
 =
 L_{\mathrm{proto}}\sqrt{A_{\mathrm{geom}}}.
\]
Defining
\[
 T_{\mathrm{geom}}
 :=
 L_{\mathrm{proto}}\sqrt{A_{\mathrm{geom}}}
 +\sigma_{\mathrm{res}},
\]
the two resulting bounds are
\[
 R_{\mathbb P_x}(\Phi,\widehat B)^{1/2}
 \leq
 T_{\mathrm{geom}}+\varepsilon_{\mathrm{est}}
\]
and
\[
 R_{\mathbb P_x}(\Phi,\widehat B)^{1/2}
 \leq
 c_\beta T_{\mathrm{geom}}
 +k_\beta\varepsilon_{\mathrm{en}}.
\]
Since both inequalities hold simultaneously, taking their minimum
gives~\eqref{eq:geometry_closure_bound}. The class-mean distances and reconstruction-score margins give bounds on the assignment contribution. Within-class successor variation $\sigma_{\mathrm{res}}$, prediction differences from $B_*$ measured by $\varepsilon_{\mathrm{est}}$, and the hypothetical-update term $\varepsilon_{\mathrm{en}}$ enter the bound separately.
\section{Concentration Details and Independent-Trajectory Evaluation}
\label{app:fidelity_concentration}
Condition throughout on $\mathcal F_{\mathrm{fit}}$, so the representation and fitted matrices are fixed. The evaluation data are sampled independently of the data and random choices used for construction, fitting, and selection.
\paragraph{Reference-label and reconstruction-margin losses.}
Conditional on the fitted objects, the evaluation losses $\ell_i=(1-(\Phi(x'^{i}))_{c_i})^2$ are independent, lie in $[0,1]$, and have common mean $A_{\mathrm{fid}}$. Hoeffding's exponential-moment inequality~\citep{hoeffding1963} gives
\begin{equation}
\Pr_{\mathcal V\sim\mathbb P_{\mathcal V}}(A_{\mathrm{fid}}-\widehat A_{\mathrm{fid}}>t)
\leq\inf_{\lambda>0}\exp(-\lambda Mt+M\lambda^2/8).
\label{eq:appendix_fidelity_mgf}
\end{equation}
Choosing $\lambda=4t$ yields
\begin{equation}
\Pr_{\mathcal V\sim\mathbb P_{\mathcal V}}(A_{\mathrm{fid}}-\widehat A_{\mathrm{fid}}>t)\leq e^{-2Mt^2}.
\label{eq:appendix_fidelity_tail}
\end{equation}
Set $t=\sqrt{\log(1/\delta)/(2M)}$ and use $A_{\mathrm{fid}}\leq1$ to obtain~\eqref{eq:fidelity_confidence_event}. Averaging this conditional guarantee over the construction, fitting, and selection data gives the same unconditional $1-\delta$ coverage.

Fix positive margin thresholds $g_c$ before evaluation. The loss $W_g$ in~\eqref{eq:geometric_evaluation_loss} lies in $[0,1]$, has expectation $A_{\mathrm{geom}}$, and is at least as large as the assignment loss at every state. The same argument gives
$A_{\mathrm{fid}}\leq A_{\mathrm{geom}}\leq U_{\mathrm{geom},\delta}$ with confidence $1-\delta$.
On the same sample at the same confidence level, $\widehat A_{\mathrm{fid}}\leq\widehat A_{\mathrm{geom}}$ implies $U_\delta\leq U_{\mathrm{geom},\delta}$.
The reconstruction-score-based version explains the role of score separation, while directly evaluating deviation from the reference label gives an equally tight or tighter upper bound. For a joint confidence statement about several bounds, allocate their failure probabilities so that their sum is at most the allowed total, or establish another simultaneous guarantee. Margin thresholds and representations must be fixed before evaluation, unless a simultaneous confidence guarantee covers their selection.

\paragraph{Independent trajectories.}
Let $\mathsf T=(T,(x_t,x_t^+)_{t=1}^T)$ be a random evaluation trajectory with distribution $\mathbb P_{\mathsf T}$ and $1\leq T<\infty$ almost surely. Each $x_t^+$ is the successor associated with $x_t$. Draw $J$ i.i.d. trajectory realizations, independently of $\mathcal F_{\mathrm{fit}}$, and write them as $\mathsf T^j=(T_j,(x_t^j,x_t^{+,j})_{t=1}^{T_j})$, $j=1,\ldots,J$. The superscript $j$ identifies the observed trajectory and the subscript $t$ identifies time within it. Within-trajectory dependence is unrestricted. The independent unit in the concentration argument is one complete trajectory, so we first reduce each trajectory to one bounded average loss. Write $\mathcal W_J=(\mathsf T^1,\ldots,\mathsf T^J)$ and $\mathbb P_{\mathcal W_J}=(\mathbb P_{\mathsf T})^{\otimes J}$ for the random trajectory sample and its sampling law. Define
\begin{equation}
\ell_{\mathrm{traj}}^j:=\frac1{T_j}\sum_{t=1}^{T_j}(1-\Phi_{c(x_t^j)}(x_t^j))^2\in[0,1],
\label{eq:trajectory_fidelity_loss}
\end{equation}
and the state distribution obtained by drawing a trajectory and then choosing one of its states uniformly:
\begin{equation}
\overline{\mathbb P}_x(A)
:=\mathbb{E}_{\mathsf T\sim\mathbb P_{\mathsf T}}\!\left[\frac1T\sum_{t=1}^T\mathbbm{1}_A(x_t)\right],
\qquad A\in\mathcal B(\mathcal X).
\label{eq:trajectory_weighted_law}
\end{equation}
The independent trajectory losses have common mean
\[
\overline A_{\mathrm{fid}}
=\mathbb{E}_{x\sim\overline{\mathbb P}_x}\!\left[(1-\Phi_{c(x)}(x))^2\right].
\]
With $\widehat A_{\mathrm{fid}}^{\mathrm{traj}}=J^{-1}\sum_{j=1}^J\ell_{\mathrm{traj}}^j$, Hoeffding gives
\[
\Pr_{\mathcal W_J\sim\mathbb P_{\mathcal W_J}}\!\left(\overline A_{\mathrm{fid}}\leq
\min\left\{1,\widehat A_{\mathrm{fid}}^{\mathrm{traj}}+
\sqrt{\frac{\log(1/\delta)}{2J}}\right\}\right)\geq1-\delta.
\]
The same statement applies to within-trajectory averages of $W_g$. For implemented coordinates with the uniform fallback, use $\widetilde W_g$ from~\eqref{eq:implemented_geometric_evaluation_loss} instead. This distribution gives every trajectory equal weight and every state within a trajectory equal weight. Thus a state in a short trajectory receives more weight than a state in a long trajectory. Pooling all states with equal weights would give a different distribution when lengths differ. Class probabilities, successor means, variances, and risks must all be computed for the stated distribution, with positive probability for every retained class.

The corresponding joint transition distribution is defined, for every bounded measurable function $h$, by
\[
\mathbb{E}_{(x,x^+)\sim\overline{\mathbb P}}\!\left[h(x,x^+)\right]
=\mathbb{E}_{\mathsf T\sim\mathbb P_{\mathsf T}}\!\left[\frac1T\sum_{t=1}^T h(x_t,x_t^+)\right].
\]
To apply Proposition~\ref{prop:evaluable_risk}, average each of its three losses within each trajectory and replace the sample count $M$ by $J$. The direct bound in Remark~\ref{rem:direct_risk_upper} extends in the same way. In both cases the target risk is taken under $\overline{\mathbb P}_{x,x^+}$, which weights trajectories equally and transitions within each trajectory uniformly.
For Proposition~\ref{prop:exclusion_certificate}, use
\begin{align*}
\widehat f&=\frac1J\sum_{j=1}^J\frac1{T_j}\sum_{t=1}^{T_j}(1-\Phi_{c(x_t^j)}(x_t^j))^2,\\
\widehat s&=\min_{a^1,\ldots,a^C\in\Delta_C}
\frac1J\sum_{j=1}^J\frac1{T_j}\sum_{t=1}^{T_j}\|\Phi(x_t^{+,j})-a^{c(x_t^j)}\|^2.
\end{align*}
Compute each class mean using transition weights $1/(JT_j)$ for trajectory $j$. Empty classes contribute zero to the loss. Replacing one complete trajectory changes $\widehat s$ by at most $2/J$. Also, its expected value over the $J$ independent trajectories is no larger than the population within-class successor variance under $\overline{\mathbb P}_{x,x^+}$, because the empirical means minimize the weighted sample loss. The same bounded-differences and assignment-loss arguments therefore give~\eqref{eq:exclusion_certificate} with $r_\delta=\log(2/\delta)/J$. Thus the confidence correction uses $J$, the number of independent trajectories.

\section{Infimum identity used in Result~\ref{result_210820160745}}\label{app:compact_risk}
For $a,b\geq0$,
\begin{equation}
\inf_{t>0}[(1+t)a+(1+t^{-1})b]=(\sqrt a+\sqrt b)^2.
\label{eq:young_infimum_repaired}
\end{equation}
The expression equals $a+b+ta+b/t$. If $a,b>0$, its minimum occurs at $t=\sqrt{b/a}$ and equals $a+b+2\sqrt{ab}$. If only $b=0$, take $t\downarrow0$; if only $a=0$, take $t\to\infty$. If both vanish, every $t>0$ attains the infimum.

% ===== END appendices/technical_details.tex =====

% ===== BEGIN appendices/experiment_protocols.tex =====
\section{Experimental Protocols and Interpretation}
\label{app:protocol}
\subsection{Purpose of the reported studies}
\label{app:experiment_selection}
The experiments address four questions: whether the certificate detects a prediction limitation; how much the coordinate vectors vary across sampled states; how fitted predictors compare on a common output target; and how prediction changes with horizon under a fixed policy. Table~\ref{tab:evaluation_roles} connects the reported studies to these questions. Experiment numbers identify the corresponding implementations and archived records. The table includes favorable, unfavorable, and inconclusive outcomes.

\begin{longtable}{@{}>{\raggedright\arraybackslash}p{0.26\linewidth}>{\raggedright\arraybackslash}p{0.27\linewidth}>{\raggedright\arraybackslash}p{0.41\linewidth}@{}}
\caption{Scientific questions, methods, and evaluation roles of the reported experiments. Experiment numbers refer to implementation-script identifiers.}\label{tab:evaluation_roles}\\
\toprule
Study family & Role here & Selection rationale \\
\midrule\endfirsthead
\toprule Study family & Role here & Selection rationale \\
\midrule\endhead
Four-state certificate evaluation (21) & Checks against an analytical optimum. & The known population optimum allows coverage checks. Counts and tolerance decisions are reported separately from the dynamical pilot. \\
Dynamical pilot and independent certificate evaluation (22-23) & Application to one selected representation. & Zero pilot outcomes, evaluation on new pairs after selection, and results at all four norm limits are included. \\
Coordinate construction (6-7) and selection by centered error (9, 12) & Representation diagnostics and hyperparameter selection. & Each map predicts its own coordinates; variation diagnostics help interpret these changing targets. \\
Increasing training-trajectory count (13) & Effect of additional training data. & Five training counts and three coordinate configurations are compared; each map is reconstructed and predicts its own coordinates. \\
State-space predictor baselines (15) & Comparison on a common target. & Direct coordinate and state-space predictors share an output map within each oscillator system. The comparison covers Duffing and Van der Pol dynamics. \\
Classic Control and Acrobot (16, 18) & Main fixed-policy diagnostics. & One-step and 50-step scores include negative long-horizon outcomes. The measured outcome is coordinate prediction under fixed policies. \\
Dominant-coordinate visualization (17) & One descriptive MountainCar figure. & Describes groups of visited states through their dominant coordinates. \\
Projection (4, 8) and spectral diagnostics (19) & Compact supporting appendix. & Clarify the difference between simplex feasibility, linearity, and spectral interpretation.  \\
Training-fraction and other reference diagnostics (1-5, 10-11, 14, 16a) & Additional sensitivity diagnostics. & They concern repeated configurations, changing evaluation segments, and sensitivity analyses of a single fit. \\
\bottomrule
\end{longtable}

\subsection{State generation, reconstruction, and fitting}
The oscillator equations and process-noise settings appear in Section~\ref{sec:exp_protocol}. Runs without added observation noise still include the specified process noise. When training observation noise is added, its standard deviation for state coordinate $j$ is $\sigma_{\mathrm{obs}}\operatorname{std}(x_j)$, with the scale calculated from training values. No observation perturbation is added to evaluation trajectories. State coordinates are used without whitening or standardization, so Euclidean clustering and reconstruction depend on these observed scales.

For each candidate, $K$-means is fitted to training states. If a class contains only one sample $x$, an additional reference point $0.95x+0.05x_{\mathrm{nn}}$ is constructed using the nearest other training sample. Reference samples are stored in single precision. Let $\widetilde N_c$ be the number of reference samples in class $c$ after this augmentation and set $N_b=100$. The class is divided, in input order, into $S_{\mathrm{sub}}^c=\max\{1,\lfloor\widetilde N_c/N_b+1/2\rfloor\}$ blocks whose sizes differ by at most one. Each block supplies a KAHM reconstruction model. The class score is the minimum over these models, so this internal subdivision leaves the coordinate dimension at $C$.

The score implementation replaces an exactly zero affine-weight denominator by one. If the query or reconstruction has zero norm, it assigns angular score $1/2$. Cosine similarities are clipped to $[-1,1]$, and class scores are clipped to $[0,1]$. With $\tau=10^{-6}$ and implemented affinities $a_c(x)=1-\widetilde{\mathcal T}_c(x)+\tau$, the returned coordinates are
\begin{equation}
\widetilde\Phi_c(x)=
\begin{cases}
a_c(x)^\omega/\sum_j a_j(x)^\omega,
&\sum_j a_j(x)^\omega>10^{-12},\\
1/C,&\sum_j a_j(x)^\omega\leq10^{-12}.
\end{cases}
\label{eq:implemented_affinity_normalization}
\end{equation}
The second branch is an absolute-threshold rule, not merely a response to floating-point underflow. Small, accurately representable positive values can activate this rule and cause the implementation to replace nonuniform coordinates with a uniform vector. For example, with $C=25$, $\omega=128$, and every affinity at most $0.78$, the sum is at most $25(0.78)^{128}<4\times10^{-13}$. Changing $\omega$ therefore affects both the relative weights and the set of states assigned uniform coordinates. The same convention applies when encoding a state-space forecast.

The certificate is evaluated for this fixed implemented simplex map. Bounds obtained from reconstruction-score margins retain the domain and normalization assumptions stated in their theorems; Remark~\ref{rem:numerical_score_scope} explains the distinction. The available summaries do not count activations of the uniform-coordinate rule. To attribute the change from $\omega=4$ to $128$ to more concentrated relative weights, the rule's activation frequency must also be measured for training states, current evaluation states, and successors. A numerically stable evaluation of the formal normalization can change the coordinates where the current fallback operates. Claims for a changed map require reevaluation of the quantities that depend on those coordinates. A matrix presented as trained for the changed coordinates should also be refitted; the certificate itself can be recomputed without fitting a matrix.

Matrix fitting uses the NLMS update in~\eqref{eq_230720260912}, step size $0.1$, identity initialization, and 20 epochs without shuffling. Once coordinates are evaluated, each dense update costs $O(C^2)$ and storing the matrix costs $O(C^2)$. Training on $N$ pairs for $N_{\mathrm{ep}}$ epochs costs $O(N_{\mathrm{ep}}NC^2)$, excluding clustering, reconstruction, and feature evaluation. 

For the oscillator centroid ablation, let $\bar x_c$ denote a training-fitted centroid. Hard coordinates select the nearest centroid. The soft alternatives are
\begin{align}
\Phi_c^{\mathrm{dist}}(x)
&=\frac{\bigl(1-\min\{\|x-\bar x_c\|/s_{\mathrm{dist}},1\}+\tau\bigr)^\omega}
{\sum_j\bigl(1-\min\{\|x-\bar x_j\|/s_{\mathrm{dist}},1\}+\tau\bigr)^\omega},\\
\Phi_c^{\mathrm{RBF}}(x)
&=\frac{\exp[-\|x-\bar x_c\|^2/(2\sigma_{\mathrm{RBF}}^2)]}
{\sum_j\exp[-\|x-\bar x_j\|^2/(2\sigma_{\mathrm{RBF}}^2)]}.
\label{eq:centroid_baseline_maps}
\end{align}
The distance scale is the 95th percentile of training distances to the nearest centroid, and the RBF bandwidth is their median; only distances above $10^{-12}$ are used. If no such distance remains, both scales are set to one; otherwise each is bounded below by $10^{-12}$. The RBF multiplier is one. Before exponentiation, the implementation subtracts the largest RBF exponent to improve numerical stability. Each map supplies its own prediction target. The displayed distance transform applies to the oscillator comparison; control implementations require their separately specified centroid conventions.

The variation diagnostics in~\eqref{eq:retained_variation}-\eqref{eq:retained_effective_rank} have separate numerical tolerances. In the implementation, $\rho_{\mathrm{var}}$ is set to zero when $1-\|\bar\phi\|^2\leq10^{-12}$, and effective rank is set to zero when total variance is at most $10^{-12}$. Negative computed covariance eigenvalues are clipped to zero. The entropy calculation retains eigenvalues above $\epsilon_{\mathrm{mach}}C\lambda_{\max}$, where $\lambda_{\max}$ is the largest computed eigenvalue. These thresholds specify the numerical implementations of the diagnostics; theoretical dimensions and variances use the exact definitions.

\subsection{Hyperparameter selection and repeated-data units}
Table~\ref{tab:exp_search_protocol} records the supplied candidate grids. Both Duffing searches and the Van der Pol search without added observation noise use seeds $0,1,2$ per candidate. The Van der Pol search without added observation noise averages $1-R_h^2$ over $h\in\{1,10,50,100,200\}$, forms the one-standard-error set, and uses the variation diagnostics to resolve remaining choices. The Van der Pol search across observation-noise levels uses horizons $50,100,200$ with weights $1/4,1/4,1/2$ and averages over $\sigma_{\mathrm{obs}}\in\{0,0.005,0.01,0.02,0.05\}$. Its perturbation seed depends on the candidate settings. A comparison that changes $\omega$ can therefore also change the realized training noise. All scores used to choose a map are reported as selection statistics.

\begin{table}[htbp]
\centering
\caption{Reported candidate grids. Count is the number of $(C,\omega)$ combinations and excludes repeated fits. Duffing searches and Van der Pol selection without added observation noise use seeds $0,1,2$; Van der Pol selection across noise levels additionally uses five training-noise levels. Acrobot selection uses seed $0$, followed by assessment of the selected setting with seeds $0,1,2$. Here $\mathcal H=\{1,10,50,100,200\}$. ``1-SE'' means the one-standard-error rule, with normalized variance and then normalized effective rank used to resolve remaining choices. The weighted centered-error score uses horizons $50,100,200$ with weights $1/4,1/4,1/2$.}
\label{tab:exp_search_protocol}
\small
\setlength{\tabcolsep}{4pt}
\begin{tabularx}{\linewidth}{@{}Y Y Y r Y@{}}
\toprule
Study & $C$ & $\omega$ & Count & Objective\\
\midrule
Duffing reference
& $15,\allowbreak 25,\allowbreak 35$
& $4,\allowbreak 8,\allowbreak 12$
& 9
& $E_1$\\

Duffing expanded
& $8,\allowbreak 10,\allowbreak 12,\allowbreak 15,\allowbreak 20,\allowbreak 25,\allowbreak 30,\allowbreak 40$
& $2,\allowbreak 4,\allowbreak 6,\allowbreak 8,\allowbreak 10,\allowbreak 12,\allowbreak 16,\allowbreak 20$
& 64
& $\overline E_{\mathcal H}$\\

Van der Pol, no added observation noise
& $10,\allowbreak 15,\allowbreak 20,\allowbreak 25,\allowbreak 30,\allowbreak 40,\allowbreak 50$
& $0.5,\allowbreak 1,\allowbreak 2,\allowbreak 4,\allowbreak 6,\allowbreak 8,\allowbreak 12$
& 49
& $\overline{(1-R^2)}_{\mathcal H}$; 1-SE\\

Van der Pol, across observation-noise levels
& $10,\allowbreak 15,\allowbreak 20,\allowbreak 25,\allowbreak 30,\allowbreak 40$
& $0.25,\allowbreak 0.5,\allowbreak 1,\allowbreak 2,\allowbreak 4$
& 30
& weighted $(1-R^2)$; 1-SE\\
Acrobot
& $20,\allowbreak 40,\allowbreak 50,\allowbreak 75,\allowbreak 100,\allowbreak 150$
& $0.5,\allowbreak 1,\allowbreak 2,\allowbreak 4,\allowbreak 8$
& 30
& $E_{50}$\\
\bottomrule
\end{tabularx}
\end{table}

The expanded Duffing search ranks the average relative squared error across the five horizons and selects $(C,\omega)=(10,2)$. Van der Pol selection without added observation noise chooses $(25,6)$; selection across observation-noise levels chooses $(25,4)$. These maps differ from the certificate map, $(25,128)$. The reference diagnostic settings are Duffing $(15,12)$ and Van der Pol $(20,4)$.

In the trajectory-count experiment, training-seed blocks begin at $0$, $1000$, and $2000$; within each block all compared configurations use the same trajectories. Each fitted model is evaluated on trajectories $100,101,102$. Paired error differences are formed within a block and summarized across the three blocks. The primary columns of Table~\ref{tab:exp_trajectory_count} report descriptive standard deviations across the three fits; its difference columns report standard errors. The common-target oscillator comparison instead trains one model per method and summarizes three evaluation trajectories. The summaries therefore distinguish variability across fitted models from variability across evaluation trajectories.

The source description identifies a 27-job Duffing reference-grid run at \path{reproduction/prospective_campaigns/duffing_reference_27/}, with a schedule, exit statuses, logs, and checksums. This count equals nine candidate settings evaluated for each of three seeds.

\subsection{Fixed-policy episode protocols}
The observation vectors have four coordinates for CartPole, two for MountainCar, and six for Acrobot. CartPole and MountainCar each use eight training episodes and four evaluation episodes. The episode cap is 250 steps for CartPole and 200 for MountainCar. Acrobot uses 12 training episodes and six evaluation episodes, all capped at 500 steps. A rollout start is included only if its target at the requested horizon remains in the same episode. Within a run, all eligible starts are pooled before computing the squared-error sums and target mean in~\eqref{eq:exp_metrics}. Longer episodes therefore receive more weight because they have more eligible starts. Policies are fixed during collection, with action randomization as specified in Section~\ref{sec:exp_policy}.

The following rules specify the fixed policies. At each step, with probability $\epsilon$, the action is sampled uniformly from the complete discrete action set. Otherwise the deterministic rule for that environment is used. The probabilities are $\epsilon=0.05$ for CartPole and MountainCar and $0.10$ for Acrobot.

For CartPole observations $(q,\dot q,\theta,\dot\theta)$, action $1$ is chosen when
\[
\theta+0.25\dot\theta+0.01q+0.10\dot q>0,
\]
and action $0$ otherwise. For MountainCar observations $(q,v)$, action $2$ is chosen when $q>0.45$ or $v\geq0$; otherwise action $0$ is used.

For Acrobot, recover $\theta_1,\theta_2$ from the observed sine-cosine pairs using $\operatorname{atan2}$. Define
\[
H=-\cos\theta_1-\cos(\theta_1+\theta_2),\qquad
v=\dot\theta_1+\dot\theta_2,
\]
\[
s=v+0.40\sin(\theta_1+\theta_2)+0.20\sin\theta_1+0.10\dot\theta_2.
\]
If $H>0.6$, choose action $2$ for $v\geq0$ and $0$ otherwise. If $H\leq0.6$, choose action $1$ when $|s|<0.05$; otherwise choose $2$ for $s>0$ and $0$ for $s<0$. Randomization samples from the full action set, including the deterministic action; $\epsilon$ is the probability of invoking this randomization.

At the reported CartPole and MountainCar settings, base training seeds $0,1,2$ are paired with base evaluation seeds $1000,1001,1002$. Acrobot uses the same training seeds with evaluation seeds $100,101,102$. The reset seed for an episode is its base seed plus its episode index. Training and evaluation reset ranges are disjoint within a fit, but adjacent fitted runs have overlapping ranges. For example, eight training episodes with base seeds $0$ and $1$ use reset seeds $0$-$7$ and $1$-$8$. The three fits are therefore dependent. Acrobot selection uses seed $0$, followed by three fits at the selected setting. A separate CartPole and MountainCar sensitivity study evaluates $C\in\{10,15,20,30,50,80,100\}$ and $\omega\in\{2,4,8,12\}$ with training seed $0$, evaluation seed $100$, and selection horizon 50. It contains no $\omega=1$ candidate and therefore does not document the choice of the fixed configurations reported in Table~\ref{tab:exp_policy}. The control studies provide exploratory diagnostics without an untouched final test set.

Acrobot summaries use the descriptive standard deviation $[\tfrac13\sum_{r=1}^3(u_r-\bar u)^2]^{1/2}$ for a score $u_r$ across three fits. This convention, with divisor $3$, is stated in the relevant captions and is preserved in the reported values.

The MountainCar visualization uses one diagnostic fit with training seed $0$ and evaluation seed $2000$. It colors states by their dominant coordinate, providing a descriptive grouping that can differ from the nearest-center reference labels used by the certificate.

The implementation locations recorded for these formulas are \path{experiment_06_duffing_abstraction_ablation.py}, \path{experiment_16_ai_classic_control_closed_loop_pylance_clean.py}, and \path{experiment_18_acrobot_sequential_decision_pylance_clean.py} at revision \texttt{0bbb2ba}. The one-standard-error selector without added observation noise is identified as \path{run_exp09_vanderpol_tuning.py}, while the normalization and variation thresholds are identified in \path{kernel_affine_hull_koopman_machines.py}. These paths identify the reported implementations; their verification and access status is stated in Appendix~\ref{app:certificate_provenance}.

\subsection{Numerical scope and source records}
\label{sec:fidelity_diagnostics}
For independent states and their reference labels, the empirical assignment statistic is
\begin{equation}
\widehat A_{\mathrm{fid}}=M^{-1}\sum_{i=1}^M(1-\Phi_{c(x'^{i})}(x'^{i}))^2.
\label{eq:empirical_fidelity_diagnostics}
\end{equation}
This equals $\widehat f$ in the certificate and directly measures empirical label agreement. The certificate campaigns evaluate Proposition~\ref{prop:exclusion_certificate}. The controlled family supplies known population optima; the dynamical application compares certificates with the empirical error of a fixed fitted predictor. Empirical validation in these campaigns concerns the lower certificate. The complementary fitted-predictor upper bounds, finite-update NLMS inequalities, and deterministic exact-closure results remain analytical contributions with their separately stated assumptions.

The principal implementation is \path{kernel_affine_hull_koopman_machines.py}, with reconstruction utilities in \path{parallel_autoencoders.py} and \path{combine_multiple_autoencoders_extended.py}. The relevant forecast runners include \path{run_exp09_vanderpol_tuning.py}, \path{run_exp12_vanderpol_noise_aware_tuning.py}, \path{run_exp13_vanderpol_tuned_trajectory_generalization.py}, and the two \path{run_exp15_*_tuned_external_baselines.py} scripts. Reported control runs use \path{run_exp16_tuned.py} and \path{run_exp18_acrobot_selected_3seed_validation.py}. Spectral diagnostics use \path{experiment_19_residual_verified_kahkm_spectra.py}. The repository's \path{reproduction/reported_results.tsv}, \path{reproduction/run_accounting/}, and \path{reproduction/REPRODUCIBILITY_EVIDENCE.md} connect reported values with stored outputs.

The forecasting inputs are archived under \path{reproduction/raw_data/canonical/}, with a manifest and checksums; the certificate evaluation uses separately archived datasets. The reference environment is identified by \path{requirements-lock-arm64.txt} and \path{provenance/reference_arm64_environment/}. The recorded licenses are MIT for author-owned code and CC BY 4.0 for distributed author-generated data; third-party components retain their own licenses. The identified files and license texts are present at the repository revision cited in the Code availability statement.

\subsection{Detailed implementation conventions and remaining limits}
\label{app:forecasting_protocol_gaps}
The released configuration and per-run records specify the data roles used
by the forecasting studies. In the Duffing coordinate-construction study,
each seed supplies 1,200 consecutive transitions: the first 900 are used
for fitting and the last 300 for evaluation. The expanded Duffing search,
Van der Pol coordinate-construction study, and the two Van der Pol
hyperparameter searches use 1,600 transitions per seed, split into 1,200
fitting pairs and 400 later pairs for evaluation or selection. These
segments belong to the same trajectory; selection statistics from them
are not independent final-test estimates. The trajectory-count study
instead uses 900 transitions per training trajectory and separate
800-transition evaluation trajectories, with the training and evaluation seed blocks specified in Section~\ref{sec:exp_robustness}. The released metadata and canonical input manifest identify the
individual inputs and their reuse.

Clustering uses the full \texttt{KMeans} implementation in scikit-learn
1.8.0, called with \texttt{n\_clusters=C}, the run seed as
\texttt{random\_state}, and \texttt{n\_init="auto"}. With that version's
$k$-means++ initialization, \texttt{n\_init="auto"} uses one initialization.
The other defaults are the Lloyd algorithm, \texttt{max\_iter=300}, and
\texttt{tol=1e-4}. The frozen certificate representation uses random state
$0$. Clustering uses the current-state fitting
matrix. State scaling is disabled. The
dependency lock fixes the implementation version, and the certificate
representation archive also supplies the actual fitted centers and
reconstruction models, so its evaluation does not require reclustering.

In the Van der Pol search across observation-noise levels, current and
successor fitting matrices receive separate sequential noise draws from
one generator for each candidate and seed. For each matrix separately,
the coordinate scale is its standard deviation with divisor equal to
the number of fitting columns, bounded below by $10^{-12}$. The added
noise has this scale multiplied by $\sigma_{\mathrm{obs}}$. The generator
seed is
\[
1000003r+10007C+\operatorname{round}(10^6\sigma_{\mathrm{obs}})
+\operatorname{round}(10^4\omega)+17,
\]
where $r$ is the trajectory seed. Thus a state that appears as both a
successor and a later current state can receive different observation
perturbations in those two fitting matrices. Evaluation states receive
no added observation noise. Scores are first combined over horizons and
noise levels within a run, then summarized across the three fitted runs.
The error-based one-standard-error rules use sample standard deviations
with divisor $n-1$; standard errors equal those standard deviations
divided by $\sqrt n$. For Acrobot, reported standard deviations use divisor $3$ rather than $n-1$, as specified in the control-protocol description.

The index \path{reproduction/reported_results.tsv} identifies per-run
forecast scores and summary records. Certificate-specific source revisions
are provided in Appendix~\ref{app:certificate_provenance}. Sample standard deviations for
the coordinate-construction comparison are present in the Experiment 6
and 7 multistep summary files. The canonical control archives preserve
episode boundaries. An episode with $T$ transitions contributes
$\max\{T-h+1,0\}$ eligible starts at horizon $h$; this determines the
weighting when starts are pooled. For the implemented relative squared
error, a denominator below $10^{-12}$ is replaced by $10^{-12}$.
The implemented centered $R^2$ is undefined when its centered target sum
of squares is nonpositive; there is no additional cutoff for a small
positive centered sum of squares. Very small target variation can
therefore produce large negative values.

Two aspects of the reported control experiments limit interpretation. The released CartPole and MountainCar runner
specifies the reported configurations directly, but the separate
sensitivity grid does not document the choice of their $\omega=1$
settings. Their results are therefore presented as exploratory fixed
configuration diagnostics. In addition, the released summaries do not
count use of the uniform-coordinate branch in
\eqref{eq:implemented_affinity_normalization}. Its frequency is needed
for training inputs, current evaluation states, successors, and encoded
forecast states, including both $\omega=4$ and $128$ where compared.
Without these counts, changes in prediction and certificate values
cannot be attributed solely to more concentrated relative affinities.

% ===== END appendices/experiment_protocols.tex =====

% ===== BEGIN appendices/certificate_provenance.tex =====
\section{Certificate Evaluation: Records and Verification Scope}
\label{app:certificate_provenance}
This appendix identifies the code commits and archived records for the certificate experiments in Section~\ref{sec:exp_certificate}. It distinguishes experiment execution, independent numerical verification, and summary generation.

\paragraph{Available experiment records.}
The experiment-evidence commit \path{0bbb2bafa107892e32e221e882326890b14fe36e} contains the experimental code, complete four-state evidence, canonical forecasting inputs, fixed Van der Pol representation, and compact confirmatory records. The four-state risk-bound comparison is archived at commit \path{78a1a5e7ea9d3a2fecd3b335bbe5a9f3f2237eb2}. The distributed file and archive-member checksums agree with their inventories. Recomputing the certificate from all 117,000 released four-state count records reproduces the 49,271 zero certificates and finds no certificate above its analytical optimum. Reaggregating the 64 released dynamical result records reproduces the displayed certificate summaries. These verification calculations use the archived evaluation data.

For the confirmatory Van der Pol study, the released records include the frozen protocol, configuration, replicate summaries, and independent numerical-verification report. Data access is described in the Data availability statement.

\subsection{Controlled four-state records}
The frozen configuration is \path{reproduction/certificates/configs/four_state_final_v1.json}. Its SHA-256 digest is
\begin{quote}\small
\path{630b6c30f5dcbda9034b85b2a3416d92ef264953cd05133f3e53fa14e7aae182}.
\end{quote}
The four-state campaign archive is available at \path{reproduction/prospective_campaigns/four_state_original_iid_final_v1_run001/}. Its inventory lists a state-count archive and separate verification records. The recorded row counts for \path{summary_v1/cell_summary.csv} and \path{summary_v1/exclusion_summary.csv} are 117 and 702, respectively. Each setting combines a fixed case and sample size with 1,000 replicate datasets. The 702 rows correspond to six tolerances for each of the 117 settings; tolerance decisions made on the same replicate are dependent.

The campaign ran from source revision \path{cde430dbcca09bf911439bdb5bc3f85350e232cd}. The campaign inventory has SHA-256 digest
\begin{quote}\small
\path{851025653f42a15bac4f2f17c2c1f792562aaccf461d94b8232ca5ac06023024}.
\end{quote}

The four-state post-hoc analysis reads frozen \path{samples.jsonl} and \path{results.jsonl} from \path{evidence.tar.gz} (SHA-256 \path{4776942edbd659f7a1c02ec554507f7c0259ca410449e2e5fd66f5353bf638d5}). At immutable revision \path{78a1a5e7ea9d3a2fecd3b335bbe5a9f3f2237eb2}, the directory \path{reproduction/posthoc_four_state_erm/} contains the input repackager, reanalysis and formula scripts, tests, \path{README_REPRODUCE.md}, and 117-cell and compressed 117,000-replicate results. The analysis reproduces all 117 design cells and 49,271 zero label-based certificates; its maximum absolute discrepancy from archived label-based bounds is $1.11\times10^{-16}$.

\paragraph{Reproduction identifiers for the four-method comparison.}
The four bounds use $\kappa=\sqrt2$ and $\delta=0.0125$ per method; the joint $95\%$ guarantee concerns one prespecified dataset, not all reported experiments. The prototypes $h^1=(p,1-p)^\top$ and $h^2=(1-p,p)^\top$ are design-known. Here $\widehat R^*_{\sqrt2,M}=\widehat s$, so no ERM solver was used. The file \path{README_REPRODUCE.md} gives the frozen-input commands (Python 3.13, NumPy 2.3.5). The SHA-256 hashes of the four deterministic \emph{uncompressed} CSV files are:
\begin{center}
\footnotesize
\begin{tabular}{@{}ll@{}}
\toprule
Derived CSV & SHA-256 \\
\midrule
\texttt{cell\_summary\_joint.csv} & \texttt{6ffea4bfe2e8e1e545ddb8aecf241fc1341a2c3b7a934767907c96d7e2bfac37} \\
\texttt{cell\_summary\_pointwise.csv} & \texttt{eefda7da73d7dc9024f36e33664b121e92cedbb6b305cb19c961ad8c298b35a6} \\
\texttt{replicates\_joint.csv} & \texttt{380476b3ce1f01129580ebfa0932828be7e2eae1c495d90bd980335278b859b9} \\
\texttt{replicates\_pointwise.csv} & \texttt{927d48bfd80707abf17ff319f712282c948b0a8162c77a86ab161d5aa4519373} \\
\bottomrule
\end{tabular}
\end{center}
Both \path{replicates_*.csv} files are distributed in \path{.csv.gz} form; their uncompressed contents must match these hashes. The timestamped run manifest may differ, but the input and output identities are fixed. The workflow \path{.github/workflows/four_state_erm_reanalysis.yml} validates the results. Table~\ref{tab:four_state_erm_historical} uses \path{selected_comparison_joint.csv} and \path{cell_summary_joint.csv}.

For the four-state experiment, the numerical verifier uses rational design values and 80-digit decimal arithmetic to distinguish strict threshold exceedance from equality.

\subsection{Dynamical protocol and random streams}
The independent-transition-pair design is specified by \path{reproduction/certificates/DYNAMICAL_CONFIRMATORY_PROTOCOL.md} and \path{reproduction/certificates/configs/dynamical_confirmatory_v1.json}. Their respective SHA-256 digests are
\begin{quote}\small
\path{daec3dc5bb96e115ce6c0bf82e64e7497318796b9cdd1dc80557b691f46b071c},\\
\path{1b1a6ec9e0d3de92f43c0f78e654f39d44ac6fd1d2d4286b88168027bd475a91}.
\end{quote}
The recorded workflow fixes the protocol before confirmatory sampling. It reuses the stored centers and reconstruction models and refits only the prediction matrix using training coordinates at $\omega=128$. Its verification record states that a separate implementation reproduced the matrix and complete NLMS update history in binary64 arithmetic before evaluation. The public package \path{reproduction/prospective_campaigns/dynamical_confirmatory_representation_v1_run001/evidence.tar.gz} contains the fixed representation, training and update records, and verification output.

The specified random-number construction uses NumPy's \texttt{SeedSequence} and \texttt{PCG64}. For each retained pair, a seed is derived from the ordered tuple
\[
(4,\,2,\,m,\,16384,\,r,\,i,\,s,\,20261003).
\]
Here $4$ identifies the confirmatory study and $2$ identifies Van der Pol. The law identifier is $m=1$ for matched stochastic evaluation and $m=2$ for deterministic evaluation; $r\in\{0,\ldots,31\}$ is the replicate index and $i\in\{0,\ldots,16383\}$ the pair index. The stream identifier is $s=1$ for trajectory randomness and $s=2$ for the time-index draw. The record retains two generated 32-bit words and the resulting 64-bit seed. The pilot uses study identifier $3$ in place of $4$. This scheme separates pilot streams from subsequent evaluation streams. The implementation calls \texttt{SeedSequence(entropy=tuple).generate\_state(2, dtype=np.uint32)} and combines the resulting words as $\texttt{int(words[0])}+2^{32}\texttt{int(words[1])}$. This unsigned 64-bit integer initializes \texttt{PCG64}. The probability guarantee concerns the stated independent sampling design; pseudorandom seed bookkeeping documents its computational implementation.

\begin{table}[tbp]
\centering
\small
\caption{Versioned implementation and verification records for the dynamical certificate evaluation. Full commit identifiers are provided in the associated reproducibility metadata.}
\label{tab:certificate_provenance}
\begin{tabularx}{\linewidth}{@{}Yl@{}}
\toprule
Operation & Git revision prefix \\
\midrule
Fix confirmatory protocol and configuration & \texttt{18c2bb4} \\
Build training-only $\omega=128$ representation and matrix & \texttt{4a0cedd} \\
Recompute the representation and matrix for verification & \texttt{cbfe438} \\
Generate 64 independent-pair evaluation datasets & \texttt{18dc9f3} \\
Recompute results from the completed evaluation & \texttt{090eac5} \\
Generate summaries and repeat numerical verification & \texttt{689d40d} \\
Release the \texttt{summary\_v1} package & \texttt{0bbb2ba} \\
\bottomrule
\end{tabularx}
\end{table}

\subsection{Independent numerical verification of dynamical results}
The workflow identifies \path{experiment_23_certificate_dynamical_confirmatory.py} as its runner and \path{scripts/verify_dynamical_confirmatory_numerics.py} as a separately implemented numerical verifier. The described verifier checks all 64 scheduled datasets and the retained checksum inventory. It reconstructs 2,097,152 trajectory and time seed streams and 1,048,576 time-index draws. Using saved physical-state pairs and the fixed representation, it recomputes current and successor coordinate arrays, nearest-center labels, class counts, class successor means and sums of squares, $\widehat f$, $\widehat s$, fitted-predictor errors, membership in each norm-constrained matrix class, and the certificates. It also compares the compact summaries with the records for each dataset.

The verifier uses the stored physical-state pairs and fixed coordinate map to recompute evaluation statistics. The summary generator \path{scripts/summarize_dynamical_confirmatory.py} runs the numerical checks before computing the reported tables.

The top-level campaign inventory has SHA-256 digest
\begin{quote}\small
\path{df9b1b7abd4a4c38b5bf8d0a5f43d7f6cb299ed23c070960b506f52641243221}.
\end{quote}
The recorded verification-report digest is
\begin{quote}\small
\path{7e36ececa64e0d27b6621814c9516cb85d5abd6d80429692e41543cfe6cd1ab3}.
\end{quote}
The published summary inventory has digest
\begin{quote}\small
\path{4d7e05746d9c575057c1a98ac44425b514dfffab0025b6f58d31bfb05de04606}.
\end{quote}
The reference environment recorded for this campaign is native ARM64 execution on the Apple M1 iMac, with Python 3.14.7, NumPy 2.4.4, and SciPy 1.17.1. The locked environment and run metadata specify the remaining dependencies and execution details.

\subsection{Confirmatory evaluation records}
The Van der Pol confirmatory campaign is documented under \path{reproduction/prospective_campaigns/dynamical_original_iid_confirmatory_v1_run001/}. Its compact archive \path{manifest_capsule.tar.gz} and \path{summary_v1/} directory provide the protocol and run metadata, \path{mode_summary.csv}, \path{budget_summary.csv}, all 64 rows of \path{primary_replicates.csv}, and \path{numerical_audit.json}. The checksum inventory and numerical-verification report identify these files. Access to the full evaluation arrays is described in the Data availability statement.

% ===== END appendices/certificate_provenance.tex =====

% ===== BEGIN appendices/supporting_diagnostics.tex =====
\section{Supporting Feasibility and Spectral Diagnostics}
\label{app:supporting_diagnostics}
\subsection{Effects of simplex projection on prediction}
\label{sec:exp_closure}
Identity-initialized NLMS preserves $\widehat B\mathbf1_C=\mathbf1_C$ in exact arithmetic, so predicted coordinates sum to one. They can nevertheless be negative. Projecting each predicted vector onto the simplex enforces feasible output coordinates but makes repeated prediction nonlinear. Projecting each row of the fitted matrix once yields a nonnegative unit-row-sum matrix whose transpose maps the simplex into itself. These are different operations.

Table~\ref{tab:exp_projection} compares prediction errors for the specified rollout rules at the reference configurations. Direct simplex-distance measurements, including the mean squared projection distance $\widehat\nu_{\mathrm E_t}$, are not reported in these summaries. Row projection increases long-horizon error substantially in these runs. Per-step vector projection changes it less and lowers Van der Pol $E_{200}$ from $0.180$ to $0.153$. Prediction accuracy therefore depends on the projection rule as well as output feasibility.

\begin{table}[tbp]
\centering
\caption{Projection diagnostics at the reference configurations. Duffing uses one fixed trajectory and three fits with different abstraction initialization seeds; Van der Pol uses three separately generated trajectories, with a separate abstraction and prediction matrix fitted to each. Entries are mean relative squared errors over the three fits for each system. These runs differ from those in Table~\ref{tab:exp_ablation}. The corresponding mean Duffing one-step $R^2$ is $0.9861$.}
\label{tab:exp_projection}
\small
\begin{tabular}{llrrr}
\toprule
System & Rollout rule & $E_1$ & $E_{20}$ & $E_{200}$\\
\midrule
Duffing & No projection & $2.09\times10^{-5}$ & $5.11\times10^{-4}$ & $9.14\times10^{-4}$\\
& Per-step vector projection & $2.08\times10^{-5}$ & $5.09\times10^{-4}$ & $1.07\times10^{-3}$\\
& Row-stochastic matrix & $4.40\times10^{-5}$ & $7.71\times10^{-3}$ & $1.59\times10^{-1}$\\
\midrule
Van der Pol & No projection & $1.89\times10^{-3}$ & $2.60\times10^{-2}$ & $1.80\times10^{-1}$\\
& Per-step vector projection & $1.84\times10^{-3}$ & $2.22\times10^{-2}$ & $1.53\times10^{-1}$\\
& Row-stochastic matrix & $3.93\times10^{-3}$ & $3.14\times10^{-1}$ & $9.56\times10^{-1}$\\
\bottomrule
\end{tabular}
\end{table}

\subsection{Interpreting the empirical spectral diagnostic}
\label{sec:exp_spectral}
For the selected oscillator maps, the spectral analysis uses 3,600 training pairs from seeds $0,1,2$, 2,400 held-out pairs from seeds $100,101,102$, and a fixed 800-pair subsample for the spectral calculations. The ridge estimator minimizes the summed squared feature-prediction error plus $10^{-8}\|B\|_F^2$. To obtain reduced coordinates, the study applies a singular value decomposition to the uncentered training-feature matrix. It retains singular values larger than $10^{-10}$ times the largest singular value, giving an orthonormal basis $\widehat Q_\Phi$ for the retained subspace.

The retained rank describes the feature directions resolved in the training data. Sampling can miss directions that occur elsewhere in $\mathcal X$, and the singular-value threshold can discard directions associated with small singular values. Thus the sampled basis estimates the full feature span $V_\Phi$ used in the theoretical proposition. Establishing that the two spans agree requires a spanning argument or additional information about the feature map.

For an eigenpair $(\lambda,\widetilde w)$ of $M_{\mathrm{red}}^\top$, where $M_{\mathrm{red}}=\widehat Q_\Phi^\top\widehat B\widehat Q_\Phi$ and $\widetilde w\neq0$, let $w=\widehat Q_\Phi\widetilde w$. The matrices $\Phi_{\mathrm{spec}},\Phi_{\mathrm{spec}}^+\in\mathbb R^{C\times800}$ contain the current and successor coordinate columns from the spectral subsample. Choose the unique minimizer $v\in\mathbb C^{800}$ of $\|\Phi_{\mathrm{spec}}v-w\|^2+10^{-8}\|v\|^2$ and compute
\begin{equation}
\eta_{\mathrm{rep}}=\frac{\|\Phi_{\mathrm{spec}}v-w\|^2}{\|w\|^2},
\qquad
\rho=\frac{\|\Phi_{\mathrm{spec}}^+v-\lambda\Phi_{\mathrm{spec}}v\|^2}{\|\Phi_{\mathrm{spec}}^+v\|^2}.
\label{eq:exp_spectral_residual}
\end{equation}
The first ratio measures how accurately the selected weights represent the candidate coefficient vector. The second compares the weighted successor sum $\Phi_{\mathrm{spec}}^+v$ with $\lambda$ times the weighted current sum $\Phi_{\mathrm{spec}}v$. A valid relative eigenfunction diagnostic requires $w\neq0$, $\Phi_{\mathrm{spec}}v\neq0$, and $\Phi_{\mathrm{spec}}^+v\neq0$. Both ratios are squared quantities; for example, $\rho\leq0.01$ corresponds to a relative residual norm of at most $0.1$.

The feature vectors are summed before the norm is taken, so contributions with opposite signs or different complex phases can cancel. A small $\rho$ therefore describes agreement between two aggregate vectors. It does not bound the error of the proposed eigenrelation separately at each state. The transitions used here are stochastic. Applying the deterministic adjoint result would require a deterministic state map and pointwise exact closure, as in Proposition~\ref{prop_koopman_eigenfunction}. A stochastic operator interpretation would instead require a stated conditional-mean closure condition and a sampling bound connecting the observed successor sums to their conditional expectations. The reported values provide finite-sample aggregate diagnostics under the documented sampling design.

\begin{table}[tbp]
\centering
\caption{Aggregate spectral diagnostics on held-out data. Candidates are eigenpairs of $M_{\mathrm{red}}^\top$, the transpose of the fitted prediction matrix expressed on the retained training-feature subspace. The squared residual ratio $\rho$ and coefficient-representation error $\eta_{\mathrm{rep}}$ use~\eqref{eq:exp_spectral_residual} on a fixed 800-pair held-out subsample. The count triple gives the numbers of valid candidates satisfying $\rho\leq0.01$, $0.05$, and $0.10$, respectively. Here $\widehat r_\Phi=C$; $\operatorname{spr}$ denotes spectral radius.}
\label{tab:exp_spectra}
\scriptsize
\setlength{\tabcolsep}{2.6pt}
\resizebox{\linewidth}{!}{%
\begin{tabular}{llrrcrrrrl}
\toprule
System &
Estimator &
$E_1$ &
$R_1^2$ &
$\widehat r_\Phi/C$ &
$\operatorname{spr}(M_{\mathrm{red}})$ &
Min.\ $\rho$ &
Median $\rho$ &
Max.\ $\eta_{\mathrm{rep}}$ &
Counts\\
\midrule
Duffing &
NLMS &
$2.50\times10^{-4}$ &
0.9835 &
$10/10$ &
1.000 &
$2.75\times10^{-5}$ &
$5.88\times10^{-2}$ &
$4.58\times10^{-10}$ &
$3/4/9$\\

&
Ridge least squares &
$2.41\times10^{-4}$ &
0.9842 &
$10/10$ &
1.000 &
$9.52\times10^{-7}$ &
$6.44\times10^{-4}$ &
$3.66\times10^{-12}$ &
$9/10/10$\\
\midrule

Van der Pol &
NLMS &
$3.05\times10^{-3}$ &
0.9963 &
$25/25$ &
1.000 &
$1.89\times10^{-5}$ &
$1.75\times10^{-4}$ &
$1.32\times10^{-14}$ &
$23/25/25$\\

&
Ridge least squares &
$3.03\times10^{-3}$ &
0.9964 &
$25/25$ &
1.000 &
$1.91\times10^{-5}$ &
$1.01\times10^{-4}$ &
$1.13\times10^{-14}$ &
$25/25/25$\\
\bottomrule
\end{tabular}%
}
\end{table}

All four reported fits have numerical training rank $C$. Within each system, their one-step prediction scores are similar, while some spectral diagnostics differ substantially. On Duffing, median $\rho$ is $5.88\times10^{-2}$ for NLMS and $6.44\times10^{-4}$ for ridge least squares. Both Van der Pol medians are below $2\times10^{-4}$. The largest coefficient-representation error is $4.58\times10^{-10}$, so the sampled sections closely approximate the reported candidate coefficients. Direct feature-vector calculations and equivalent Gram-matrix calculations agree within $2.7\times10^{-14}$. This agreement checks the numerical evaluation of the diagnostic; the exact-closure assumption requires separate evidence.

All reported spectral radii round to $1.000$, so the displayed precision cannot resolve whether a radius exceeds one. Even a radius of exactly one does not, by itself, ensure bounded matrix powers: a nontrivial Jordan block at an eigenvalue of modulus one can cause polynomial growth. Nonorthogonal eigenvectors can also cause substantial transient amplification. These properties matter because multistep prediction repeatedly applies the fitted matrix. The rounded spectral radii therefore do not establish stability of long rollouts. The repository's per-mode outputs and plots supply further numerical detail.

% ===== END appendices/supporting_diagnostics.tex =====

\printbibliography[heading=bibintoc,title={References}]
\end{document}